\pdfoutput=1
\documentclass[10pt]{article}

\usepackage{etoolbox}
\usepackage{comment}
\newtoggle{workshop}
\togglefalse{workshop}
\newtoggle{colt}
\togglefalse{colt}
\newcommand{\workshop}[1]{}
\newcommand{\arxiv}[1]{\iftoggle{colt}{}{#1}}
\newcommand{\colt}[1]{\iftoggle{colt}{#1}{}}

\newcommand{\loose}{\colt{\looseness=-1}}

\usepackage[utf8]{inputenc} %
\usepackage[T1]{fontenc}    %
\usepackage{url}            %
\usepackage{booktabs}       %
\usepackage{amsfonts}       %
\usepackage{nicefrac}       %
\usepackage{microtype}      %

\usepackage{enumitem}

\usepackage{mathrsfs}

\usepackage{algorithm}
\usepackage{verbatim}
\usepackage[noend]{algpseudocode}

\usepackage{colortbl}

\usepackage{subcaption}
\usepackage{transparent}

\usepackage{inconsolata}
\usepackage[scaled=.90]{helvet}
\usepackage{xspace}

\usepackage{pifont}

\usepackage{setspace}

\arxiv{\usepackage[letterpaper, left=1in, right=1in, top=1in, bottom=1in]{geometry}}

\PassOptionsToPackage{hypertexnames=false}{hyperref}  %
\arxiv{\usepackage{parskip}}

\usepackage[dvipsnames]{xcolor}
\arxiv{\usepackage[colorlinks=true, linkcolor=blue!70!black, citecolor=blue!70!black,urlcolor=black]{hyperref}}
\usepackage{microtype}
\usepackage{hhline}

\usepackage{algorithm}

\arxiv{
\usepackage{natbib}
\bibliographystyle{plainnat}
\bibpunct{(}{)}{;}{a}{,}{,}
}
\workshop{
\usepackage[numbers]{natbib}
\bibliographystyle{plainnat}
}

\usepackage{amsthm}
\usepackage{mathtools}
\usepackage{amsmath}
\usepackage{bbm}
\usepackage{amsfonts}
\usepackage{amssymb}

\usepackage{xpatch}

\theoremstyle{definition}  %

\theoremstyle{plain}

\xpatchcmd{\proof}{\itshape}{\normalfont\proofnameformat}{}{}
\newcommand{\proofnameformat}{\bfseries}

\usepackage[nameinlink,capitalize]{cleveref}

\newcommand{\pref}[1]{\cref{#1}}
\newcommand{\pfref}[1]{Proof of \pref{#1}}

\crefformat{equation}{#2(#1)#3}
\Crefformat{equation}{#2(#1)#3}

\Crefformat{figure}{#2Figure #1#3}
\Crefformat{assumption}{#2Assumption #1#3}

\usepackage{xparse}

\ExplSyntaxOn
\DeclareDocumentCommand{\XDeclarePairedDelimiter}{mm}
 {
  \__egreg_delimiter_clear_keys: %
  \keys_set:nn { egreg/delimiters } { #2 }
  \use:x %
   {
    \exp_not:n {\NewDocumentCommand{#1}{sO{}m} }
     {
      \exp_not:n { \IfBooleanTF{##1} }
       {
        \exp_not:N \egreg_paired_delimiter_expand:nnnn
         { \exp_not:V \l_egreg_delimiter_left_tl }
         { \exp_not:V \l_egreg_delimiter_right_tl }
         { \exp_not:n { ##3 } }
         { \exp_not:V \l_egreg_delimiter_subscript_tl }
       }
       {
        \exp_not:N \egreg_paired_delimiter_fixed:nnnnn 
         { \exp_not:n { ##2 } }
         { \exp_not:V \l_egreg_delimiter_left_tl }
         { \exp_not:V \l_egreg_delimiter_right_tl }
         { \exp_not:n { ##3 } }
         { \exp_not:V \l_egreg_delimiter_subscript_tl }
       }
     }
   }
 }

\keys_define:nn { egreg/delimiters }
 {
  left      .tl_set:N = \l_egreg_delimiter_left_tl,
  right     .tl_set:N = \l_egreg_delimiter_right_tl,
  subscript .tl_set:N = \l_egreg_delimiter_subscript_tl,
 }

\cs_new_protected:Npn \__egreg_delimiter_clear_keys:
 {
  \keys_set:nn { egreg/delimiters } { left=.,right=.,subscript={} }
 }

\cs_new_protected:Npn \egreg_paired_delimiter_expand:nnnn #1 #2 #3 #4
 {%
  \mathopen{}
  \mathclose\c_group_begin_token
   \left#1
   #3
   \group_insert_after:N \c_group_end_token
   \right#2
   \tl_if_empty:nF {#4} { \c_math_subscript_token {#4} }
 }
\cs_new_protected:Npn \egreg_paired_delimiter_fixed:nnnnn #1 #2 #3 #4 #5
 {
  \mathopen{#1#2}#4\mathclose{#1#3}
  \tl_if_empty:nF {#5} { \c_math_subscript_token {#5} }
 }
\ExplSyntaxOff

\XDeclarePairedDelimiter{\supnorm}{
  left=\lVert,
  right=\rVert,
  subscript=\infty
  }

\DeclarePairedDelimiter{\abs}{\lvert}{\rvert} %
\DeclarePairedDelimiter{\brk}{[}{]}
\DeclarePairedDelimiter{\crl}{\{}{\}}
\DeclarePairedDelimiter{\prn}{(}{)}
\DeclarePairedDelimiter{\nrm}{\|}{\|}
\DeclarePairedDelimiter{\tri}{\langle}{\rangle}

\DeclarePairedDelimiter{\floor}{\lfloor}{\rfloor}

\let\Pr\undefined

\DeclareMathOperator{\En}{\mathbb{E}}

\newcommand{\mb}[1]{\boldsymbol{#1}}

\newcommand{\wh}[1]{\widehat{#1}}

\def\ddefloop#1{\ifx\ddefloop#1\else\ddef{#1}\expandafter\ddefloop\fi}
\def\ddef#1{\expandafter\def\csname bb#1\endcsname{\ensuremath{\mathbb{#1}}}}
\ddefloop ABCDEFGHIJKLMNOPQRSTUVWXYZ\ddefloop
\def\ddefloop#1{\ifx\ddefloop#1\else\ddef{#1}\expandafter\ddefloop\fi}
\def\ddef#1{\expandafter\def\csname b#1\endcsname{\ensuremath{\mathbf{#1}}}}
\ddefloop ABCDEFGHIJKLMNOPQRSTUVWXYZ\ddefloop
\def\ddef#1{\expandafter\def\csname sf#1\endcsname{\ensuremath{\mathsf{#1}}}}
\ddefloop ABCDEFGHIJKLMNOPQRSTUVWXYZ\ddefloop
\def\ddef#1{\expandafter\def\csname c#1\endcsname{\ensuremath{\mathcal{#1}}}}
\ddefloop ABCDEFGHIJKLMNOPQRSTUVWXYZ\ddefloop
\def\ddef#1{\expandafter\def\csname h#1\endcsname{\ensuremath{\widehat{#1}}}}
\ddefloop ABCDEFGHIJKLMNOPQRSTUVWXYZ\ddefloop
\def\ddef#1{\expandafter\def\csname hc#1\endcsname{\ensuremath{\widehat{\mathcal{#1}}}}}
\ddefloop ABCDEFGHIJKLMNOPQRSTUVWXYZ\ddefloop
\def\ddef#1{\expandafter\def\csname t#1\endcsname{\ensuremath{\widetilde{#1}}}}
\ddefloop ABCDEFGHIJKLMNOPQRSTUVWXYZ\ddefloop
\def\ddef#1{\expandafter\def\csname tc#1\endcsname{\ensuremath{\widetilde{\mathcal{#1}}}}}
\ddefloop ABCDEFGHIJKLMNOPQRSTUVWXYZ\ddefloop
\def\ddefloop#1{\ifx\ddefloop#1\else\ddef{#1}\expandafter\ddefloop\fi}
\def\ddef#1{\expandafter\def\csname scr#1\endcsname{\ensuremath{\mathscr{#1}}}}
\ddefloop ABCDEFGHIJKLMNOPQRSTUVWXYZ\ddefloop

\newcommand{\ind}{\mathbbm{1}}    %

\newcommand{\eps}{\epsilon}
\newcommand{\veps}{\varepsilon}

\newcommand{\ldef}{\vcentcolon=}

\newcommand{\gammaconst}{(1-\gamma)^{-1}}

\newcommand{\filt}{\mathscr{F}}

\newcommand{\pushforward}{pushforward concentrability\xspace}

\newcommand{\nullstate}{\mathfrak{s}}
\newcommand{\init}{\nullstate}

\newcommand{\Conc}{C_{\mathrm{conc}}}

\newcommand{\datadist}{data distribution\xspace}

\newcommand{\dataset}{D_n}
\newcommand{\Dn}{D_n}

\newcommand{\overcoverage}{over-coverage\xspace}
\newcommand{\overcoverages}{strong over-coverage\xspace}

\newcommand{\dnot}{d_0}

\newcommand{\Jm}{J\subs{M}}

\newcommand{\Enm}[2]{\En^{\sss{#1},#2}}

\newcommand{\Pr}{\bbP}

\newcommand{\Prpi}[1][\pi]{\bbP^{#1}}
\newcommand{\Prm}[2]{\bbP^{\sss{#1},#2}}

\newcommand{\Qm}[2]{Q_{\sss{#1}}^{#2}}
\newcommand{\Vm}[2]{V_{\sss{#1}}^{#2}}

\newcommand{\Qstar}{Q^{\star}}
\newcommand{\Vstarm}[1]{V_{\sss{#1}}^{\star}}
\newcommand{\Qstarm}[1]{Q_{\sss{#1}}^{\star}}

\newcommand{\dm}[2]{d_{\sss{#1}}^{#2}}

\newcommand{\pistarm}{\pi_{\sss{M}}^{\star}}
\newcommand{\pistar}{\pi^{\star}}

\newcommand{\psub}{I}
\newcommand{\qsub}{J}

\newcommand{\initac}{\mathfrak{a}}

\newcommand{\M}{\sss{M}}
\newcommand{\sM}{\sss{M}}

\newcommand{\subs}[1]{_{{#1}}}

\newcommand{\sss}[1]{{#1}}

\newcommand{\pihat}{\wh{\pi}}

\newcommand{\trn}{\top}

\newcommand{\approxgeq}{\gtrsim}

\newcommand{\eigmin}{\lambda_{\mathrm{min}}}

\renewcommand{\ind}[1]{^{{\scriptscriptstyle(#1)}}}
\newcommand{\Ind}{\mathbbm{1}}

\newcommand{\bigom}{\Omega}

\newcommand{\poly}{\mathrm{poly}}

\newcommand{\dmid}{\,\|\,}

\newcommand{\Dchis}[2]{D_{\chi^{2}}\prn*{#1\dmid#2}}
\newcommand{\chisquared}{$\chi^{2}$-divergence\xspace}

\newcommand{\Dtv}[2]{D_{\mathrm{TV}}\prn*{#1,#2}}

\newcommand{\Ssafe}{\cS_{\mathrm{safe}}}

\newcommand{\unif}{\mathrm{Unif}}

\renewcommand{\Ssafe}[1][h]{\cS_{h;\mathrm{safe}}}

  \newcommand{\mathand}{\quad\text{and}\quad}

\def\multiset#1#2{\ensuremath{\left(\kern-.3em\left(\genfrac{}{}{0pt}{}{#1}{#2}\right)\kern-.3em\right)}}

\newcommand{\iid}{i.i.d.\xspace}

\let\underbar\undefined

\makeatletter
\let\save@mathaccent\mathaccent
\newcommand*\if@single[3]{%
  \setbox0\hbox{${\mathaccent"0362{#1}}^H$}%
  \setbox2\hbox{${\mathaccent"0362{\kern0pt#1}}^H$}%
  \ifdim\ht0=\ht2 #3\else #2\fi
  }
\newcommand*\rel@kern[1]{\kern#1\dimexpr\macc@kerna}
\newcommand*\widebar[1]{\@ifnextchar^{{\wide@bar{#1}{0}}}{\wide@bar{#1}{1}}}
\newcommand*\underbar[1]{\@ifnextchar_{{\under@bar{#1}{0}}}{\under@bar{#1}{1}}}
\newcommand*\wide@bar[2]{\if@single{#1}{\wide@bar@{#1}{#2}{1}}{\wide@bar@{#1}{#2}{2}}}
\newcommand*\under@bar[2]{\if@single{#1}{\under@bar@{#1}{#2}{1}}{\under@bar@{#1}{#2}{2}}}
\newcommand*\wide@bar@[3]{%
  \begingroup
  \def\mathaccent##1##2{%
    \let\mathaccent\save@mathaccent
    \if#32 \let\macc@nucleus\first@char \fi
    \setbox\z@\hbox{$\macc@style{\macc@nucleus}_{}$}%
    \setbox\tw@\hbox{$\macc@style{\macc@nucleus}{}_{}$}%
    \dimen@\wd\tw@
    \advance\dimen@-\wd\z@
    \divide\dimen@ 3
    \@tempdima\wd\tw@
    \advance\@tempdima-\scriptspace
    \divide\@tempdima 10
    \advance\dimen@-\@tempdima
    \ifdim\dimen@>\z@ \dimen@0pt\fi
    \rel@kern{0.6}\kern-\dimen@
    \if#31
      \overline{\rel@kern{-0.6}\kern\dimen@\macc@nucleus\rel@kern{0.4}\kern\dimen@}%
      \advance\dimen@0.4\dimexpr\macc@kerna
      \let\final@kern#2%
      \ifdim\dimen@<\z@ \let\final@kern1\fi
      \if\final@kern1 \kern-\dimen@\fi
    \else
      \overline{\rel@kern{-0.6}\kern\dimen@#1}%
    \fi
  }%
  \macc@depth\@ne
  \let\math@bgroup\@empty \let\math@egroup\macc@set@skewchar
  \mathsurround\z@ \frozen@everymath{\mathgroup\macc@group\relax}%
  \macc@set@skewchar\relax
  \let\mathaccentV\macc@nested@a
  \if#31
    \macc@nested@a\relax111{#1}%
  \else
    \def\gobble@till@marker##1\endmarker{}%
    \futurelet\first@char\gobble@till@marker#1\endmarker
    \ifcat\noexpand\first@char A\else
      \def\first@char{}%
    \fi
    \macc@nested@a\relax111{\first@char}%
  \fi
  \endgroup
}
\newcommand*\under@bar@[3]{%
  \begingroup
  \def\mathaccent##1##2{%
    \let\mathaccent\save@mathaccent
    \if#32 \let\macc@nucleus\first@char \fi
    \setbox\z@\hbox{$\macc@style{\macc@nucleus}_{}$}%
    \setbox\tw@\hbox{$\macc@style{\macc@nucleus}{}_{}$}%
    \dimen@\wd\tw@
    \advance\dimen@-\wd\z@
    \divide\dimen@ 3
    \@tempdima\wd\tw@
    \advance\@tempdima-\scriptspace
    \divide\@tempdima 10
    \advance\dimen@-\@tempdima
    \ifdim\dimen@>\z@ \dimen@0pt\fi
    \rel@kern{0.6}\kern-\dimen@
    \if#31
      \underline{\rel@kern{-0.6}\kern\dimen@\macc@nucleus\rel@kern{0.4}\kern\dimen@}%
      \advance\dimen@0.4\dimexpr\macc@kerna
      \let\final@kern#2%
      \ifdim\dimen@<\z@ \let\final@kern1\fi
      \if\final@kern1 \kern-\dimen@\fi
    \else
      \underline{\rel@kern{-0.6}\kern\dimen@#1}%
    \fi
  }%
  \macc@depth\@ne
  \let\math@bgroup\@empty \let\math@egroup\macc@set@skewchar
  \mathsurround\z@ \frozen@everymath{\mathgroup\macc@group\relax}%
  \macc@set@skewchar\relax
  \let\mathaccentV\macc@nested@a
  \if#31
    \macc@nested@a\relax111{#1}%
  \else
    \def\gobble@till@marker##1\endmarker{}%
    \futurelet\first@char\gobble@till@marker#1\endmarker
    \ifcat\noexpand\first@char A\else
      \def\first@char{}%
    \fi
    \macc@nested@a\relax111{\first@char}%
  \fi
  \endgroup
}
\makeatother

 \usepackage[textsize=tiny]{todonotes}

\usepackage{thmtools}
\usepackage{thm-restate}
\declaretheorem[name=Theorem,parent=section]{theorem}
\declaretheorem[name=Lemma,parent=section]{lemma}

\declaretheorem[name=Proposition, parent=section]{proposition}

\makeatletter
  \renewenvironment{proof}[1][Proof]%
  {%
   \par\noindent{\bfseries\upshape {#1.}\ }%
  }%
  {\qed\newline}
  \makeatother

\makeatletter
\let\OldStatex\Statex
\renewcommand{\Statex}[1][3]{%
  \setlength\@tempdima{\algorithmicindent}%
  \OldStatex\hskip\dimexpr#1\@tempdima\relax}
\makeatother

\usetikzlibrary{decorations.pathreplacing}
\usetikzlibrary{patterns}
\usetikzlibrary{arrows.meta}
\usetikzlibrary{shapes.symbols}
\let\oldparagraph\paragraph
\renewcommand{\paragraph}[1]{\oldparagraph{#1.}}

\title{Offline Reinforcement Learning: \\Fundamental Barriers for Value Function Approximation
  }

\colt{
\title[Offline RL: Fundamental Barriers for Value Function Approximation]{Offline Reinforcement Learning: \\Fundamental Barriers for Value Function Approximation}
\usepackage{times}

\coltauthor{%
 \Name{Author Name1} \Email{abc@sample.com}\\
 \addr Address 1
 \AND
 \Name{Author Name2} \Email{xyz@sample.com}\\
 \addr Address 2%
}

}
  
  \arxiv{
  \author{%
Dylan J. Foster
\quad\quad
Akshay Krishnamurthy\\
{\normalsize Microsoft Research}\\
{\small\texttt{\{dylanfoster,akshaykr\}@microsoft.com}}
\and
David Simchi-Levi
\quad\quad
Yunzong Xu\thanks{Part of this work was completed while Y. Xu was an intern at Microsoft Research.}\\
 {\normalsize Massachusetts Institute of Technology}\\
 {\small\texttt{\{dslevi,yxu\}@mit.edu}}
}

\date{}
}

\addtocontents{toc}{\protect\setcounter{tocdepth}{0}}

\begin{document}
\maketitle

\begin{abstract}
  We consider the offline reinforcement learning problem, where the aim is to
  learn a decision making policy from logged data. Offline
  RL---particularly when coupled with (value) function approximation to allow
  for generalization in large or continuous state spaces---is becoming increasingly relevant in
  practice, because it avoids costly and time-consuming online data
  collection and is well suited to safety-critical domains.
  Existing sample complexity guarantees for offline value function
  approximation methods typically require both (1) distributional
  assumptions (i.e., good coverage) and (2) representational assumptions
  (i.e., ability to represent some or all $Q$-value functions)
  stronger than what is required for supervised learning.
  However, the
necessity of these conditions and the fundamental limits of offline RL
are not well understood in spite of decades of research. This led Chen and Jiang (2019) to conjecture that \emph{concentrability} (the
most standard notion of coverage) and \emph{realizability} (the
weakest representation condition) %
alone are not sufficient for
sample-efficient offline RL. We resolve this conjecture in
the positive by proving that in general, even if both
concentrability and realizability are satisfied, any algorithm
requires sample complexity either polynomial in the size of the state space or exponential in other parameters to
learn a non-trivial policy.

Our results show that sample-efficient offline reinforcement
learning requires either restrictive coverage conditions or representation
conditions that go beyond supervised learning, and highlight a phenomenon
called \emph{\overcoverage} which serves as a fundamental barrier for offline
value function approximation methods. A consequence of our results for
reinforcement learning with linear function approximation is that the separation between online and offline RL can be \emph{arbitrarily large}, even in
constant dimension.%

\end{abstract}

\section{Introduction}
\label{sec:intro}

In offline reinforcement learning, we aim to evaluate or optimize
decision making policies using logged transitions and
rewards from historical experiments or expert
demonstrations. Offline RL has great promise for decision making
applications where actively acquiring data is expensive or
cumbersome (e.g., robotics \citep{pinto2016supersizing,levine2018learning,kalashnikov2018scalable}), or where safety is critical (e.g.,
autonomous driving \citep{sallab2017deep,kendall2019learning} and healthcare \citep{gottesman2018evaluating,gottesman2019guidelines,wang2018supervised,yu2019deep,nie2021learning}). In particular, there is substantial
interest in combining offline reinforcement learning with function
approximation (e.g., deep neural networks) in order to encode inductive biases and enable
generalization across large, potentially continuous state spaces, with
recent progress on both model-free and model-based approaches
\citep{ross2012agnostic,laroche2019safe,fujimoto2019off,kumar2019stabilizing,agarwal2020optimistic}. However,
existing algorithms are extremely data-intensive, and offline RL
methods---to date---have seen limited deployment in the aforementioned
applications. To enable practical deployment going forward, it is
paramount that we develop a strong understanding of the statistical
foundations for reliable, sample-efficient offline reinforcement
learning with function approximation, as well as an understanding of
when and why existing methods succeed and how to effectively collect data.

Compared to the basic supervised learning problem,
offline reinforcement learning with function approximation
  poses substantial algorithmic challenges due to two issues: \emph{distribution shift} and \emph{credit assignment}. Within the literature on \emph{value} function approximation (or,
approximate dynamic programming), all existing methods require both
(1) distributional conditions, which assert that the logged data has
good coverage (addressing distribution shift), and (2) representational conditions, which assert that
the function approximator is flexible enough to represent value
functions induced by certain policies (addressing credit assignment). Notably, sample complexity
analyses for standard offline RL methods (e.g., fitted Q-iteration)
require representation conditions considerably more restrictive than what is
required for supervised learning \citep{munos2003error,munos2007performance,munos2008finite,antos2008learning}, and these methods can diverge
when these conditions do not hold %
\citep{gordon1995stable,tsitsiklis1996feature,tsitsiklis1997analysis,wang2021instabilities}. Despite
substantial research effort, it is not known whether these
conditions constitute fundamental limits or whether the algorithms
can be improved. Resolving this issue would serve as a stepping stone
toward developing a theory for offline reinforcement learning that
parallels our understanding of supervised (statistical) learning.

The lack of understanding of fundamental limits in offline
reinforcement learning was highlighted by
\citet{chen2019information}, who observed that all existing
finite-sample analyses for offline RL algorithms
based on \emph{concentrability} \citep{munos2003error}---the most ubiquitous notion of data
coverage---require representation conditions significantly stronger
than \emph{realizability}, a standard condition from supervised
learning which asserts that the function approximator can represent
optimal value functions. \citet{chen2019information} conjectured that realizability and
concentrability alone do not suffice for sample-efficient offline
RL and noted that proving such a result seemed to be out of reach for
existing lower bound techniques. Subsequent progress led to positive results for
sample-efficient offline RL under coverage conditions stronger than concentrability \citep{xie2021batch} and impossibility results under
weaker coverage conditions \citep{wang2020statistical,zanette2021exponential}, but the original
conjecture remained open.

\paragraph{Contributions}
  We provide information-theoretic lower bounds which show that, in general, concentrability and realizability together are not sufficient for sample efficient offline reinforcement learning. 
Our first result concerns the standard offline RL setup, where the data
collection distribution is only required to satisfy concentrability,
and establishes a sample complexity lower bound scaling polynomially with
the size of the state space. This result resolves the conjecture of \citet{chen2019information} in the positive.
For our second result, we further restrict the data distribution to be induced by a policy (i.e., admissible), and show that any algorithm requires sample complexity either polynomial in
the size of the state space or exponential in other problem
parameters. Together, our results 
establish that sample-efficient offline RL in
large state spaces is not possible unless more stringent conditions,
either distributional or representational, hold.

Our lower bound constructions are qualitatively different from
    previous approaches and hold even when the number
    of actions is constant and the value function class
    has constant size. Our first lower bound highlights the role of a phenomenon we
    call strong \emph{\overcoverage} (first documented by \citet{xie2021batch}), wherein the data collection distribution is
    supported over spurious states that are not reachable by
    any policy. Despite the irrelevance of these states for learning
    in the online setting, their inclusion in the offline dataset creates significant
    uncertainty. 
    Our second lower bound discovers a weak variant of \overcoverage, wherein the data collection distribution is induced by running an exploratory policy in particular time steps, but many of the states supported by this distribution are not reachable in other time steps, creating spurious correlations. 
    Our work shows that both the strong and weak \overcoverage phenomena serve as 
    fundamental, information-theoretic barriers for the design of offline
    reinforcement learning algorithms.

\subsection{Offline Reinforcement Learning Setting}
\paragraph{Markov decision processes}
We consider the infinite-horizon discounted reinforcement
learning setting. Formally, a Markov decision process $M=(\cS, \cA, P, R,
\gamma, \dnot)$ consists of a (potentially large/continuous) state
space $\cS$, action space $\cA$, probability transition function
$P:\cS\times\cA\to\Delta(\cS)$, reward function
$R:\cS\times{}\cA\to\brk*{0,1}$, discount factor $\gamma\in[0,1)$, and
initial state distribution $\dnot \in \Delta(\cS)$. 
Each (randomized) policy $\pi:\cS\to\Delta(\cA)$ induces a
distribution over trajectories $(s_0,a_0, r_0),(s_1,a_1,r_1),\ldots$
via the following process. For $h=0,1,\ldots:$ $a_h\sim\pi(s_h)$,
$r_h=R(s_h,a_h)$, and $s_{h+1}\sim{}P(s_h,a_h)$, with
  $s_0\sim\dnot$. We let $\Enm{M}{\pi}\brk*{\cdot}$ and $\Prm{M}{\pi}\prn{\cdot}$
  denote expectation and probability under this process, respectively. 

  The expected return for policy $\pi$ is defined as $\Jm(\pi) :=
  \Enm{M}{\pi}\brk*{\sum_{h=0}^{\infty}\gamma^{h}r_h}$, and the value
  function and $Q$-function for $\pi$ are given by
  \[\textstyle
    \Vm{M}{\pi}(s):=\Enm{M}{\pi}\brk*{\sum_{h=0}^{\infty}\gamma^{h}r_h\mid{}s_0=s},\mathand
    \Qm{M}{\pi}(s,a):=\Enm{M}{\pi}\brk*{\sum_{h=0}^{\infty}\gamma^{h}r_h\mid{}s_0=s,
      a_0=a}.\]
  It is well-known that there exists a {deterministic} policy $\pi_M^\star:\cS\rightarrow\cA$ that maximizes $\Vm{M}{\pi}(s)$ for all $s\in\cS$ simultaneously and thus also maximizes $\Jm(\pi)$.   %
  Letting $\Vstarm{M}:=\Vm{M}{\pistarm}$ and
  $\Qstarm{M}:=\Qm{M}{\pistarm}$, we have $\pi_M^*(s)=\arg\max_{a\in\cA}Q_M^\star(s,a)$ for all $s\in\cS$.
 Finally, we define the occupancy measure for policy $\pi$
via $\dm{M}{\pi}(s,a) :=
(1-\gamma)\sum_{h=0}^{\infty}\gamma^{h}\Prm{M}{\pi}\prn{s_h=s,a_h=a}$. We
drop the dependence on the model $M$ when it is clear from context.

\paragraph{Offline policy learning}
In the offline policy learning (or, optimization) problem, we do not have
direct access to the underlying MDP and instead receive a dataset
$\dataset$ of tuples $(s, a, r, s')$ with $r=R(s,a)$,
$s'\sim{}P(s,a)$, and $(s,a)\sim\mu$ \iid, where
$\mu\in\Delta(\cS\times\cA)$ is the \emph{data collection
  distribution}. The goal of the learner is to use the dataset
$\dataset$ to learn an $\veps$-optimal policy $\pihat$, that is:
\[
J(\pistar) - \En\brk*{J(\pihat)} \leq \veps,
\]
where the expectation $\En\brk*{\cdot}$ is over the draw of $\dataset$
and any randomness used by the algorithm.

In order to provide sample-efficient learning guarantees that do
not depend on the size of the state space, value function
  approximation methods take advantage of the following conditions.
\begin{itemize}
\item \textbf{Realizability.} This condition asserts that we have access to a class of candidate
  value functions $\cF\subseteq(\cS\times\cA\to\bbR)$ (e.g., linear
  models or neural networks) such that $\Qstar\in\cF$. Realizability
  (that is, a well-specified model) is the most common
  representation condition in supervised learning and statistical
  estimation \citep{bousquet2003introduction,wainwright2019high} and is also
  widely used in contextual bandits
  \citep{agarwal2012contextual,foster2018practical}. %
\item \textbf{Concentrability.} Call a distribution
  $\nu\in\Delta(\cS\times\cA)$ \emph{admissible} for the MDP $M$ if
  there exists a (potentially stochastic and non-stationary\footnote{A non-stationary policy is a sequence $\{\pi_h\}_{h\geq 0}$, which generates a trajectory via $a_h \sim \pi_h(s_h)$.}) policy
  $\pi$ and index $h$ such that
  $\nu(s,a)=\Prpi\brk*{s_h=s,a_h=a}$. This condition asserts that there exists a constant $\Conc<\infty$ such that for all admissible $\nu$,
  \begin{equation}
    \label{eq:concentrability}
    \nrm*{\frac{\nu}{\mu}}_{\infty}\ldef\sup_{(s,a)\in\cS\times\cA}\crl*{\frac{\nu(s,a)}{\mu(s,a)}}\leq{}\Conc.
  \end{equation}
  Concentrability is a simple but fairly strong notion of coverage
  which demands that the \datadist uniformly covers all reachable states.
\end{itemize}
Under these conditions, an offline RL algorithm is said to be
sample-efficient if it learns an $\veps$-optimal policy with
$\poly(\veps^{-1}, (1-\gamma)^{-1},\Conc ,\log\abs{\cF})$
samples. Notably, such a guarantee depends only on the complexity
$\log\abs{\cF}$ for the value function class, not on the size of the
state space.\footnote{For infinite function classes ($|\cF| = \infty$),
  one can replace $\log |\cF|$ with other standard measures of
  statistical capacity, such as Rademacher complexity or metric
  entropy. For example, when $\cF$ is a class of $d$-dimensional
  linear functions, $\log|\cF|$ can be replaced by the dimension $d$, which is an upper bound on the metric entropy.}

\oldparagraph{Are realizability and concentrability sufficient?}
While realizability and concentrability are appealing in their
simplicity, these assumptions alone are not known to suffice for
sample-efficient offline RL. The most well-known line of research
\citep{munos2003error,munos2007performance,munos2008finite,antos2008learning,chen2019information}
analyzes offline RL methods such as fitted Q-iteration under the
stronger representation condition that $\cF$ is closed
under Bellman updates (``completeness''),\footnote{Precisely, $\cT\cF\subseteq\cF$, where
  $\cT$ is the \emph{Bellman operator}:
  $\brk{\cT{}f}(s,a)\ldef R(s,a)+\En_{s'\sim{}P(s,a)}\brk*{\max_{a'}f(s',a')}$.}
and obtains
$\poly(\veps^{-1}, (1-\gamma)^{-1},\Conc,\log\abs{\cF})$ sample
complexity. Completeness is a widely used assumption, but it is substantially more restrictive than
realizability and can be violated by adding a single function to
$\cF$. Subsequent years have seen extensive research into algorithmic improvements
and alternative representation and coverage conditions, but the
question of whether realizability and concentrability alone are sufficient
remains open.

\subsection{Main Results}
The first of our main results is an information-theoretic lower bound which shows
that realizability and concentrability are not sufficient for
sample-efficient offline RL.
\begin{theorem}[Main theorem]
  \label{thm:main}
  For all $S\geq{}9$ and $\gamma\in(1/2,1)$, there exists a family of
  MDPs $\cM$ with $\abs{\cS}\leq{}S$ and $\abs{\cA}=2$, a value function class
  $\cF$ with $\abs{\cF}=2$, and a \datadist
  $\mu$ such that:\loose
  \begin{enumerate}
  \item We have {$Q^{\pi}\in\cF$ for all $\pi: \cS \to \Delta(\cA)$ (all-policy realizability) and $\Conc\leq{}{16}$ (concentrability)} 
    for all models in $\cM$.
  \item Any algorithm using less than $c\cdot{}S^{1/3}$ samples must have
    $J(\pistar) - \En\brk*{J(\pihat)}\geq{}c'{/(1-\gamma)}$ for some
    instance in $\cM$, where $c$ and $c'$ are absolute numerical
    constants.
  \end{enumerate}
\end{theorem}
This result shows that even though realizability and
concentrability are satisfied, any algorithm requires at least
$\Omega(S^{1/3})$ samples to learn a near-optimal policy. Since $S$ can be
arbitrarily large, this establishes that sample-efficient offline RL in
large state spaces is impossible without stronger
representation or coverage conditions and resolves the conjecture of \citet{chen2019information}.

In fact, the theorem establishes hardness under a substantially
stronger representation condition than realizability---\emph{all
policy realizability}---which requires that $Q^\pi \in \cF$ for
\emph{every} policy $\pi$, rather than just for
$\pi^\star$. When one has the ability to interact with the MDP
  starting from the data collection distribution $\mu$ (e.g., via a
  generative model), it is known that all policy realizability and
  concentrability suffice for \emph{approximate policy iteration}
  methods~\citep{antos2008learning,lattimore2020learning}. However,
  the offline RL setting does not permit interaction, and so
  ~\pref{thm:main} yields a separation between offline RL and online RL with
  a generative model (and an exploratory distribution).  The lower
bound construction can also be extended to related settings, including
policy evaluation and linear function approximation;
see~\pref{sec:extensions} for discussion.

\pref{thm:main} relies on a strong version of the
  \emph{\overcoverage} phenomenon, where the \datadist contains states not
  visited by any admissible policy.\footnote{Note that while the
    states may not be reachable for a given MDP in the family $\cM$,
    in our construction, all states are reachable for \emph{some} MDP in the family.} The issue of \overcoverage was first noted by
  \citet{xie2021batch}, who observed that it can lead to pathological
  behavior in certain algorithms. Our result shows---somewhat
  surprisingly---that this phenomenon is a fundamental barrier that
  applies to \emph{any} value approximation method. In
  particular, we show that \overcoverage causes spurious correlations across
  reachable and unreachable states which leads to significant
  uncertainty in the dynamics when the number of states is
  large.%

\pref{thm:main} has constant suboptimality gap for $\Qstar$, which
rules out gap-dependent regret bounds as a path toward
sample-efficient offline RL. We focus on policy optimization and
infinite-horizon RL for concreteness, but the lower bound readily
extends to the finite-horizon setting (in fact, with $H=3$), and
provides, to our knowledge, the first impossibility result for offline
RL with constant horizon.

\paragraph{A lower bound for admissible data distributions}

Up to this point, we have considered the most ubiquitous formulation of
the offline RL problem, in which
$\mu\in\Delta(\cS\times\cA)$ is an arbitrary distribution over
state-action pairs. \pref{thm:main} exploits this formulation by
placing mass on states not reachable by any policy, leading to a strong
version of the \overcoverage phenomenon. Our next result shows
concentrability and realizability are still insufficient for
sample-efficient offline RL even when the data distribution $\mu$ is
\emph{admissible}, in the sense that it is induced by a policy or
mixture of policies. While strong \overcoverage is impossible in this
setting, the lower bound relies on a weak notion of \overcoverage in which
$\mu$ places significant mass on low-probability states.

\begin{restatable}[Lower bound for admissible data]{theorem}{admissible}
  \label{thm:admissible}
For any $S\ge9$, $\gamma\in(1/2,1)$, and $C\ge
64$, there exists a family of MDPs $\cM$ with $\abs{\cS}=S$ and
$\abs{\cA}=2$, a value function class $\cF$ with $\abs{\cF}=2$, and a
data distribution $\mu$ which is a mixture of admissible distributions,
such that:
\begin{enumerate}
    \item We have $Q^\pi\in\cF$ for all $\pi:\cS\rightarrow\Delta(\cA)$ (all-policy realizability) and $\Conc\le C$ (concentrability) for all models in $\cM$.
    \item Any algorithm using less than $c\cdot \min\crl*{S^{1/3}/(\log S)^2,2^{C/32},2^{1/(1-\gamma)}}$ samples must have
    $J(\pistar) - \En\brk*{J(\pihat)}\geq{} c'$ for some
    instance in $\cM$, where $c$ and $c'$ are absolute numerical
    constants. \loose %
\end{enumerate}
\end{restatable}

Compared to \pref{thm:main}, which shows that for general data
distributions any algorithm must
have sample complexity polynomial in the number of states even when
concentrability is constant, \pref{thm:admissible} shows that, for
admissible data distributions,\footnote{The fact that the data collection distribution is a mixture, is not critical for the result. It can be weakened to a single admissible distribution with realizability (rather than all-policy realizability).} any
algorithm must have sample complexity that is \emph{either} polynomial
in the number of states \emph{or} exponential in concentrability (or
the effective horizon $(1-\gamma)^{-1}$).
This result is incomparable to~\pref{thm:main} since, it is quantitatively slightly weaker from a sample complexity perspective, but stronger in that applies to admissible data distributions. 
Since admissible distributions are perhaps more natural in practice,~\pref{thm:admissible} serves as a strong impossibility result.

While we cannot rely on strong \overcoverage to
  prove~\pref{thm:admissible}, we are still able to create spurious
  correlations between a set of states that are useful for estimation
  and the remaining states, which are less useful. Indeed, our
  construction embeds a structure used in the proof of~\pref{thm:main}
  in a nested fashion, so that even an admissible data distribution
  provides insufficient information to disentangle this correlation
  and learn a near-optimal policy.

\arxiv{
\subsection{Related Work}
\label{sec:related}

\arxiv{We close this section with a detailed discussion of some of the most
  relevant related work.}
\colt{In this section we discuss additional related work not already covered.}
\paragraph{Lower bounds}
While algorithm-specific counterexamples for offline reinforcement
learning algorithms have a long history
\citep{gordon1995stable,tsitsiklis1996feature,tsitsiklis1997analysis,wang2021instabilities},
information-theoretic lower bounds are a more recent subject of
investigation. \citet{wang2020statistical}
(see also \citet{amortila2020variant})
consider the setting where $\cF$ is linear (i.e.,
$\Qstar(s,a)=\tri*{\phi(s,a),\theta}$, where $\phi(s,a)\in\bbR^{d}$ is a
known feature map). They consider a weaker coverage condition tailored
to the linear setting, which asserts that $\eigmin\prn*{
  \En_{(s,a)\sim{}\mu}\brk*{\phi(s,a)\phi(s,a)^{\trn}}
  }\geq{}\tfrac{1}{d}$, and they show that this condition and
  realizability alone are not strong enough for sample-efficient offline RL. The
  feature coverage condition is strictly weaker than
  concentrability, so this does not suffice to resolve the conjecture
  of \citet{chen2019information}. 
  Instead, the conceptual takeaway
  is that the feature coverage condition can lead to
  \emph{under-coverage} and may not be the right assumption for offline
  RL. This point is further highlighted by \citet{amortila2020variant}
  who show that in the infinite-horizon setting, the feature coverage
  condition can lead to non-identifiability in MDPs with only two states, meaning one
  cannot learn an optimal policy even with
  infinitely many samples. Concentrability places stronger restrictions on the
  \datadist and underlying dynamics and always implies identifiability when the state and
  action space are finite. Establishing impossibility of sample-efficient learning under concentrability and realizability requires very new ideas (which we provide in this paper,  via the notion of \emph{\overcoverage}).
 
 The results of \citet{wang2020statistical} and \citet{amortila2020variant} are extended by \citet{zanette2021exponential}, who provides a slightly more general
  lower bound for linear realizability. The results of \citet{zanette2021exponential} \emph{cannot} resolve the conjecture of \citet{chen2019information} either, because for the family of MDPs  constructed therein, no data distribution can satisfy concentrability, which means that the failure of algorithms can still be attributed to the failure of concentrability rather than the hardness under concentrability. There is also a parallel line of work providing  lower bounds for \emph{online} reinforcement learning with linear realizability  \citep{du2019good,weisz2021exponential,wang2021exponential}, which are based on very different constructions and techniques.

Compared to the offline RL lower bounds above
  \citep{wang2020statistical,amortila2020variant,zanette2021exponential}, our lower bounds
  have a less geometric, more information-theoretic flavor, and share
  more in common with lower bounds for sparsity and support testing in
  statistical estimation
  \citep{paninski2008coincidence,verzelen2010goodness,verzelen2018adaptive,
    canonne2020survey}. While previous work considers a relatively small state
  space but large horizon and feature dimension, we grow the state space, leading to polynomial
  dependence on $S$ in our lower bounds; the horizon is somewhat immaterial in our construction.

Another interesting feature is that while previous lower bounds
\citep{wang2020statistical,amortila2020variant,zanette2021exponential} are based on
deterministic MDPs, our constructions critically use stochastic
dynamics, which is a \emph{necessary} departure from a
technical perspective. Indeed, for \emph{any} family of deterministic
MDPs, \emph{any} data distribution satisfying concentrability (if such
a distribution exists) would enable sample-efficient learning,  simply because all
  MDPs in the family have deterministic dynamics, and the Bellman error
minimization algorithm in \citet{chen2019information} succeeds under
concentrability and realizability when the dynamics are
deterministic.\footnote{Deterministic dynamics allow
  one to avoid the well-known \emph{double sampling} problem and in particular
  cause the conditional variance in Eq. (3) of
  \citet{chen2019information} to vanish.} Therefore, any construction involving deterministic MDPs \citep{wang2020statistical,amortila2020variant,zanette2021exponential} cannot be used to establish impossibility of sample-efficient learning under
  concentrability and realizability.

  \paragraph{Upper bounds}
  Classical analyses for offline reinforcement learning algorithms 
  such as FQI
  \citep{munos2003error,munos2007performance,munos2008finite,antos2008learning}
  provide sample complexity upper bounds in terms of concentrability
  under the strong representation condition of Bellman
  completeness. The path-breaking recent work of \citet{xie2021batch}
  provides an algorithm which requires only realizability, but uses
  a stronger coverage condition (``\pushforward'') which requires that $P(s'\mid{}s,a)/\mu(s')\leq{}C$
  for all $(s,a,s')$. Our results imply that this condition cannot be
  substantially relaxed.

  A complementary line of work, primarily focusing on policy
  evaluation \citep{uehara2020minimax,xie2020q,jiang2020minimax,uehara2021finite}, 
  provides upper bounds that require only concentrability and
  realizability, but assume access to an additional \emph{weight function class}
  that is flexible enough to represent various occupancy measures for the underlying MDP. These
  results scale with the complexity of the weight function class. In
  general, the
  complexity of this class may be prohibitively large without prior
  knowledge; this is witnessed by our lower bound construction.

  \colt{
  \oldparagraph{Why aren't stronger coverage or representation conditions satisfied?} While \akedit{our construction, described in~\pref{sec:con},} satisfies
concentrability and realizability, it fails to satisfy stronger
coverage and representation conditions for which sample-efficient
upper bounds are known. This is to be expected,\arxiv{ (or else we would have
a contradiction!)} but understanding why is instructive. Here we
discuss connections to some notable conditions.

\noindent\emph{Pushforward concentrability.} The stronger notion of
concentrability that $P(s'\mid{}s,a)/\mu(s')\leq{}C$ for all
$(s,a,s')$, as used in \cite{xie2021batch}, fails to hold
because the state $Z$ is not covered by $\mu$. This presents no issue
for standard concentrability because $Z$ is not reachable starting
from $\init$.

\noindent\emph{Completeness.}
Bellman completeness requires that the value function class $\cF$
has $\cT_{M}\cF\subseteq\cF$ for all $M\in\cM$, where $\cT_{M}$ is the Bellman operator for $M$. We show in
\pref{eq:f1} that the set of optimal Q-value functions
$\crl*{Q_{M}^{\star}}_{M\in\cM}$ is small, but completeness requires that the class remains closed even when we mix and match value functions and Bellman operators from $\cM_1$ and $\cM_2$, which results in an exponentially large class in our construction. 
  To see why, first note that by Bellman optimality, we must have
$\crl*{Q_{M}^{\star}}_{M\in\cM}\subseteq\cF$ if $\cF$ is
complete. We therefore also require
$\cT_{M'}Q^{\star}_M\in\cF$ for $M\in\cM_1$ and $M'\in\cM_2$. Unlike
the optimal Q-functions, which are constant across $\cS_1$, 
the value of $\brk*{\cT_{M'}Q^{\star}_M}(s,2)$ for
$s\in\cS_1$ depends on whether $s\in{}\psub$ or $s\in\cS_1\setminus{}\psub$,
where $\psub$ is the collection of planted states for $M'$.\footnote{Recall that $f_1$ is the optimal Q-function for any $M \in \cM_1$ and consider $\cT_{M'}f_1$ where $M' \in \cM_2$ has planted set $I$. 
For $s \in I$, we have $[\cT_{M'}f_1](s,2) = (1/2\cdot 1 + 1/2 \cdot 2/3)\gamma = 5/6 \gamma$ while for $s \in \cS_1\setminus{}I$, we have $[\cT_{M'}f_1](s,2) = (1/2\cdot 1 + 1/2 \cdot 0)\gamma = 1/2\gamma$.}
As a result, there are ${S_1\choose\abs{\psub}}$ possible values for the
Bellman backup, which means that the cardinality of $\cF$ must be
exponential in $S$.

}

}
  
\arxiv{
\subsection{Preliminaries}

{For any $x\in\bbR$, let $\prn{x}_+:=\max\crl{x,0}$.} For an integer $n\in\bbN$, we let $[n]$ denote the set
  $\{1,\dots,n\}$. For a finite set $\cX$, $\unif(\cX)$ denotes the uniform distribution over $\cX$, and $\Delta(\cX)$ denotes the set of all probability distributions
  over $\cX$. For probability distributions $\bbP$ and $\bbQ$ over a measurable space
  $(\Omega,\filt)$ with a common dominating measure, we define the total variation distance as
  $\Dtv{\bbP}{\bbQ}=\sup_{A\in\filt}\abs{\bbP(A)-\bbQ(A)}
    = \frac{1}{2}\int\abs{d\bbP-d\bbQ}$
  and define the \chisquared as $\Dchis{\bbP}{\bbQ}:=\bbE_{\bbQ}\brk[\big]{\prn[\big]{\frac{d\bbP}{d\bbQ}-1}^2}=\int\frac{d\bbP^{2}}{d\bbQ}-1$
  when $\bbP\ll\bbQ$ and $+\infty$ otherwise.
}

\section{Fundamental Barriers for Offline Reinforcement Learning}
\label{sec:main}
\newcommand{\term}{\mathfrak{t}}

In this section we present the lower bound construction for \pref{thm:main} and prove the result, then discuss consequences. The proof of \pref{thm:admissible}---which can be viewed as a generalization of this result, but is somewhat more involved---is deferred to \pref{app:admissible}, with an overview given in \pref{sec:admissible_overview}.\footnote{Compared to the first version of this paper (arXiv preprint v1), the current version uses a slight modification to the \pref{thm:main} construction. The only purpose of this change is to emphasize similarity to the construction for \pref{thm:admissible}.}

\colt{
\paragraph{Preliminaries}
{For any $x\in\bbR$, let $\prn{x}_+:=\max\crl{x,0}$.} For a natural number $n\in\bbN$, we let $[n]$ denote the set
  $\{1,\dots,n\}$. For a finite set $\cX$, $\unif(\cX)$ denotes the uniform distribution over $\cX$, and $\Delta(\cX)$ denotes the set of all probability distributions
  over $\cX$. For probability distributions $\bbP$ and $\bbQ$ over a measurable space
  $(\Omega,\filt)$ with a common dominating measure, we define the total variation distance as
  $\Dtv{\bbP}{\bbQ}=\sup_{A\in\filt}\abs{\bbP(A)-\bbQ(A)}
    = \frac{1}{2}\int\abs{d\bbP-d\bbQ}$
  and define the \chisquared as $\Dchis{\bbP}{\bbQ}:=\bbE_{\bbQ}\brk[\big]{\prn[\big]{\frac{d\bbP}{d\bbQ}-1}^2}=\int\frac{d\bbP^{2}}{d\bbQ}-1$
  when $\bbP\ll\bbQ$ and $+\infty$ otherwise.
}

 \subsection{Construction: MDP Family, Value Functions,
   and Data Distribution}\label{sec:con}

We first provide our lower bound construction, which entails specifying the MDP family $\cM$, the value function class $\cF$, and the data distribution $\mu$.

 All MDPs in $\cM$ belong to a parameterized MDP family with shared transition and reward structure. In what follows, we first describe the structure of the parameterized family (\cref{sec:structure}) and provide intuition behind why this structure leads to statistical hardness (\pref{sec:intuition}). We then
provide a specific collection of parameters that gives rise to the hard
family $\cM$ (\cref{sec:parameter}) and complete the construction by specifying the value function class $\cF$ and data
distribution $\mu$ (\cref{sec:complete}).\loose

\begin{figure}[t]
    \centering
\begin{tikzpicture}
\draw[draw=black] (-2,1.5) circle (10pt) node {$\mathfrak{s}$};
\node[inner sep=0pt] (s0) at (-1.67,1.5) {};
\node[inner sep=0pt] (s0p) at (1.1,1.6) {};

\node[inner sep=0pt] (win) at (-0.76,0.5) {};
\draw (-.4,0.5) circle (10pt) node {$W$};
\node[inner sep=0pt] (wout) at (-0.07,0.5) {};
\node[inner sep=-1pt] (wtmp) at (0.95,0.5) {$+w$};
\draw [-{Latex[length=2mm,color=black]}] (wout) to[out=50, in=95,looseness=1] (0.63,0.5)  to[out=-85,in=-50,looseness=1] (wout);

\draw[draw=black, pattern=north west lines] (2,0) circle (7pt);
\node[inner sep=0pt] (s1) at (2.22,0) {};
\node[inner sep=0pt] (s1p) at (3., 0) {};
\draw[draw=black, pattern=north west lines] (2,0.75) circle (7pt);
\node[inner sep=0pt] (s2) at (2.22,0.75) {};
\node[inner sep=0pt] (s2p) at (3., 0.75) {};
\draw[draw=black] (2,1.3) circle (2pt);
\draw[draw=black] (2,1.5) circle (2pt);
\draw[draw=black] (2,1.7) circle (2pt);

\draw node at (2,-0.5) {$\textrm{States }\cS^1$};

\draw[draw=black, fill=gray, fill opacity=0.5] (2,2.25) circle (7pt);
\node[inner sep=0pt] (s5in) at (1.78,2.25) {};
\node[inner sep=0pt] (s5) at (2.22,2.25) {};
\node[inner sep=0pt] (s5p) at (3., 2.25) {};
\draw[draw=black, fill=gray, fill opacity=0.5] (2,3) circle (7pt);
\node[inner sep=0pt] (s6in) at (1.78,3) {};
\node[inner sep=0pt] (s6) at (2.22,3) {};
\node[inner sep=0pt] (s6p) at (3., 2.8) {};

\node[inner sep=0pt] (s7in) at (1.78,1.75) {};

\draw [decorate,decoration={brace},yshift=0pt] (1.0,1.67) -- (1.0,3.12) node [black,midway,xshift=-0.3cm] {$\psub$};

\node[inner sep=0pt] (yin) at (4.67,2.8) {};
\draw (5,2.8) circle (10pt) node {$X$};
\node[inner sep=0pt] (yout) at (5.33,2.8) {};
\node[inner sep=-1pt] (ytmp) at (6.35,2.75) {$+1$};

\node[inner sep=0pt] (xin) at (4.67,1.5) {};
\draw (5,1.5) circle (10pt) node {$Y$};
\node[inner sep=0pt] (xout) at (5.33,1.5) {};
\node[inner sep=-1pt] (xtmp) at (6.35,1.45) {$+0$};

\node[inner sep=0pt] (zin) at (4.67,0.2) {};
\draw (5,0.2) circle (10pt) node {$Z$};
\node[inner sep=0pt] (zout) at (5.33,0.2) {};
\node[inner sep=-1pt] (ztmp) at (6.35,0.15) {$+\frac{\alpha}{\beta}$};

\draw [draw=blue, -,line width=1.2] (s0) edge (s0p);
\draw [draw=red, -{Latex[length=2mm,color=red]}, line width=1.2] (s0) edge (win);
\draw [draw=blue, -{Latex[length=2mm,color=blue]}] (s0p) edge (s5in);
\draw [draw=blue, -{Latex[length=2mm,color=blue]}] (s0p) edge (s6in);
\draw [draw=blue, -{Latex[length=2mm,color=blue]}] (s0p) edge (s7in);

\draw [draw=black, -] (s1) edge (s1p);
\draw [draw=black, -{Latex[length=2mm,color=black]}] (s1p) edge (xin);
\draw [draw=black, -{Latex[length=2mm,color=black]}] (s1p) edge (zin);

\draw [draw=black, -] (s2) edge (s2p);
\draw [draw=black, -{Latex[length=2mm,color=black]}] (s2p) edge (xin);
\draw [draw=black, -{Latex[length=2mm,color=black]}] (s2p) edge (zin);

\draw [draw=black, -] (s5) edge (s5p);
\draw [draw=black, -{Latex[length=2mm,color=black]}] (s5p) edge (xin);
\draw [draw=black, -{Latex[length=2mm,color=black]}] (s5p) edge (yin);

\draw [draw=black, -] (s6) edge (s6p);
\draw [draw=black, -{Latex[length=2mm,color=black]}] (s6p) edge (xin);
\draw [draw=black, -{Latex[length=2mm,color=black]}] (s6p) edge (yin);

\draw [-{Latex[length=2mm,color=black]}] (yout) to[out=50, in=95,looseness=1] (6.03,2.8)  to[out=-85,in=-50,looseness=1] (yout);
\draw [-{Latex[length=2mm,color=black]}] (xout) to[out=50, in=95,looseness=1] (6.03,1.5)  to[out=-85,in=-50,looseness=1] (xout);
\draw [-{Latex[length=2mm,color=black]}] (zout) to[out=50, in=95,looseness=1] (6.03,0.2)  to[out=-85,in=-50,looseness=1] (zout);

\node [inner sep=0pt, color=black] () at (3.5, 3.05) {$\alpha$};
\node [inner sep=0pt, color=black] () at (3.5, -0.2) {$\beta$};

\end{tikzpicture}

  \caption{The MDPs in $\cM$ are parametrized
    by three scalars $\alpha,\beta,w$ and
    a subset of states $\psub$. The state space consists of an \emph{initial state} $\init$, a large number of \emph{intermediate states} $\cS^1$, and four self-looping \emph{terminal states} $\{W,X,Y,Z\}$. 
    From the initial state $\init$, action $1$ (in red) transitions to state
    $W$, while action $2$ (in blue) transitions to a subset of intermediate states $\psub\subset\cS^1$ with equal probability. 
    In all intermediate states and terminal states, actions 1 and 2 have the same effect, with transitions denoted
      in black. Among the intermediate states, $\psub \subset \cS^1$ (the gray ones) are the 
    \emph{planted states} which transition with probability $\alpha$ to state $X$ and $1-\alpha$
    to state $Y$, and the remaining $\cS^1 \setminus \psub$ (the striped ones) are the
    \emph{unplanted states} which transition 
    with probability $\beta$ to $Z$ and $(1-\beta)$ to $Y$. There are combinatorially many choices for $\psub$. Only terminal states can generate non-zero rewards: the rewards of the states $W$, $X$, $Y$ and $Z$ are $w$, $1$, $0$ and $\alpha/\beta$, respectively. }
    \label{fig:1}
\end{figure}
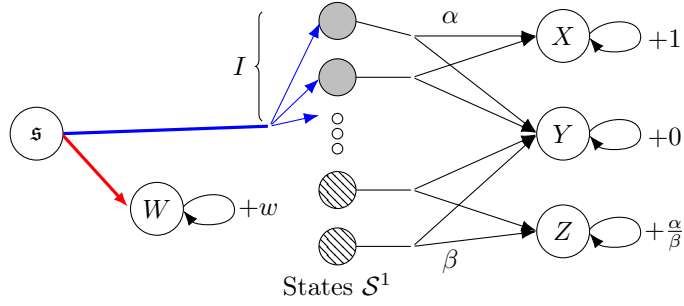

\subsubsection{MDP Parameterization}\label{sec:structure}%
\newcommand{\Mparam}{M_{\alpha,\beta,w,\psub}}%
\newcommand{\br}{\mathbf{r}}%

Let the discount factor
$\gamma\in(0,1)$ be fixed, and let $S\in\bbN$ be
given. 
Assume without loss of generality that $S>5$ and that
$(S-5)/4$ is an integer. We consider the parameterized MDP family illustrated in \cref{fig:1}. Each MDP takes the form
$M_{\alpha,\beta,w,\psub}=\prn*{\cS,\cA,P_{\alpha,\beta,\psub},R_{\alpha,\beta,w},\gamma,d_0}$,
and is parametrized by two probability parameters
$\alpha,\beta\in(0,1)$, a reward parameter
$w\in[0,1]$, and a subset of states $\psub$.  %
All MDPs in the family $\crl{M_{\alpha,\beta,w,\psub}}$ share the same state space $\cS$, action
space $\cA$,  discount factor $\gamma$, and initial state
distribution $\dnot$, and differ only in terms of the transition function $P_{\alpha,\beta,\psub}$ and the reward function $R_{\alpha,\beta,w}$.

\paragraph{State space} 
We consider a state space $\cS:=\{\init\}\cup\cS^1\cup\{W,X,Y,Z\}$,
where $\init$ is the single
\emph{initial} state (occurring at $h=0$), $W,X,Y,Z$ are  four self-looping \emph{terminal}
  states, and 
$\cS^1$ is a collection of
\emph{intermediate} (i.e., neither initial nor terminal) states which may occur between the initial state and the terminal states $\crl{X,Y,Z}$. The number of intermediate states is $S_1:=\abs{\cS^1}=S-5$ which ensures $\abs{\cS}=S$.

\paragraph{Action space} 
Our action space is given by
$\cA=\crl{1,2}$. For the initial state $\mathfrak{s}$, the two actions have distinct effects, while for all other states in $\cS\setminus\{\init\}$ both actions have identical effects. As a result, the value of a given
policy only depends on the action it selects in $\mathfrak{s}$. For \arxiv{the sake of }compactness, we use the symbol $\mathfrak{a}$ as a placeholder to denote either action when taken in $s \in \cS\setminus\{\init\}$, since the choice is immaterial.\footnote{It is conceptually simpler to consider a construction where only a single action is available in $\cS\setminus\{\init\}$ and $\cS^1$, but this is notationally more cumbersome.}

\paragraph{Transition operator} For an MDP $\Mparam$, we let
$\psub\subset\cS^1$ parameterize a subset of the {intermediate states}. We
call each $s\in \psub$ a \emph{planted state} and $s\in\widebar{\psub}:=\cS^1\setminus \psub$ an
\emph{unplanted state}. The
dynamics $P_{\alpha,\beta,\psub}$ for $M_{\alpha,\beta,w,\psub}$
are determined by $\psub$ and the parameters $\alpha,\beta\in(0,1)$ as
follows (cf. \pref{fig:1}):
    \begin{itemize}
        \item \emph{Initial state $\init$}. We define $P_{\alpha,\beta,\psub}\prn{
            \mathfrak{s},1}=\unif(\{W\})$ and 
          $P_{\alpha,\beta,\psub}\prn{
            \mathfrak{s},2}=\unif(\psub)$. 
            That is, from the
          initial state $\mathfrak{s}$, choosing action 1 makes the MDP transitions to 
          state $W$ deterministically (see the red arrow in
          \cref{fig:1}), while choosing 2 makes the MDP transitions to each
          planted state in $\psub$ with equal probability (see the
          blue arrow in \cref{fig:1}); unplanted states are not reachable.
        \item \emph{Intermediate states}. Transitions from states in
          $\cS^1$ are defined as follows.
          \begin{itemize}
            \item 
          For each planted state $s\in\psub$, define
          \[
          P_{\alpha,\beta,\psub}(s,\initac)=\alpha \text{Unif}(\{X\})+(1-\alpha)\text{Unif}(\{Y\}).
          \]
          \item 
          For each unplanted state $s\in\widebar{\psub}$, define
          \[
          P_{\alpha,\beta,\psub}(s,\initac)=\beta \text{Unif}(\{Z\})+(1-\beta)\text{Unif}(\{Y\}).
          \]

          \end{itemize}That is, the MDP
          transitions stochastically to either $\crl{X,Y}$ or $\crl{Z,Y}$,
          depending on whether the source state $s\in\cS^1$ is planted
          or unplanted; see the black straight arrows in \cref{fig:1}.
        \item \emph{Terminal states}.
          All states in $\{W,X,Y,Z\}$ self-loop indefinitely. That is $P_{\alpha,\beta,\psub}(s,{\mathfrak{a}}) = \textrm{Unif}(\{s\})$ for all $s \in \{W,X,Y,Z\}$.%
    \end{itemize}

\paragraph{Reward function} 
The initial and intermediate states have no reward, i.e., $R_{\alpha,\beta,w}(s,a)=0, \forall s\in\{\init\}\cup\cS^1,\forall
a\in\cA$. Each of
the self-looping terminal states $\crl{W,X,Y,Z}$ has a fixed reward
determined by the parameters $\alpha$, $\beta$ and $w$. In particular, we
define $R_{\alpha,\beta,w}(W,{\mathfrak{a}})=w$, $R_{\alpha,\beta,w}(X,{\mathfrak{a}})=1$, $R_{\alpha,\beta,w}(Y,{\mathfrak{a}})=0$, and
$R_{\alpha,\beta,w}(Z,\mathfrak{a})=\alpha/\beta$.

\paragraph{Initial state distribution}
All MDPs in $\Mparam$ start at $\mathfrak{s}$ deterministically
(that is, the initial state distribution $d_0$ puts all the
probability mass on $\mathfrak{s}$). Note that since $d_0$ does not
vary between instances, it may be thought of as \emph{known} to the learning algorithm.

\subsubsection{Intuition Behind the Construction}
\label{sec:intuition}

The family of MDPs $\cM$ that witnesses our lower bound is a subset of the collection $\crl{M_{\alpha,\beta,w,\psub}}$. Before specifying the family precisely, we give intuition as to why this MDP structure leads to statistical hardness for offline reinforcement learning. 

Evidently, for any MDP $M_{\alpha,\beta,w,\psub}$, there is only a single effective decision that the learner needs to make: to choose action 1 in $\init$ (whose value is completely determined by $w$) or to choose action 2 in $\init$ (whose value is completely determined by $\alpha$). In our construction of the MDP family (to be specified shortly), we keep $w$ fixed over all MDPs (i.e., we make it \emph{known} to the learner), so the only challenge left to the learner is to learn the value of $\alpha$ of the underlying MDP. As we will explain, this seemingly simple task is surprisingly hard, leading to the hardness of offline reinforcement learning. The hardness arises as a result of two general principles, \emph{planted subset structure} and (strong) \emph{over-coverage}.

\paragraph{Planted Subset Structure}\label{sec:planted}
The intermediate states in $\cS^1$ are partitioned into planted and unplanted states. Each planted state in $\psub$ (the \emph{planted subset}) transitions to $X$ and $Y$ with probability $\alpha$ and $1-\alpha$ respectively, while each unplanted state in $\widebar{\psub}$  transitions to $Z$ and $Y$ with probability $\beta$ and $1-\beta$ respectively. We call such a structure the \emph{planted subset structure}, which has two important features:
\begin{itemize}
    \item The choice of $\psub\subset\cS^1$ is combinatorial in nature\arxiv{ (for example, the number of all planted subsets of size $S_1/2$ is $S_1 \choose S_1/2$, which is exponential in $S_1$)}. %
    \item Planted and unplanted states have the same value, which only depends on $\alpha$. This holds because the rewards of $X,Y,Z$ are $1,0,\alpha/\beta$ respectively (note that $\alpha\cdot1=\beta\cdot(\alpha/\beta)$). As a result, the choice of $\psub\subset\cS^1$ does not affect the value function at all.
\end{itemize}

The first feature serves as the basis for statistical hardness and leads to the appearance of the state space size in the sample complexity lower bound. For intuition as to why, suppose we are given a batch dataset of independent examples in which a state in $\cS^1$ is selected uniformly at random and we observe a sample from the next state distribution. One can show that basic statistical inference tasks such as estimating the size $\abs{I}$ of the planted subset require $\poly(S_1)$ samples, as this entails detecting the subset based on data generated from a mixture of planted and unplanted states. For example, it is well known that testing if a distribution
  is uniform on a set $\psub\subset \brk{N}$ with $|\psub| = \Theta(N)$ versus
  uniform on all of $\brk{N}$ requires $\textrm{poly}(N)$
  samples~\citep[see e.g.,][Section 5.1]{paninski2008coincidence,ingster2012nonparametric,canonne2020survey}.

Building on this hardness, we can show that any algorithm requires at least $\poly(S_1)$ samples to reliably estimate the transition probability parameter $\alpha$ if $\beta$ and $\abs{\psub}$ are unknown. Intuitively, this arises because the only way to avoid 
estimating $\abs{\psub}$ (which is hard) as a means to estimate $\alpha$ is to directly look at the marginal distribution over $\{X,Y,Z\}$. However, the marginal distribution is uninformative for estimating $\alpha$ when there is uncertainty about $\beta$ and $\abs{\psub}$. For example, the marginal probability of transitioning to $X$ is $\alpha|\psub|/S_1$, from which $\alpha$ cannot be directly recovered if $\abs{\psub}$ is unknown.

The takeaway is that while estimating $\alpha$ would be trivial if the dataset only consisted of transitions generated from planted states, estimating this parameter when states are drawn uniformly from $\cS^1$ is very difficult because an unknown subset comes from unplanted states. This is relevant because---as we will show---{in our construction, any near-optimal policy learning algorithm must have the ability to recover the value of $\alpha$.}

The second feature, that all states in $\cS^1$ share the same value, is also essential. Since the choice of $\psub\subset\cS^1$ does not affect the value function at all, this feature allows us to consider exponentially many choices of $\psub$ while ensuring that realizability is satisfied with a value function class $\cF$ of constant size. Thus, the $\poly(\abs{\cS})$ factor in our lower bound cannot be attributed to any other problem parameter, such as $\log |\cF|$.

\paragraph{(Strong) Over-coverage}\label{sec:overcoverage}
It remains to show that the hardness described above can be embedded in the offline RL setting, since (i) we must ensure concentrability is satisfied, and (ii) the learner observes rewards, not just transitions. Returning to \pref{fig:1}, we observe that the transitions from the initial state $\nullstate$ are such that all planted states in $\psub$ are \emph{reachable}, but the unplanted states in $\cS^1\setminus \psub$ are \emph{not reachable by any policy}. In particular, since all unplanted states are unreachable, any state that can only be reached from unplanted states is also unreachable, and hence we can achieve concentrability \cref{eq:concentrability} without covering such states. This allows us to choose the data distribution $\mu$ to be (roughly) uniform over all states except for the unreachable state $Z$. This choice satisfies concentrability, but renders all reward observations uninformative (cf. \cref{sec:structure}). As a consequence, we show that the task described in \cref{sec:planted}, i.e., detecting $\alpha$ based on transition data, is unavoidable for any algorithm with non-trivial offline RL performance.

The key principle at play here is the \overcoverage phenomenon (in particular, the strong version, where $\mu$ is supported over unreachable states). Per the discussion above, we know that if the data distribution $\mu$ were supported only over reachable states for a given MDP, all ``time step 1'' examples $(s,a,r,s')$ in $D_n$ would have $s\in\psub$, which would make estimating $\alpha$ trivial. 
\arxiv{Our construction for $\mu$ is uniform over all states in $\cS^1$, and hence satisfies over-coverage, since it is supported over a mix of planted states and spurious (unplanted) states not reachable by any policy. 
This makes estimating $\alpha$ challenging because---due to correlations between planted and unplanted states---no algorithm can accurately estimate $\alpha$ or recover the planted states until the number of samples scales with the number of states. }
{We emphasize, however, that while strong \overcoverage makes the construction for \pref{thm:main} comparatively simple, the weak variant of \overcoverage 
(where all states are reachable, but the offline data distribution creates a spurious correlation by favoring unplanted states)
still presents a fundamental barrier and is the mechanism behind \pref{thm:admissible}.}

\looseness=-1

\subsubsection{Specifying the MDP Family}\label{sec:parameter}
Using the parameterized MDP family $\crl{\Mparam}$, we construct the hard family $\cM$ for our lower bound by selecting a
specific collection of values for the parameters $(\alpha,\beta,w,\psub)$. %
Define $\cI_{\theta}:=\crl{\psub : \abs{\psub}=\theta S_1}$ for all $\theta\in(0,1)$ such that $\theta S_1$ is an integer. We define two sub-families of MDPs,
\begin{align*}
\cM_1:=\bigcup_{\psub\in\cI_{\theta_1}}\crl{M_{\alpha_1,\beta_1,w,\psub}}, \quad \textrm{and} \quad 
\cM_2:=\bigcup_{\psub\in\cI_{\theta_2}}\crl{M_{\alpha_1,\beta_1,w,\psub}}, 
\end{align*}
where $w:={\gamma(\alpha_1+\alpha_2)}/2$ is fixed for all MDPs, $\cM_1$ is specified by
$(\theta_1,\alpha_1,\beta_1)=(1/2,1/4,3/4)$, and $\cM_2$ is specified by
$(\theta_2,\alpha_2,\beta_2)=(1/4,1/2,1/2)$.\footnote{Recall that we assume without
  loss of generality that $S_1/4$ is an integer.} Finally, we define
the hard family $\cM$ via
\colt{
$\cM := \cM_1 \cup \cM_2$.
}
\arxiv{\[
\cM := \cM_1 \cup \cM_2.
\]}

Let us discuss some basic properties of the construction that are
used to prove the lower bound.
\begin{itemize}
\item 
For all MDPs in $\cM$, the rewards of the terminal states $W,X,Y$ are the same. This means there is no uncertainty in the reward function outside of state $Z$, which has $R_{\alpha_1,\beta_1,w}(Z,\initac) = \alpha_1/\beta_1=1/3$ when $M\in\cM_1$ and
$R_{\alpha_2,\beta_2,w}(Z,\initac) = \alpha_2/\beta_2=1$ when $M\in\cM_2$. As we mentioned, the reward of $Z$ ($=\alpha/\beta$) is chosen to ensure that all states in $\cS^1$ have the same value ($=\gamma\alpha/(1-\gamma)$), given the choice for $(\alpha,\beta)$. %
\item 
All MDPs in $\cM_1$ (resp. $\cM_2$) differ only in the choice
of $\psub \subset \cS^1$.
This property, along with the aforementioned fact that all states in $\cS^1$ have the same value $\gamma\alpha_1/(1-\gamma)$ (resp. $\gamma\alpha_2/(1-\gamma)$), ensures that $Q_M^\star$ is the same for all $M\in\cM_1$
(resp. $M\in\cM_2$).
Furthermore, our choice for $w\in(\gamma\alpha_1,\gamma\alpha_2)$ ensures that 
the optimal action in $\init$ is action 1 (resp. action 2) for all MDPs in $\cM_1$ (resp. $\cM_2$).
\item Our choice for $(\theta_1,\alpha_1,\beta_1)$ and $(\theta_2,\alpha_2,\beta_2)$ ensures that the marginal distribution of $s'$ under the process $s\sim \unif(\cS^1), s'\sim P(s,\initac)$ is the same for all $M\in\cM$. This property is motivated by the hard inference task described in \cref{sec:planted}, which requires an uninformative marginal distribution.
\end{itemize}
The exact numerical values for the MDP parameters chosen
  above are not essential to the result. Any tuple $(\theta_1,\alpha_1,\beta_1;\theta_2,\alpha_2,\beta_2;w)$ can be used to establish a result similar to \cref{thm:main},
as long as it satisfies certain properties
described in \cref{sec:general}.\loose

\subsubsection{Finishing the Construction: Value Functions
and Data Distribution}\label{sec:complete}

We complete our construction by specifying a value function class $\cF$ that satisfies (all-policy) realizability and a data distribution $\mu$ that satisfies concentrability \cref{eq:concentrability}.

\paragraph{Value function class}
{Define functions $f_1,f_2: \cS \times \cA \to \mathbb{R}$\arxiv{ as
  follows}; differences are highlighted in
  blue:}
\begin{align}\label{eq:f1}
    f_1(s,a) := \frac{1}{1-\gamma}\cdot\begin{cases}
   {\frac{3}{8}{\gamma^2}},&s=\mathfrak{s},\;a=1\\
    {\color{blue}\frac{1}{4}\gamma^2},&{s=\init,\; a=2}\\
    {\color{blue}\frac{1}{4}\gamma},&s\in\cS^1\\
    \frac{3}{8}\gamma,&s=W\\
    1,&s=X\\
    0,&s=Y\\
    {\color{blue}\frac{1}{3}},&s=Z\end{cases}
\quad \textrm{ and } \quad
    f_2(s,a) := \frac{1}{1-\gamma}\cdot\begin{cases}
    {\frac{3}{8}\gamma^{2}},&s=\mathfrak{s},\; a=1\\
    {\color{blue}\frac{1}{2}\gamma^2},&{s=\init,\; a=2}\\
    {\color{blue}\frac{1}{2}\gamma},&s\in\cS^1\\
    \frac{3}{8}\gamma,&s=W\\
    1,&s=X\\
    0,&s=Y\\
    {\color{blue}1},&s=Z\end{cases}.
\end{align}
The following result is elementary; see \pref{sec:value} for a detailed calculation.
\begin{proposition}
  \label{prop:value_calculation}
For all $M\in\cM_1$, we have $Q_M^\pi=f_1$ for all $\pi:\cS\rightarrow\Delta(\cA)$. For all  $M\in\cM_2$, we have $Q_M^\pi=f_2$ for all $\pi:\cS\rightarrow\Delta(\cA)$.
\end{proposition}
It follows that by choosing $\cF :=
  \{f_1,f_2\}$, all-policy realizability  holds for all  $M \in
  \cM$. Note that the all-policy realizability condition ($Q_M^\pi
  \in \cF$ for all $M\in\cM$ and for \emph{all} policies $\pi$) is substantially stronger than the standard realizability condition ($Q^{\star}_M\in\cF$ for all $M\in\cM$), as it requires $Q^\pi\in\cF$ for \emph{every} policy rather than just for $\pi^\star$. Since the conjecture of \citet{chen2019information} only asks for a construction that satisfies standard realizability, by considering all-policy realizability, we are proving a \emph{stronger} hardness result. This is possible because in our construction, different actions have identical effects on all states except for the initial state $\init$; as a result, $Q^\pi$ does not depend on $\pi$ at all (in other words, our construction ensures that $Q^\pi$ is always the same as $Q^\star$).

\paragraph{Data distribution} Recall that
the learner is provided with an \iid dataset
$D_n=\crl{(s_i,a_i,r_i,s'_i)}_{i=1}^n$ where $(s_i,a_i)\sim \mu$,
$s_i'\sim P(\cdot\mid s_i,a_i)$, and $r_i=R(s_i,a_i)$ (here $P$ and $R$ are the transition and reward functions for the {underlying} MDP).
We define the data collection distribution via:
  \[
\mu := \frac{1}{8}\unif(\crl{\mathfrak{s}}\times\crl{{1,2}})+\frac{1}{2}\unif(\cS^1\times\crl{1,2})+\frac{3}{8}\unif(\crl{W,X,Y}\times\crl{{1,2}}).
\]
This choice for $\mu$ forces the learner to suffer from the hardness described in \cref{sec:planted}. Salient properties include: (i) both planted and unplanted states in $\cS^1$
are covered, and (ii)  the state $Z$
is not covered.
Property (i) results in \overcoverages, which makes estimating the parameters of the underlying MDP from transitions statistically hard, while property (ii) hides the difference between the rewards of $Z$ for the two-subfamilies of MDPs and hence makes all reward observations uninformative. 

We now verify the concentrability condition \cref{eq:concentrability}:
\begin{enumerate}
    \item[-] For time step $h=0$, for any $\pi:\cS\rightarrow\Delta(\cA)$, the
      distribution of $(s_0,a_0)$ is $d_0\times\pi$. It follows that
      \arxiv{\[\nrm*{\frac{d_0\times\pi}{\mu}}_{\infty}\le {\frac{{1}}{\frac{1}{8}\cdot\frac{1}{2}}=16}.\]}
      \colt{$\nrm*{\frac{d_0\times\pi}{\mu}}_{\infty}\le {\frac{{1}}{\frac{1}{8}\cdot\frac{1}{2}}=16}$.}
\item[-]For time step $h=1$, for any $\pi:\cS\rightarrow\Delta(\cA)$, the
  distribution of $(s_1,a_1)$ is $\unif(\psub)\times\pi$. We
  conclude that
  \arxiv{\[\nrm*{\frac{\unif(\psub)\times \pi}{\mu}}_{\infty}\le
      {\frac{\frac{1}{S_1/4}}{\frac{1}{2}\cdot \frac{1}{S_1}\cdot\frac{1}{2}}=16},\]}
  \colt{$\nrm*{\frac{\unif(\psub)\times \pi}{\mu}}_{\infty}\le
  {\frac{\frac{1}{S_1/2}}{\frac{1}{2}\cdot \frac{1}{S_1}\cdot\frac{1}{2}}=8}$,}
where we have used that $\abs{\psub}\geq{}S_1/4$.
\item[-]For time step $h\ge2$, for any $\pi:\cS\rightarrow\Delta(\cA)$, the
  distribution of $(s_h,a_h)$ (denoted by $d^\pi_h$) is supported on
  $\crl{W,X,Y}\times\crl{{1,2}}$. Therefore, we have
  \arxiv{\[\nrm*{\frac{d^\pi_h}{\mu}}_{\infty}\le {\frac{{1}}{\frac{3}{8}\cdot \frac{1}{3}\cdot \frac{1}{2} }= 16}.\]}
  \colt{$\nrm*{\frac{d^\pi_h}{\mu}}_{\infty}\le {\frac{{1}}{\frac{3}{8}\cdot \frac{1}{3}\cdot \frac{1}{2} }= 16}$.}
\end{enumerate}
We conclude that the construction satisfies concentrability with $\Conc={16}$.

\subsection{Proof of Theorem \ref*{thm:main}}

Having specified the lower bound construction, we proceed to prove \cref{thm:main}. For any MDP $M\in\cM$, we know from \cref{eq:f1} that the optimal policy $\pi_M^\star$ has
\[
\pi_M^\star(\init)=\begin{cases}
1,&\text{if }M\in\cM_1\\
2,&\text{if }M\in\cM_2
\end{cases},
\]
and that $Q_M^\star$ has a constant gap in value between the optimal and suboptimal actions in the initial state $\init$:
\begin{equation}\label{eq:gap}
Q_M^\star(\init,\pi^\star_M(\init))-Q_M^\star(\init,a)\ge\frac{1}{8}\frac{\gamma^2}{1-\gamma}, ~~~\forall a\ne\pi^\star_M(\init).
\end{equation}
This implies that any policy $\pi:\cS\rightarrow\Delta(\cA)$ with high
value must choose action 1 in
$\init$ with high probability when $M\in\cM_1$, and must choose action 2 in $\init$ with high probability when $M\in\cM_2$. As a result, any offline RL algorithm
with non-trivial performance must reliably distinguish between
$M\in\cM_1$ and $M\in\cM_2$ using the offline dataset $D_n$. In what
follows we make this intuition precise.

For each $M\in\cM$, let $\bbP^{\sM}_{n}$ denote the law of the offline
dataset $D_n$ when the underlying MDP is $M$, and let $\bbE^{\sM}_{n}$
be the associated expectation operator. We formalize the idea of
distinguishing between $M\in\cM_1$ and $M\in\cM_2$ using
\cref{lm:reduction}, which reduces the task of proving a policy
learning lower bound to the task of upper bounding the total variation distance between two \emph{mixture distributions} $\bbP_{n}^1:=\frac{1}{\abs{\cM_1}}\sum_{M\in\cM_1}\bbP^{\sM}_{n}$ and $\bbP_{n}^2:=\frac{1}{\abs{\cM_2}}\sum_{M\in\cM_2}\bbP^{\sM}_{n}$.

\begin{lemma}\label{lm:reduction}Let $\gamma\in(0,1)$ be fixed. For any offline RL algorithm which takes $D_n=\crl{(s_i,a_i,r_i,s'_i)}_{i=1}^n$ as input and returns a stochastic policy $\wh{\pi}_{D_n}:\cS\rightarrow\Delta(\cA)$, we have
\[\sup_{M\in\cM}\crl*{J_{M}(\pi_M^\star)-\bbE^{\sM}_{n}\brk*{J_{M}(\wh{\pi}_{D_n})}}\ge\frac{\gamma^2}{16(1-\gamma)}\prn*{1-\Dtv{\bbP_{n}^1}{\bbP_{n}^2}}.\]
\end{lemma}
\pref{lm:reduction} implies that if the difference between the average dataset generated by all $M\in\cM_1$ and that generated by all $M\in\cM_2$ is sufficiently small, no algorithm can reliably distinguish $M\in\cM_1$ and $M\in\cM_2$ based $D_n$, and hence must have poor performance on some instance. See \cref{sec:reduction} for a proof.%

We conclude by bounding $\Dtv{\bbP_{n}^1}{\bbP_{n}^2}$. Since directly calculating the
total variation distance is difficult, we proceed in two steps. We first design an auxiliary \emph{reference measure} $\bbP_{n}^0$, and then bound $\Dtv{\bbP_{n}^1}{\bbP_{n}^0}$ and
$\Dtv{\bbP_{n}^2}{\bbP_{n}^0}$ separately. For the latter step, we move from total
variation distance to \emph{\chisquared} and bound
$\Dchis{\bbP_{n}^1}{\bbP_{n}^0}$
(resp. $\Dchis{\bbP_{n}^2}{\bbP_{n}^0}$) using a mix of
combinatorial arguments and concentration inequalities. This
constitutes the most technical portion of the proof, and formalizes the intuition about hardness of estimation under planted subset structure described in \pref{sec:intuition}. Our final bound on the total variation distance (proven in
\cref{sec:tv}) is as follows.\loose
\begin{lemma}\label{lm:tv}
  For all $n\le \sqrt[3]{(S-5)}/20$, we have
  \arxiv{\[\Dtv{\bbP_{n}^1}{\bbP_{n}^2}\le 1/2.\]}
  \colt{$\Dtv{\bbP_{n}^1}{\bbP_{n}^2}\le 3/4$.}
\end{lemma}
\cref{thm:main} immediately follows by combining \cref{lm:reduction} and \cref{lm:tv}. %
\qed

\subsection{Discussion}

Having proven \pref{thm:main}, we briefly interpret the result and discuss some
additional consequences. We refer to \pref{sec:extensions} for extensions and further discussion.%

\paragraph{Separation between online and offline reinforcement
  learning}
In the online reinforcement learning setting, the learner can
execute any policy in the underlying MDP and observe the resulting
trajectory. Our results show that in general, the separation between
the sample complexity of online RL and offline RL can be arbitrarily
large, even when concentrability is satisfied. To see this, recall
that in the online RL setting, we can evaluate any fixed policy to
precision $\veps$ using $\poly((1-\gamma)^{-1})\cdot\veps^{-2}$ trajectories via
Monte-Carlo rollouts. Since the class $\cM$ we construct essentially only has two
possible choices for the optimal policy and has suboptimality gap
$\frac{\gamma^2}{1-\gamma}$, we can learn the optimal policy in the
online setting using $\poly((1-\gamma)^{-1})$ trajectories, with no dependence on the number of states. On the
other hand, \pref{thm:main} shows that the sample complexity of
offline RL for this family can be made arbitrarily large.

\paragraph{Linear function approximation}
The observation above is particularly salient in the context of linear
function approximation, where
$\cF=\crl*{(s,a)\mapsto\tri*{\phi(s,a),\theta}:\theta\in\bbR^{d}}$ for
a known feature map $\phi(s,a)$. Our lower bound construction {for \cref{thm:main}} can be viewed as a
special case of the linear function approximation setup with
$d=2$ by choosing $\phi(s,a)=(f_1(s,a), f_2(s,a))$. Consequently, our
results show that the separation between the complexity of offline RL and online RL with
linearly realizable function approximation can be arbitrarily large, even when the dimension
is constant. This strengthens one of the results of
\cite{zanette2021exponential}, which provides a linearly realizable
construction in which the separation between online and offline RL is
exponential with respect to dimension.\loose

\arxiv{
\oldparagraph{Why aren't stronger coverage or representation conditions satisfied?} While our construction satisfies
concentrability and realizability, it fails to satisfy stronger
coverage and representation conditions for which sample-efficient
upper bounds are known. This is to be expected,\arxiv{ (or else we would have
a contradiction!)} but understanding why is instructive. Here we
discuss connections to some notable conditions.

\noindent\emph{Pushforward concentrability.} The stronger notion of
concentrability that $P(s'\mid{}s,a)/\mu(s')\leq{}C$ for all
$(s,a,s')$, which is used in \cite{xie2021batch}, fails to hold
because the state $Z$ is not covered by $\mu$. This presents no issue
for standard concentrability because $Z$ is not reachable starting
from $\init$.

\noindent\emph{Completeness.}
Bellman completeness requires that the value function class $\cF$
has $\cT_{M}\cF\subseteq\cF$ for all $M\in\cM$, where $\cT_{M}$ is the Bellman operator for $M$. We show in
\pref{eq:f1} that the set of optimal Q-value functions
$\crl*{Q_{M}^{\star}}_{M\in\cM}$ is small, but completeness requires that the class remains closed even when we mix and match value functions and Bellman operators from $\cM_1$ and $\cM_2$, which results in an exponentially large class in our construction. 
\arxiv{To see why, first note that by Bellman optimality, we must have
$\crl*{Q_{M}^{\star}}_{M\in\cM}\subseteq\cF$ if $\cF$ is
complete. We therefore also require
$\cT_{M'}Q^{\star}_M\in\cF$ for $M\in\cM_1$ and $M'\in\cM_2$. Unlike
the optimal Q-functions, which are constant across $\cS^1$, 
the value of $\brk*{\cT_{M'}Q^{\star}_M}(s,\initac)$ for
$s\in\cS^1$ depends on whether $s\in{}\psub$ or $s\in\cS^1\setminus{}\psub$,
where $\psub$ is the collection of planted states for $M'$.\footnote{Recall that $f_1$ is the optimal Q-function for any $M \in \cM_1$ and consider $\cT_{M'}f_1$ where $M' \in \cM_2$ has planted set $\psub$. 
For $s \in \psub$, we have $[\cT_{M'}f_1](s,\initac) = (1/2\cdot 1 + 1/2 \cdot 0)\gamma = \gamma/2$ while for $s \in \cS^1\setminus{}\psub$, we have $[\cT_{M'}f_1](s,\initac) = (1/2\cdot (1/3) + 1/2 \cdot 0)\gamma = \gamma/6$.}
As a result, there are ${S_1\choose\abs{\psub}}$ possible values for the
Bellman backup, which means that the cardinality of $\cF$ must be exponential in $S$.}

}

\colt{
  \paragraph{Finite horizon}
It should be clear at this point that the lower
bound construction in \pref{thm:main} extends to the finite-horizon setting
with $H=3$ by simply removing the self-loops from the terminal
states and adjusting the optimal Q-value functions accordingly.
}

\arxiv{
\subsection{Extensions}
\label{sec:extensions}

\arxiv{\pref{thm:main} presents the simplest variant of our lower
bound for clarity of exposition. In what follows we sketch some
straightforward extensions.}
\begin{itemize}

\arxiv{\item \emph{Policy evaluation.}}
\colt{\paragraph{Policy evaluation}}
  Our lower bound immediately extends
  from policy optimization to policy evaluation. Indeed, letting
  $\pistar_1$ and $\pistar_2$ denote the optimal policies for $\cM_1$
  and $\cM_2$ respectively, we have
  $\abs{J_M(\pistar_1)-J_M(\pistar_2)}\propto\frac{\gamma^2}{1-\gamma}$
  for all $M\in\cM$, and we know that $J_{M}(\pistar_1)$ is constant across all $M\in\cM$. {It follows that any algorithm which evaluates policy $\pi_2^\star$ to
  precision $\veps\cdot{}\frac{\gamma^2}{1-\gamma}$ with probability at
  least $1-\delta$ for sufficiently small 
  numerical constants $\veps,\delta>0$ can be used to select the optimal
  policy with probability $(1-\delta)$, and thus guarantee
  $J(\pi^\star) - \En\brk*{J(\hat{\pi})} \lesssim
  \delta\frac{\gamma^2}{1-\gamma}$. Hence, such an algorithm
  must use $n=\Omega(\abs{\cS}^{1/3})$ samples by our policy optimization lower bound}.

  To formally cast this setup in the policy evaluation setting, we
  take $\Pi=\crl*{\pistar_2}$ as the class of policies to be
  evaluated, and we require a value function class $\cF$ such that
  $Q^{\pi}_{M}\in\cM$ for all $\pi\in\Pi$, $M\in\cM$. By \cref{prop:value_calculation}, it suffices to
    select $\cF=\crl[\big]{f_1,f_2}$.

  \arxiv{\item \emph{Learning an $\veps$-suboptimal policy.}}
  \colt{\paragraph{Learning an $\veps$-suboptimal policy}}
  \pref{thm:main}
  shows that for any $\gamma\in(1/2,1)$, $n\approxgeq{}S^{1/3}$ samples are required to learn a
  $\gammaconst$-optimal policy. We can extend the construction to show
  that more generally, for any $\veps\in(0,1)$, $n\approxgeq{}\frac{S^{1/3}}{\veps}$
  samples are required to learn an $\veps\cdot\gammaconst$-optimal
  policy. We modify the MDP family $M_{\alpha,\beta,w,\psub}$ by adding
  a single dummy state $\term$ with a self-loop and zero reward. The
  initial state distribution is changed so that $d_0(\term)=1-\veps$
  and $d_1(\init)=\veps$. That is, with probability $1-\veps$,
  the agent begins in $\term$ and stays there forever, collecting no
  reward, and otherwise the agent begins at $\init$ and proceeds as in the
  original construction. Analogously, we replace the original data
  distribution $\mu$ with
  $\mu'\ldef{}(1-\veps)\delta_{\term}+\veps\mu$, where
  $\delta_{\term}$ is a point mass on $\term$. This preserves the
  concentrability bound $\Conc\leq{}16$.
  This modification rescales the optimal value functions, and the conclusion of \pref{lm:reduction} is replaced by
  \[
\sup_{M\in\cM}\crl*{J_{M}(\pi_M^\star)-\bbE^{M}_{n}\brk*{J_{M}(\wh{\pi}_{D_n})}}\ge\veps\cdot{}\frac{\gamma^2}{16(1-\gamma)}\prn*{1-\Dtv{\bbP_{n}^1}{\bbP_{n}^2}}.
\]
On the other hand, since samples from the state $\term$ provide no
information about the underlying instance, the effective number of
samples is reduced to $\veps{}n$. One can make this intuition
precise and prove that $\Dtv{\bbP_{n}^1}{\bbP_{n}^2}\le 3/4$ whenever
$\veps{}n\leq{}c\cdot{}S^{1/3}$ for a numerical constant
$c$. Combining this with the previous bound yields the result.

\arxiv{\item \emph{Linear function approximation.} } As
  discussed above,~\pref{thm:main} can be viewed as a special case of
  linear function approximation with $d=2$ and $\phi(s,a) = (f_1(s,a),
  f_2(s,a))$. Compared with recent lower bounds in the linear
  setting~\citep{wang2020statistical,zanette2021exponential}, this
  result is significantly stronger in that (a) it considers a stronger
  coverage condition, (b) holds with constant dimension and constant
  effective horizon, and (c) scales with the number of states, which
  can be arbitrarily large.
\end{itemize}

\arxiv{
Lastly, it should be clear at this point that our lower
bound construction extends to the finite-horizon setting
with $H=3$ by simply removing the self-loops from the terminal
states. The only difference is that the optimal Q-value functions require a new
calculation since rewards are no longer discounted.
}

}

\section{Proof Overview for~\pref{thm:admissible}}
\label{sec:admissible_overview}

In this section we present a high-level overview of the construction
and proof for \pref{thm:admissible}.  We defer the complete proof, as
well as additional discussion, to~\pref{app:admissible}. The proof is
based on a extension of the construction used in~\pref{thm:main}. We
still use the concept of planted and unplanted states, but since the
data collection distribution must be admissible, we cannot rely on
strong \overcoverage to create spurious correlations. In particular, a
na\"{i}ve adaptation would require that the unplanted states are
reachable with sufficient probability, which would necessitate that
state $Z$ is supported by $\mu$ with sufficient probability. The
resulting construction would not lead to a meaningful lower bound, as
the reward information from $Z$ can be used to learn the optimal
policy.

To avoid this issue, we modify the construction in \pref{thm:main} to replace $Z$ with
another ``layer'' of states with planted subset structure (see \cref{sec:construction+} and \cref{fn:layer} for the precise definition of a layer). By repeating this
several times, we obtain a family of MDPs with $L > 1$ layers of
planted subset structure, connected in the manner displayed
in~\pref{fig:adm-illustration-3} (see \cref{sec:construction+} for the details). Specifically, taking action $2$ (in
blue) from the initial state $\init$, the $l^{\rm th}$ layer is selected with
probability $\propto 1/2^l$, and we transit uniformly to the states in
the $l^{\textrm{th}}$ layer. In each layer, the
planted states behave similarly to the construction for \pref{thm:main},
transitioning to terminal states $X$ and $Y$ with specific probabilities that are
chosen such that the marginal distribution provides no
information. However, except for at the last layer, the unplanted
states do not transition directly to the terminal state $Z$, but rather to the planted
states of the next layer. Overall, $Z$ can be reached with only $O(1/2^L)$ probability.

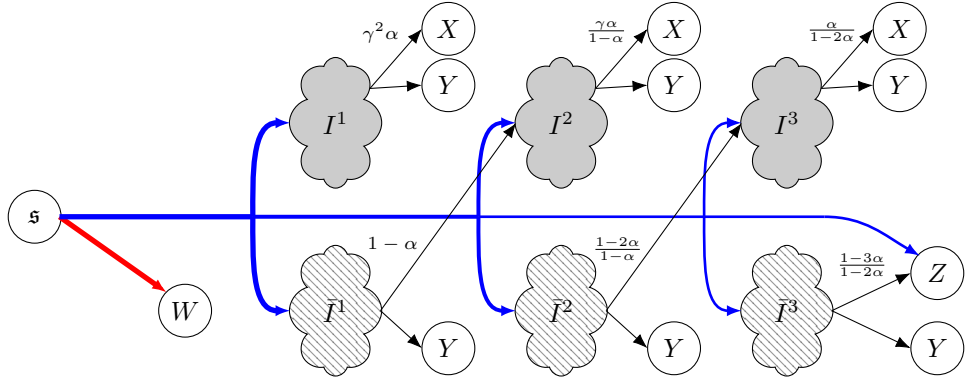
\begin{figure}[tbp]
    \centering
\newcommand{\ellipsestripe}[4]{%
  \node[cloud,cloud puffs=8,draw,minimum width=5pt, minimum height=10pt,aspect=0.5,pattern=north west lines,fill opacity=0.4, text opacity=1.0] () at (#1,#2) {#3};
  \coordinate[] (#4in) at (#1-17pt,#2) {};
  \coordinate[] (#4out) at (#1+17pt,#2) {};
  \coordinate[] (#4tr) at (#1+13pt,#2+13pt) {};
  }
\newcommand{\ellipsegray}[4]{%
  \node[cloud,cloud puffs=8,draw,minimum width=5pt, minimum height=10pt,aspect=0.5,fill=gray, fill opacity=0.4,text=black,text opacity=1.0] () at (#1,#2) {#3};
  \coordinate[] (#4in) at (#1-17pt,#2) {};
  \coordinate[] (#4out) at (#1+17pt,#2) {};
  \coordinate[] (#4tr) at (#1+13pt,#2+13pt) {};
  }
\newcommand{\xnode}[4]{
  \draw (#1,#2) circle (10pt) node {#3};
  \coordinate[] (#4in) at (#1-10pt, #2);
  }
\newcommand{\ynode}[3]{
  \draw (#1,#2) circle (10pt) node {$Y$};
  \coordinate[] (#3in) at (#1-10pt, #2);
}
\newcommand{\znode}[3]{
  \draw (#1,#2) circle (10pt) node {$Z$};
  \coordinate[] (#3in) at (#1-10pt, #2);
  \coordinate[] (#3tl) at (#1-6.5pt, #2+6.5pt);
}

\begin{tikzpicture}
\draw[draw=black] (-2,0) circle (10pt) node {$W$};
\coordinate[] (wr) at (-2.24,0.24) {};
\draw[draw=black] (-4,1.25) circle (10pt) node {$\mathfrak{s}$};
\coordinate[] (sl) at (-3.67,1.25) {};
\coordinate[] (s0) at (-3.67,1.25) {};
\coordinate[] (s01) at (-1.1,1.25) {};
\coordinate[] (s02) at (1.9,1.25) {};
\coordinate[] (s03) at (4.9,1.25) {};
\coordinate[] (s04) at (6.5,1.25) {};

\draw [draw=red, -{Latex[length=2mm,color=red]},line width=2] (sl) to (wr);

\ellipsegray{0cm}{2.5cm}{$I^1$}{p1}
\ellipsestripe{0cm}{0cm}{$\bar{I}^1$}{u1}
\xnode{1.5cm}{3.75cm}{$X$}{x1}
\ynode{1.5cm}{3.0cm}{y1t}
\ynode{1.5cm}{-0.5cm}{y1b}
\draw [draw=blue, -, line width=2] (s0) edge (s01);
\draw [draw=blue, -{Latex[length=2mm,color=blue]}, line width=2] (s01) to [out=90,in=180] (p1in);
\draw [draw=blue, -{Latex[length=2mm,color=blue]}, line width=2] (s01) to [out=270,in=180] (u1in);
\draw [draw=black, -{Latex[length=2mm,color=black]}] (p1tr) to  (x1in);
\node at (1.5cm-0.9cm,3.7) {\scriptsize $\gamma^{2}\alpha$};
\draw [draw=black, -{Latex[length=2mm,color=black]}] (p1tr) to  (y1tin);
\draw [draw=black, -{Latex[length=2mm,color=black]}] (u1out) to  (y1bin);

\ellipsegray{3cm}{2.5cm}{$I^2$}{p2}
\ellipsestripe{3cm}{0cm}{$\bar{I}^2$}{u2}
\xnode{4.5cm}{3.75cm}{$X$}{x2}
\ynode{4.5cm}{3cm}{y2t}
\ynode{4.5cm}{-0.5cm}{y2b}
\draw [draw=blue, -, line width=1.5] (s01) edge (s02);
\draw [draw=blue, -{Latex[length=2mm,color=blue]}, line width=1.5] (s02) to [out=90,in=180] (p2in);
\draw [draw=blue, -{Latex[length=2mm,color=blue]}, line width=1.5] (s02) to [out=270,in=180] (u2in);
\draw [draw=black, -{Latex[length=2mm,color=black]}] (u1out) -- (p2in);
\node at (0.76cm,0.88cm) {\scriptsize $1-\alpha$};
\draw [draw=black, -{Latex[length=2mm,color=black]}] (p2tr) to  (x2in);
\node at (4.5cm-0.9cm,3.7) {\scriptsize $\frac{\gamma\alpha}{1-\alpha}$};
\draw [draw=black, -{Latex[length=2mm,color=black]}] (p2tr) to  (y2tin);
\draw [draw=black, -{Latex[length=2mm,color=black]}] (u2out) to  (y2bin);

\ellipsegray{6cm}{2.5cm}{$I^3$}{p3}
\ellipsestripe{6cm}{0cm}{$\bar{I}^3$}{u3}
\xnode{7.5cm}{3.75cm}{$X$}{x3}
\ynode{7.5cm}{3cm}{y3t}
\ynode{8cm}{-0.5cm}{y3b}
\znode{8cm}{0.5cm}{z}

\draw [draw=blue, -, line width=1] (s02) edge (s03);
\draw [draw=blue, -{Latex[length=2mm,color=blue]}, line width=1] (s03) to [out=90,in=180] (p3in);
\draw [draw=blue, -{Latex[length=2mm,color=blue]}, line width=1] (s03) to [out=270,in=180] (u3in);
\draw [draw=black, -{Latex[length=2mm,color=black]}] (u2out) -- (p3in);
\node at (3.75cm,0.81cm) {\scriptsize $\frac{1-2\alpha}{1-\alpha}$};
\draw [draw=black, -{Latex[length=2mm,color=black]}] (p3tr) to  (x3in);
\node at (7.5cm-0.9cm,3.7) {\scriptsize $\frac{\alpha}{1-2\alpha}$};
\draw [draw=black, -{Latex[length=2mm,color=black]}] (p3tr) to  (y3tin);
\draw [draw=black, -{Latex[length=2mm,color=black]}] (u3out) to  (y3bin);
\draw [draw=black, -{Latex[length=2mm,color=black]}] (u3out) to  (zin);
\node at (7cm,0.6cm) {\scriptsize $\frac{1-3\alpha}{1-2\alpha}$};

\draw [draw=blue, line width=1] (s03) to [out=0,in=180] (s04);
\draw [draw=blue, -{Latex[length=2mm,color=blue]}, line width=1] (s04) to [out=0,in=150] (ztl);

\end{tikzpicture}

    \caption{
      Illustration of the MDP family used to prove~\pref{thm:admissible}, with $L = 3$ layers of the planted subset structure
      (note that in general, we take $L>3$). The rewards of the states
      $W$, $X$, $Y$ and $Z$ are $w$, $1$, $0$ and
      $\alpha/(1-L\alpha)$, respectively, where $w$ and $\alpha$ are
      parameters of the MDP family.}
    \label{fig:adm-illustration-3}
\end{figure}

Similar to our previous construction, the new multi-layer
  construction ensures that every MDP in the family differs only in terms of the
  reward of $Z$ and the transition probabilities for planted and unplanted states. Moreover, while
  the state $Z$ is no longer unreachable, we know that since all policies only
  reach $Z$ with exponentially small (in $L$) probability, we can
  satisfy concentrability with a data collection distribution that
  places exponentially small mass on $Z$ (which---if it appeared in the
  dataset---would reveal the optimal policy). As a result, we have that with high
  probability, all reward observations in the dataset provide no information.
Intuitively, this allows us to apply an inductive argument to show that one
  cannot learn the value function.\footnote{The inductive argument is discussed here mainly for providing intuition. The proof in \cref{app:admissible} is more direct and does not involve a formal inductive argument.} As the base case, when the reward
  for $Z$ is unobserved (which happens with high probability), the
  $L^{\textrm{th}}$ layer resembles an instance of the construction
  used to prove~\pref{thm:main}. Then, going backwards, if one cannot estimate the
  $(l+1)^{\mathrm{st}}$ layer, we can view the $l^{\mathrm{th}}$ layer
  as an instance of the previous construction to show that one cannot
  estimate the value of this layer as well. This induction relies on
  the delicate design of the data collection distribution $\mu$, which
  is supported on both planted and unplanted states, but
  nevertheless exhibits a weak notion of \overcoverage resulting in
  spurious correlations. The argument also requires gradually
  decreasing the difference in the value function (between the two MDP
  families) from the $L^{\textrm{th}}$ layer to the first layer;
  however, we can ensure that the rate of decrease is very slow, which
  leads to statistical hardness.

On a technical level, after constructing the MDP family, many of the
calculations are similar to those used to
prove~\pref{thm:main}. Analogously
to~\pref{lm:reduction}, we lower bound the suboptimality of any
algorithm by the total variation distance between two mixture
distributions. Then we bound this TV distance by constructing
auxiliary reference measures and passing to the $\chi^2$-divergence,
analogously to~\pref{lm:tv}. Finally, since we rely on a similar
planted subset structure, the $\chi^2$-divergence calculation shares
many technical elements with the proof of~\pref{lm:tv}.

\section{Conclusion}
\label{sec:conclusion}
We have proven that concentrability and realizability alone are not
sufficient for sample-efficient offline reinforcement learning,
resolving the conjecture of \cite{chen2019information}. Our results
establish that sample-efficient offline RL requires coverage or
representation conditions beyond what is required for
supervised learning and show that \overcoverage is a fundamental barrier for offline RL. \loose

For future research, an immediate question is whether it is possible to
circumvent our lower bound by considering trajectory-based data rather
than $(s,a,r,s')$ tuples. More broadly, while our results elucidate the role of
concentrability and realizability, it remains to obtain a sharp,
distribution-dependent characterization for the sample complexity of offline RL with general
function approximation. %
Such a characterization would need to recover our result and previous results---both positive
and negative---as special cases.

\arxiv{
\subsection*{Acknowledgements}
We thank Nan Jiang for insightful discussions and helpful feedback on a draft of the manuscript.
}

\bibliography{refs}

\clearpage

\renewcommand{\contentsname}{Contents of Appendix}
\tableofcontents
\addtocontents{toc}{\protect\setcounter{tocdepth}{2}}

\clearpage

\appendix

\colt{
\section{Additional Related Work}
\label{sec:related}

\arxiv{We close this section with a detailed discussion of some of the most
  relevant related work.}
\colt{In this section we discuss additional related work not already covered.}
\paragraph{Lower bounds}
While algorithm-specific counterexamples for offline reinforcement
learning algorithms have a long history
\citep{gordon1995stable,tsitsiklis1996feature,tsitsiklis1997analysis,wang2021instabilities},
information-theoretic lower bounds are a more recent subject of
investigation. \citet{wang2020statistical}
(see also \citet{amortila2020variant})
consider the setting where $\cF$ is linear (i.e.,
$\Qstar(s,a)=\tri*{\phi(s,a),\theta}$, where $\phi(s,a)\in\bbR^{d}$ is a
known feature map). They consider a weaker coverage condition tailored
to the linear setting, which asserts that $\eigmin\prn*{
  \En_{(s,a)\sim{}\mu}\brk*{\phi(s,a)\phi(s,a)^{\trn}}
  }\geq{}\tfrac{1}{d}$, and they show that this condition and
  realizability alone are not strong enough for sample-efficient offline RL. The
  feature coverage condition is strictly weaker than
  concentrability, so this does not suffice to resolve the conjecture
  of \citet{chen2019information}. 
  Instead, the conceptual takeaway
  is that the feature coverage condition can lead to
  \emph{under-coverage} and may not be the right assumption for offline
  RL. This point is further highlighted by \citet{amortila2020variant}
  who show that in the infinite-horizon setting, the feature coverage
  condition can lead to non-identifiability in MDPs with only two states, meaning one
  cannot learn an optimal policy even with
  infinitely many samples. Concentrability places stronger restrictions on the
  \datadist and underlying dynamics and always implies identifiability when the state and
  action space are finite. Establishing impossibility of sample-efficient learning under concentrability and realizability requires very new ideas (which we provide in this paper,  via the notion of \emph{\overcoverage}).
 
 The results of \citet{wang2020statistical} and \citet{amortila2020variant} are extended by \citet{zanette2021exponential}, who provides a slightly more general
  lower bound for linear realizability. The results of \citet{zanette2021exponential} \emph{cannot} resolve the conjecture of \citet{chen2019information} either, because for the family of MDPs  constructed therein, no data distribution can satisfy concentrability, which means that the failure of algorithms can still be attributed to the failure of concentrability rather than the hardness under concentrability. There is also a parallel line of work providing  lower bounds for \emph{online} reinforcement learning with linear realizability  \citep{du2019good,weisz2021exponential,wang2021exponential}, which are based on very different constructions and techniques.

Compared to the offline RL lower bounds above
  \citep{wang2020statistical,amortila2020variant,zanette2021exponential}, our lower bounds
  have a less geometric, more information-theoretic flavor, and share
  more in common with lower bounds for sparsity and support testing in
  statistical estimation
  \citep{paninski2008coincidence,verzelen2010goodness,verzelen2018adaptive,
    canonne2020survey}. While previous work considers a relatively small state
  space but large horizon and feature dimension, we grow the state space, leading to polynomial
  dependence on $S$ in our lower bounds; the horizon is somewhat immaterial in our construction.

Another interesting feature is that while previous lower bounds
\citep{wang2020statistical,amortila2020variant,zanette2021exponential} are based on
deterministic MDPs, our constructions critically use stochastic
dynamics, which is a \emph{necessary} departure from a
technical perspective. Indeed, for \emph{any} family of deterministic
MDPs, \emph{any} data distribution satisfying concentrability (if such
a distribution exists) would enable sample-efficient learning,  simply because all
  MDPs in the family have deterministic dynamics, and the Bellman error
minimization algorithm in \citet{chen2019information} succeeds under
concentrability and realizability when the dynamics are
deterministic.\footnote{Deterministic dynamics allow
  one to avoid the well-known \emph{double sampling} problem and in particular
  cause the conditional variance in Eq. (3) of
  \citet{chen2019information} to vanish.} Therefore, any construction involving deterministic MDPs \citep{wang2020statistical,amortila2020variant,zanette2021exponential} cannot be used to establish impossibility of sample-efficient learning under
  concentrability and realizability.

  \paragraph{Upper bounds}
  Classical analyses for offline reinforcement learning algorithms 
  such as FQI
  \citep{munos2003error,munos2007performance,munos2008finite,antos2008learning}
  provide sample complexity upper bounds in terms of concentrability
  under the strong representation condition of Bellman
  completeness. The path-breaking recent work of \citet{xie2021batch}
  provides an algorithm which requires only realizability, but uses
  a stronger coverage condition (``\pushforward'') which requires that $P(s'\mid{}s,a)/\mu(s')\leq{}C$
  for all $(s,a,s')$. Our results imply that this condition cannot be
  substantially relaxed.

  A complementary line of work, primarily focusing on policy
  evaluation \citep{uehara2020minimax,xie2020q,jiang2020minimax,uehara2021finite}, 
  provides upper bounds that require only concentrability and
  realizability, but assume access to an additional \emph{weight function class}
  that is flexible enough to represent various occupancy measures for the underlying MDP. These
  results scale with the complexity of the weight function class. In
  general, the
  complexity of this class may be prohibitively large without prior
  knowledge; this is witnessed by our lower bound construction.

  \colt{
  \oldparagraph{Why aren't stronger coverage or representation conditions satisfied?} While \akedit{our construction, described in~\pref{sec:con},} satisfies
concentrability and realizability, it fails to satisfy stronger
coverage and representation conditions for which sample-efficient
upper bounds are known. This is to be expected,\arxiv{ (or else we would have
a contradiction!)} but understanding why is instructive. Here we
discuss connections to some notable conditions.

\noindent\emph{Pushforward concentrability.} The stronger notion of
concentrability that $P(s'\mid{}s,a)/\mu(s')\leq{}C$ for all
$(s,a,s')$, as used in \cite{xie2021batch}, fails to hold
because the state $Z$ is not covered by $\mu$. This presents no issue
for standard concentrability because $Z$ is not reachable starting
from $\init$.

\noindent\emph{Completeness.}
Bellman completeness requires that the value function class $\cF$
has $\cT_{M}\cF\subseteq\cF$ for all $M\in\cM$, where $\cT_{M}$ is the Bellman operator for $M$. We show in
\pref{eq:f1} that the set of optimal Q-value functions
$\crl*{Q_{M}^{\star}}_{M\in\cM}$ is small, but completeness requires that the class remains closed even when we mix and match value functions and Bellman operators from $\cM_1$ and $\cM_2$, which results in an exponentially large class in our construction. 
  To see why, first note that by Bellman optimality, we must have
$\crl*{Q_{M}^{\star}}_{M\in\cM}\subseteq\cF$ if $\cF$ is
complete. We therefore also require
$\cT_{M'}Q^{\star}_M\in\cF$ for $M\in\cM_1$ and $M'\in\cM_2$. Unlike
the optimal Q-functions, which are constant across $\cS_1$, 
the value of $\brk*{\cT_{M'}Q^{\star}_M}(s,2)$ for
$s\in\cS_1$ depends on whether $s\in{}\psub$ or $s\in\cS_1\setminus{}\psub$,
where $\psub$ is the collection of planted states for $M'$.\footnote{Recall that $f_1$ is the optimal Q-function for any $M \in \cM_1$ and consider $\cT_{M'}f_1$ where $M' \in \cM_2$ has planted set $I$. 
For $s \in I$, we have $[\cT_{M'}f_1](s,2) = (1/2\cdot 1 + 1/2 \cdot 2/3)\gamma = 5/6 \gamma$ while for $s \in \cS_1\setminus{}I$, we have $[\cT_{M'}f_1](s,2) = (1/2\cdot 1 + 1/2 \cdot 0)\gamma = 1/2\gamma$.}
As a result, there are ${S_1\choose\abs{\psub}}$ possible values for the
Bellman backup, which means that the cardinality of $\cF$ must be
exponential in $S$.

}

}

\part{Proofs for Theorem \ref*{thm:main}}

\section{General Scheme to Construct Hard Families of Instances}\label{sec:general}
Recall that \pref{sec:parameter} gives specific numerical values for the parameters that
define the model class $\cM$ used in our lower bound construction. The
precise values are not critical for our proof, and in this section we
give general conditions on the parameters under which one can derive a
similar lower bound. In doing so, we also provide some intuition
behind the specific choice of parameters used for \pref{thm:main}.

In more detail, for any tuple of parameters $(\theta_1,\alpha_1,\beta_1;\theta_2,\alpha_2,\beta_2;w)$,
we consider the family of MDPs $\cM$ given by
\[
\cM_1:=\bigcup_{\psub\in\cI_{\theta_1}}M_{\alpha_1,\beta_1,w,\psub},~~\cM_2:=\bigcup_{\psub\in\cI_{\theta_2}}M_{\alpha_2,\beta_2,w,\psub},~~ \cM:=\cM_1\cup\cM_2. 
\]
There are 7 independent scalars in the tuple, all of which lie in $\brk{0,1}$:
$\theta_1,\alpha_1,\beta_1,\theta_2,\alpha_2,\beta_2,w$;
note that the parameter $w$ is shared between
  $\cM_1$ and $\cM_2$. The family $\cM$ above can be
  used to derive a hardness result similar to \cref{thm:main} as long
  as the following three general equality and inequality constraints are satisfied.
\begin{itemize}
    \item All $M\in\cM$ have the same marginal distribution for $s'$ under the process $s\sim \textrm{Unif}(\cS^1)$, $s'\sim
      P( s,\initac)$:
    \begin{equation}\label{eq:marginal}
    \theta_1\alpha_1=\theta_2\alpha_2,\mathand
    (1-\theta_1)\beta_1=(1-\theta_2)\beta_2.
  \end{equation}
  This ensures that the learner cannot trivially test whether
  $M\in\cM_1$ or $\cM_2$ using marginals, which is tacitly used in the
  proof of \pref{lm:tv}.
\item The parameters
  $\theta_1,\alpha_1,\beta_1,\theta_2,\alpha_2,\beta_2$ are bounded away from $0$ and $1$:
\begin{equation}\label{eq:interior}
    {\theta_1, \alpha_1, \beta_1,\theta_2,\alpha_2,\beta_2}\in (0,1).
  \end{equation}
  In particular, the distance from the boundary should be a constant independent of $\frac{1}{\abs{\cS}}$ and $\gamma$.
\item In state $\init$ (i.e., the only state where the two actions have distinct effects), action 1 is strictly better (resp. worse) than action 2 if $M\in\cM_1$ (resp. $M\in\cM_2$), which means
$w/(1-\gamma)=Q^*_M(\init,1)>Q^*_M(\init,2)=\gamma\alpha_1(1-\gamma)$ (resp. $w/(1-\gamma)=Q^*_M(\init,1)<Q^*_M(\init,2)=\gamma\alpha_2(1-\gamma)$). This means
\begin{equation}\label{eq:different}
  \gamma\alpha_1 <w< \gamma\alpha_2.
\end{equation}
The final lower bound depends on this separation quantitatively.
\end{itemize}
Any tuple simultaneously satisfying \cref{eq:marginal,eq:interior,eq:different} is sufficient for our proof (modulo numerical differences). Naturally, the numerical values for the function class $\cF$ defined in \pref{eq:f1} must be changed accordingly so that the class contains $Q^{\star}$ for both $\cM_1$ and $\cM_2$.

\section{Computation of Value Functions (Proposition \ref*{prop:value_calculation})}\label{sec:value}
In this section, we verify \pref{prop:value_calculation}, which asserts
that for all $\pi$, $Q_M^\pi=f_1$ for all $M\in\cM_1$ and $Q_M^\pi=f_2$ for all
$M\in\cM_2$, where $f_1$ and $f_2$ are defined in \pref{eq:f1}.  Note
  that the calculation we present here is based on the precise values for the
  parameters $(\theta_1,\alpha_1,\beta_1;
  \theta_2,\alpha_2,\beta_2;w)$ given in
  \cref{sec:parameter}, not the general scheme given in \pref{sec:general}.

  \begin{proof}[\pfref{prop:value_calculation}]
    Suppose $M\in\cM_1$. 
    Let $\psub_M$ denote the planted subset
    associated with $M$. First, for any self-looping terminal state
    $s\in\{W,X,Y,Z\}$, since all actions in $\cA$ have identical effects, we have
    \[
      V_M^\pi(s)=Q_M^\pi(s,\initac)= \sum_{h=0}^\infty \gamma^h
      R_{\alpha_1,\beta_1,w}(s,\initac)=\frac{1}{1-\gamma}\cdot \begin{cases}
        \frac{3}{8}\gamma,&s=W\\
        1,&s=X\\
        0,&s=Y\\
        \frac{1}{3},&s=Z\end{cases}
    \]
    for all $\pi:\cS\rightarrow\Delta(\cA)$, where we utilize the fact that  $R_{\alpha_1,\beta_1,w}(W,\initac)=w=\gamma(\alpha_1+\alpha_2)/2=3\gamma/8$ and  $R_{\alpha_1,\beta_1,w}(Z,\initac)=\alpha_1/\beta_1=1/3$.
    
    Next, for any intermediate state $s\in\cS^1$, since all actions in $\cA$ have identical effects, we have
    \begin{align*}
      V_M^\pi(s)=Q_M^\pi(s,\initac)&= R_{\alpha_1,\beta_1,w}(s,\initac)+\gamma\bbE_{s'\sim P_{\alpha_1,\beta_1,w,\psub_M}(s,\initac)}\brk{V_M^\pi(s')}\\
                    &=\begin{cases}
                      0+\gamma\brk*{ \alpha_1 V_M^\pi(X)+(1-\alpha_1) V_M^\pi(Y)}, &s\in \psub_M\\
                      0+\gamma\brk*{ \beta_1 V_M^\pi(Z)+(1-\beta_1) V_M^\pi(Y)}, &s\in\cS^1\setminus \psub_M
                    \end{cases}\\
                    &=\begin{cases}
                      \frac{\gamma}{1-\gamma}(\frac{1}{4}\times 1+ \frac{3}{4}\times0),&s\in \psub_M\\
                      \frac{\gamma}{1-\gamma}(\frac{3}{4}\times \frac{1}{3}+
                      \frac{1}{4}\times0),&s\in\cS^1\setminus \psub_M
                    \end{cases}\\
                    &=\frac{\gamma}{1-\gamma}\frac{1}{4}
    \end{align*}
    for all $\pi:\cS\rightarrow\Delta(\cA)$.

    Thus, for the initial state $\nullstate$, we have
    \[
      Q_M^\pi(\nullstate,1)=
      R_{\alpha_1,\beta_1,w}(\nullstate,1)+\gamma\bbE_{s'\sim
        P_{\alpha_1,\beta_1,w,\psub_M}(\nullstate,1)}\brk{V_M^\pi(s')}=0+\gamma
      V_M^\pi(W)=\frac{\gamma^2}{1-\gamma}\frac{3}{8},
    \]
    \[
      Q_M^\pi(\nullstate,2)=
      R_{\alpha_1,\beta_1,w}(\nullstate,2)+\gamma\bbE_{s'\sim
        P_{\alpha_1,\beta_1,w,\psub_M}(\nullstate,2)}\brk{V_M^\pi(s')}=0+\gamma
      \bbE_{s'\sim\textrm{Unif}(I_M)}V_M^\pi(s')=\frac{\gamma^2}{1-\gamma}\frac{1}{4}
    \]
    for all $\pi:\cS\rightarrow\Delta(\cA)$.
    
    Therefore, $Q_M^\pi(s,a)=f_1(s,a)$ for all
    $(s,a)\in\cS\times \cA$, for all $\pi:\cS\rightarrow\Delta(\cA)$.

    Now suppose $M\in\cM_2$. Let $\psub_M$ denote the planted subset
    associated with $M$. For any self-looping terminal state
    $s\in\{W,X,Y,Z\}$, since all actions in $\cA$ have identical effects, we have
    \[
      V_M^\pi(s)=Q_M^\pi(s,\initac)= \sum_{h=0}^\infty \gamma^h
      R_{\alpha_2,\beta_2,w}(s,\initac)=\frac{1}{1-\gamma}\cdot \begin{cases}
        \frac{3}{8}\gamma,&s=W\\
        1,&s=X\\
        0,&s=Y\\
        1,&s=Z\end{cases}
    \]
    for all $\pi:\cS\rightarrow\Delta(\cA)$, where we utilize the fact that  $R_{\alpha_2,\beta_2,w}(W,\initac)=w=\gamma(\alpha_1+\alpha_2)/2=3\gamma/8$ and  $R_{\alpha_2,\beta_2,w}(Z,\initac)=\alpha_2/\beta_2=1$.
    For any intermediate state $s\in\cS^1$, since all actions in $\cA$ have identical effects, we have
    \begin{align*}
      V_M^\pi(s)=Q_M^\pi(s,\initac)&= R_{\alpha_2,\beta_2,w}(s,\initac)+\gamma\bbE_{s'\sim P_{\alpha_2,\beta_2,w,\psub_M}(s,\initac)}\brk{V_M^\pi(s')}\\
                    &=\begin{cases}
                      0+\gamma\brk*{ \alpha_2 V_M^\pi(X)+(1-\alpha_2) V_M^\pi(Y)}, &s\in \psub_M\\
                      0+\gamma\brk*{ \beta_2 V_M^\pi(Z)+(1-\beta_2) V_M^\pi(Y)}, &s\in\cS^1\setminus \psub_M
                    \end{cases}\\
                    &=\begin{cases}
                      \frac{\gamma}{1-\gamma}(\frac{1}{2}\times 1+ \frac{1}{2}\times0),&s\in \psub_M\\
                      \frac{\gamma}{1-\gamma}(\frac{1}{2}\times 1+
                      \frac{1}{2}\times0),&s\in\cS^1\setminus \psub_M
                    \end{cases}\\
                    &=\frac{\gamma}{1-\gamma}\frac{1}{2}
    \end{align*}
    for all $\pi:\cS\rightarrow\Delta(\cA)$.  Thus, for the initial state $\nullstate$, we have
    \[
      Q_M^\pi(\nullstate,1)=
      R_{\alpha_2,\beta_2,w}(\nullstate,1)+\gamma\bbE_{s'\sim
        P_{\alpha_2,\beta_2,w,\psub_M}(\nullstate,1)}\brk{V_M^\pi(s')}=0+\gamma
      V_M^\pi(W)=\frac{\gamma^2}{1-\gamma}\frac{3}{8},
    \]
    \[
      Q_M^\pi(\nullstate,2)=
      R_{\alpha_2,\beta_2,w}(\nullstate,2)+\gamma\bbE_{s'\sim
        P_{\alpha_2,\beta_2,w,\psub_M}(\nullstate,2)}\brk{V_M^\pi(s')}=0+\gamma
      \bbE_{s'\sim\textrm{Unif}(I_M)}V_M^\pi(s')=\frac{\gamma^2}{1-\gamma}\frac{1}{2}
    \]
    for all $\pi:\cS\rightarrow\Delta(\cA)$.
    It follows that $Q_M^\star(s,a)=f_2(s,a)$ for all
    $(s,a)\in\cS\times \cA$, for all $\pi:\cS\rightarrow\Delta(\cA)$.
  \end{proof}

\section{Proof of Lemma \ref*{lm:reduction}}\label{sec:reduction}

We now prove \pref{lm:reduction}. Before proceeding, let us note
  that this lemma is proven only for the precise values for the
  parameters $(\theta_1,\alpha_1,\beta_1;
  \theta_2,\alpha_2,\beta_2;w)$ given in
  \cref{sec:parameter}. One could establish a more general lemma using
  the generic parameters introduced in \cref{sec:general}, but this
  would require changing the numerical constants appearing in the statement.

We begin the proof by lower bounding the regret for any MDP in the family $\cM$. For
any $i\in\crl{1,2}$, any MDP $M\in\cM_i$, and any policy
$\pi:\cS\rightarrow\Delta(\cA)$, we have
\begin{align}\label{eq:lower}
    J_{M}(\pistar_{\sM})-J_{M}({\pi})&={Q^\star_{\sM}(\init,\pistar_{\sM}(\init))-Q^\pi_{\sM}(\init,{\pi}(\init))}\notag\\ 
    &= {Q^\star_{\sM}(\init,\pistar_{\sM}(\init))-Q^\star_{\sM}(\init,{\pi}(\init))}\notag\\
    &= {Q^\star_{\sM}(\init,i)-Q^\star_{\sM}(\init,{\pi}(\init))}\notag\\
    &\ge\frac{\gamma^2}{8(1-\gamma)}\bbP(\pi(\init)\ne i),
\end{align}
where the second equality follows because $\init$ is the only state where different actions have distinct effects, and the last inequality follows from \cref{eq:gap}.

Now, consider any fixed offline reinforcement learning algorithm which takes
the offline dataset $D_n$ as an input and returns a stochastic policy
$\wh{\pi}_{\Dn}: \cS\rightarrow\Delta(\cA)$.
{For each $i \in \{1,2\}$, we apply} \cref{eq:lower} to all MDPs in
$\cM_i$ and average to obtain
\begin{align*}
    \frac{1}{\abs{\cM_i}}\sum_{M\in\cM_i}\bbE^{M}_{n}\brk*{J_{M}(\pistar_{\sM})-J_{M}(\wh{\pi}_{D_n})}\ge \frac{\gamma^2}{8(1-\gamma)} \frac{1}{\abs{\cM_i}}\sum_{M\in\cM_{i}}\bbP^{M}_{n}(\wh{\pi}_{D_n}(\init)\ne i).
\end{align*}
Applying the inequality above for $i=1$ and $i=2$ and combining the results, we have
\begin{align*}%
    &\max_{M\in\cM}\bbE^{M}_{n}\brk*{J_{M}(\pistar_{\sM})-J_{M}(\wh{\pi}_{D_n})}\notag\\
    &\ge \frac{1}{2\abs{\cM_{1}}}\sum_{M\in\cM_{1}}\bbE^{M}_{n}\brk*{J_{M}(\pistar_{\sM})-J_{M}(\wh{\pi}_{D_n})}+ \frac{1}{2\abs{\cM_{2}}}\sum_{M\in\cM_{2}}\bbE^{M}_{n}\brk*{J_{M}(\pistar_{\sM})-J_{M}(\wh{\pi}_{D_n})}\notag\\
    &\ge\frac{\gamma^2}{16(1-\gamma)}\crl*{\frac{1}{\abs{\cM_{1}}}\sum_{M\in\cM_{1}}\bbP^{M}_{n}(\wh{\pi}_{D_n}(\init)\ne 1)+ \frac{1}{\abs{\cM_{2}}}\sum_{M\in\cM_{2}}\bbP^{M}_{n}(\wh{\pi}_{D_n}(\init)\ne 2)}\notag\\
    &\ge \frac{\gamma^2}{16(1-\gamma)}\prn*{1-\Dtv{\frac{1}{\abs{\cM_{1}}}\sum_{M\in\cM_{1}}\bbP^{M}_{n}}{\frac{1}{\abs{\cM_{2}}}\sum_{M\in\cM_{2}}\bbP^{M}_{n}}},
\end{align*}
where the last inequality follows because  $\bbP(E) + \mathbb{Q}(E^c)
\geq 1 - \Dtv{\bbP} {\mathbb{Q}}$ for any event $E$. \qed

  \section{Proof of Lemma \ref*{lm:tv}}\label{sec:tv}

This proof is organized as follows. In \cref{sec:reference}, we
introduce a reference measure and move from the total variation
distance to the $\chi^2$-divergence. This allows us to reduce the task
of upper bounding  $\Dtv{\bbP_{n}^1}{\bbP_{n}^2}$ to the task of upper
bounding two manageable density ratios
(\cref{eq:simplify1,eq:simplify2} in the sequel). We develop
several intermediate technical lemmas related to the density ratios in
\cref{sec:tech}, and in \cref{sec:calc} we put everything together to
bound the density ratios, thus completing the proof of \cref{lm:tv}.

For the statement of \cref{lm:tv} and the main subsections of this appendix
  (\cref{sec:reference,sec:calc}), we only consider the specific
  values for the parameters $(\theta_1,\alpha_1,\beta_1;\theta_2,\alpha_2,\beta_2;w)$ 
given in \cref{sec:parameter}. However, in \cref{sec:tech}, which contains intermediate technical lemmas, results are presented
under a slightly more general setup, as explained at the beginning of the
subsection.

\subsection{Introducing a Reference Measure and Moving to
  $\chi^2$-Divergence}
\label{sec:reference}

Directly calculating the total variation distance $\Dtv{\bbP_{n}^1}{\bbP_{n}^2}$ 
is challenging, so we design an auxillary \emph{reference measure}
$\bbP_{n}^0$  which serves as an intermediate quantity to help with the upper bound. 
The reference measure $\bbP_{n}^0$ lies in the same measurable space as $\bbP_{n}^1$ and $\bbP_{n}^2$, and is defined as follows:
\[
\bbP_{n}^0\prn{\crl{\prn{s_i,a_i,r_i,s_i'}}_{i=1}^n}:=\prod_{i=1}^n \mu(s_i,a_i)\Ind_{\crl{r_i=R_0(s_i,a_i)}}P_0(s_i'\mid s_i,a_i),~~~\forall\; \crl{\prn{s_i,a_i,r_i,s_i'}}_{i=1}^n,
\]
where %
\[
R_0(s,\initac):=\begin{cases}
0,&s\in\{\init\}\cup\cS^1,\\
w=3\gamma/8,&s=W,\\
1,&s=X,\\
0,&s=Y,\\
0,&s=Z,
\end{cases}
\]
and %
\begin{align*}
P_0(\cdot\mid \nullstate,1)&:=W, ~\text{w.p. }1, \\
P_0(\cdot\mid\init,2)&:=\textrm{Unif}(\cS^1),\\
\forall s\in\cS^1:~~P_0(\cdot\mid s,\initac)&:=\begin{cases}
X,&\textrm{w.p. } \theta_1\alpha_1,\\
Y,&\textrm{w.p. } 1-\theta_1\alpha_1-(1-\theta_1)\beta_1,\\
Z,&\textrm{w.p. } (1-\theta_1)\beta_1,
\end{cases}\\
\forall s\in\{W,X,Y,Z\}:~~ P_0(\cdot\mid s,\initac)&:=s, ~\text{w.p. }1.
\end{align*}
The reference measure $\bbP_{n}^0$ can be understood as the law of
$D_n$ when the data collection distribution is $\mu$ and the
underlying MDP is $M_0:=(\cS,\cA,P_0,R_0,\gamma,d_0)$. {Note
  that although we define the transition operator $P_0$ above based on
  the tuple $(\theta_1,\alpha_1,\beta_1)$, substituting in
  $(\theta_2,\alpha_2,\beta_2)$ leads to the same operator --- this is
  guaranteed by an important feature of our construction: the families
  $\cM_1$ and $\cM_2$ in our construction satisfy the constraint
  \cref{eq:marginal}, so that $\theta_1\alpha_1 = \theta_2\alpha_2$
  and $(1-\theta_1)\beta_1=(1-\theta_2)\beta_2$.}

{In what follows, we provide more explanations on the design of the transition operator $P_0$ and the reward function $R_0$.}

\paragraph{Properties and Intuition of $P_0$} 
{There are two ways to understand $P_0$. Operationally, $P_0$ is simply the pointwise \emph{average} transition operator of the MDPs in $\cM_1$ or $\cM_2$, in the sense that
\begin{equation*}\label{eq:average-tr}
\forall s \in \cS, a \in \cA:~~ P_0(\cdot \mid s,a) =\frac{1}{\abs{\cM_1}}\sum_{M\in\cM_1}P_{M}(\cdot \mid s,a)=\frac{1}{\abs{\cM_2}}\sum_{M\in\cM_2}P_{M}(\cdot \mid s,a),
\end{equation*}
where $P_M$ is the transition operator associated with each MDP $M$. 
More conceptually, $P_0$ is the transition operator obtained by performing state aggregation using the value function class $\cF = \{f_1,f_2\}$, where states with the same values for both $f_1$ and $f_2$ are viewed as identical and constrained to share dynamics (which is induced by averaging over the data collection distribution). 
}

\paragraph{Properties and Intuition of $R_0$} {Outside of state $Z$, the reward function $R_0$ is the same as the reward function of any MDP in $\cM$, i..e.,
\begin{equation*}\label{eq:average-rf}
\forall s \ne Z, a \in \cA:~~ R_0(s,a) =R_{M}(s,a),~\forall M\in\cM,
\end{equation*}
where $R_M$ is the transition operator associated with each MDP $M$. The value of $R_0(Z,\initac)$ is immaterial, as the data collection distribution $\mu$ is not supported on $(Z, \initac)$ (in other words, different values of $R_0(Z,\initac)$ lead to essentially the same reference measure $\bbP_n^0$); we choose $R_0(Z, a) = 0$ for concreteness.
}

\paragraph{Moving to $\chi^2$-Divergence} 
{Equipped with the definition of the reference measure $\bbP_n^0$, we proceed to bound $\Dtv{\bbP_{n}^1}{\bbP_{n}^2}$.}
By the triangle inequality for the total variation distance, we have
\begin{align}
\Dtv{\bbP_{n}^1}{\bbP_{n}^2}&\le \Dtv{\bbP_{n}^1}{\bbP_{n}^0}+\Dtv{\bbP_{n}^2}{\bbP_{n}^0}\notag\\
&\le\frac{1}{2}\sqrt{D_{\chi^2}\prn{\bbP_{n}^1\dmid\bbP_{n}^0}}+\frac{1}{2}\sqrt{D_{\chi^2}\prn{\bbP_{n}^2\dmid\bbP_{n}^0}},\label{eq:tv_triangle}
\end{align}
where the last inequality follows from the fact that $\Dtv{\bbP}{\bbQ}\le\frac{1}{2}\sqrt{D_{\chi^2}\prn{\bbP\dmid\bbQ}}$ for any $\bbP,\bbQ$ (see Proposition 7.2 or Section 7.6 of \cite{polyanskiy}).

In what follows, we derive simplified expressions for $D_{\chi^2}\prn{\bbP_{n}^1\dmid\bbP_{n}^0}$ and $D_{\chi^2}\prn{\bbP_{n}^2\dmid\bbP_{n}^0}$.
We first expand and simplify $D_{\chi^2}\prn{\bbP_{n}^1,\bbP_{n}^0}$, then obtain a similar
expression for $D_{\chi^2}\prn{\bbP_{n}^2\dmid\bbP_{n}^0}$. 

For each
MDP $M\in\cM$, let $P_M$ and $R_M$ denote the associated transition and reward
functions.
Observe that our construction for $P_M$, $R_M$, and $\mu$ (see
\cref{sec:con}) ensures that for any
$(s,a,r,s')\in\cS\times\cA\times[0,1]\times\cS$ with
$\mu(s,a)\Ind_{\crl{r=R_0(s,a)}}P_0(s'\mid s,a)=0$, we have
$\mu(s,a)\Ind_{\crl{r=R_M(s,a)}}P_M(s'\mid s,a)=0$. As a result, we
have $\bbP_{n}^{\sM}\ll \bbP^0_{n}$ for any $M\in\cM$, which implies
that $\bbP_{n}^1,\bbP_{n}^{2}\ll \bbP^0_{n}$. Hence, we can expand the $\chi^2$-divergence as
\begin{align}\label{eq:simplify1}
&D_{\chi^2}\prn{\bbP_{n}^1\dmid\bbP_{n}^0}\notag\\
&=\bbE_{\crl{\prn{s_i,a_i,r_i,s_i'}}_{i=1}^n\sim \bbP_{n}^0}\brk*{\prn*{\frac{\frac{1}{|\cM_1|}\sum_{M\in\cM_1}\bbP^M_{n}(\crl{\prn{s_i,a_i,r_i,s_i'}}_{i=1}^n)}{\bbP_{n}^0(\crl{\prn{s_i,a_i,r_i,s_i'}}_{i=1}^n)}}^2}-1\notag\\
&=\bbE_{\crl{\prn{s_i,a_i,r_i,s_i'}}_{i=1}^n\sim
                                                                                                                                                                                                                                         \bbP_{n}^0}\brk*{\prn*{\frac{{\frac{1}{|\cM_1|}\sum_{M\in\cM_1}\prod_{i=1}^n\mu(s_i,a_i)\Ind_{\crl{r_i=R_M(s_i,a_i)}}P_M(s_i'\mid s_i,a_i)}}{\prod_{i=1}^n\mu(s_i,a_i)\Ind_{\crl{r_i=R_0(s_i,a_i)}}{P_0(s_i'\mid s_i,a_i)}}}^2}-1\notag\\
  &=\bbE_{\crl{\prn{s_i,a_i,r_i,s_i'}}_{i=1}^n\sim \bbP_{n}^0}\brk*{\prn*{\frac{{\frac{1}{|\cM_1|}\sum_{M\in\cM_1}\prod_{i=1}^nP_M(s_i'\mid s_i,a_i)}}{\prod_{i=1}^n{P_0(s_i'\mid s_i,a_i)}}}^2}-1\notag\\
&=\frac{1}{\abs{\cM_1}^2}\sum_{M,M'\in\cM_{1}}\bbE_{\crl{\prn{s_i,a_i,r_i,s_i'}}_{i=1}^n\sim \bbP_{n}^0}\brk*{{\frac{\prod_{i=1}^n{P_M(s_i'\mid s_i,a_i)P_{M'}(s_i'\mid s_,a_i)}}{\prod_{i=1}^n{P_0^2(s_i'\mid s_i,a_i)}}}}-1\notag\\
&=\frac{1}{\abs{\cM_1}^2}\sum_{M,M'\in\cM_{1}}\prn*{\bbE_{(s,a)\sim\mu,s'\sim P_0(\cdot\mid s,a)}\brk*{{\frac{P_M(s'\mid s,a)P_{M'}(s'\mid s,a)}{P_0^2(s'\mid s,a)}}}}^n-1,
\end{align}
where the third equality follows because (i)
$R_M(s,a)=R_0(s,a), \forall M\in\cM,\forall a\in\cA, \forall s\ne Z$,
and (ii) state $Z$ is not covered by $\mu$. Indeed, since the
reward function for every MDP in $\cM$ is the same as $R_0$ 
\emph{for all $(s,a)$  covered by $\mu$}, the rewards $r_1,\dots,r_n$
in $\Dn$ are completely uninformative in our construction---they have the same distribution regardless of the underlying MDP. This
is why the final expression for
$D_{\chi^2}\prn{\bbP_{n}^1\dmid\bbP_{n}^0}$ in \pref{eq:simplify1} is completely
independent of the reward distribution for both measures.

Using an identical calculation, we also have
\begin{align}\label{eq:simplify2}
D_{\chi^2}\prn{\bbP_{n}^2\dmid\bbP_{n}^0}=\frac{1}{\abs{\cM_2}^2}\sum_{M,M'\in\cM_{2}}\prn*{\bbE_{(s,a)\sim\mu,s'\sim P_0(\cdot\mid s,a)}\brk*{{\frac{P_M(s'\mid s,a)P_{M'}(s'\mid s,a)}{P_0^2(s'\mid s,a)}}}}^n-1.
\end{align}

Equipped with these expressions for the \chisquared,
the next step in the proof of \pref{lm:tv} is to upper bound the
right-hand side for \cref{eq:simplify1,eq:simplify2}. This is done in
\cref{sec:calc}, but before proceeding we require several intermediate technical lemmas.

\subsection{Technical Lemmas for Density Ratios}\label{sec:tech}

{In this subsection, we state a number of technical lemmas which
  can be used to bound the density ratio appearing inside the square in \pref{eq:simplify1,eq:simplify2} for generic MDPs
  $M_{\alpha,\beta,w,\psub}$ with
  $I\in\cI_{\theta}$. The lemmas hold for any choice of
  $(\theta, \alpha,\beta)$, and are independent of the reward parameter
  $w$. For this general setup, we work
  with a variant of the reference operator $P_0$ defined based on the values $(\theta,\alpha,\beta)$ via
  \begin{align*}
P_0(\cdot\mid \nullstate,1)&:=W, ~\text{w.p. }1, \\
P_0(\cdot\mid\init,2)&:=\textrm{Unif}(\cS^1),\\
\forall s\in\cS^1:~~P_0(\cdot\mid s,\initac)&:=\begin{cases}
X,&\textrm{w.p. } \theta\alpha,\\
Y,&\textrm{w.p. } 1-\theta\alpha-(1-\theta)\beta,\\
Z,&\textrm{w.p. } (1-\theta)\beta,
\end{cases}\\
\forall s\in\{W,X,Y,Z\}:~~ P_0(\cdot\mid s,\initac)&:=s, ~\text{w.p. }1.
  \end{align*}
 In \pref{sec:calc}, we
  instantiate the results from this subsection with $(\theta_i,\alpha_i,\beta_i)$ for $i \in
  \{1,2\}$. Recall that per the discussion in \pref{sec:reference}, our
  specific parameter choices for the families $\cM_1$ and $\cM_2$
  induce the same reference operator $P_0$.
} %

\begin{lemma}\label{lm:cap-bound}
For all $\psub,\psub'\in\cI_{\theta}$, 
$(2\theta-1)_+  S_1\le \abs{\psub\cap \psub'}\le \theta  S_1$.
\end{lemma}
\begin{proof}
Since $\abs{\psub}=\abs{\psub'}=\theta  S_1$, we have
$\abs{\psub\cap \psub'}\le\abs{\psub}=\theta  S_1$
and
\[
\abs{\psub\cap \psub'}=\abs{\psub}+\abs{\psub'}-\abs{\psub\cup \psub'}\ge \abs{\psub}+\abs{\psub'}- S_1=(2\theta-1) S_1.
\]
Since $|\psub\cap \psub'|\ge0$ trivially, the result follows.
\end{proof}

The next lemma controls the density ratio for states in $\cS^1$. To state the result compactly, we define
\begin{equation}
\phi_{\theta,\alpha,\beta}:=\theta^2\prn*{\frac{(\beta-\alpha)^2}{\theta(\beta-\alpha)+1-\beta}+\frac{\theta(\beta-\alpha)+\alpha}{\theta(1-\theta)}}.\label{eq:phi}
\end{equation}
Since \cref{lm:ratio1} is stated for any given $\theta,\alpha,\beta$, we use $P_\psub$ to denote $P_{\alpha,\beta,\psub}$ to keep notation compact.

\begin{lemma}\label{lm:ratio1}
For all $\psub,\psub'\in\cI_{\theta}$, we have
\[\bbE_{s\sim{\rm Unif}(\cS^1),s'\sim P_0(\cdot\mid s{, \initac})}\brk*{
    \frac{ P_\psub(s'\mid s{, \initac}) P_{\psub'}(s'\mid s{, \initac})}{P_0^2(s'\mid s{,
        \initac})}}=1 + \phi_{\theta,\alpha,\beta}\cdot{}\prn*{\frac{\abs{\psub\cap \psub'}}{\theta^2  S_1}-1}.\]
\end{lemma}
\begin{proof}
For any $\psub,\psub'\in\cI_{\theta}$, we observe that
\begin{align*}
    \bbE_{s\sim\textrm{Unif}(\cS^1),s'\sim P_0(\cdot\mid s{, \initac})}\brk*{ \frac{ P_\psub(s'\mid s{, \initac})P_{\psub'}(s'\mid s{, \initac})}{P_0^2(s'\mid s{, \initac})}}
    &=\bbE_{s\sim {\rm Unif}(\cS^1)}\brk*{ {\sum_{s'\in\crl{X,Y,Z}}\frac{ P_\psub(s'\mid s{, \initac}) P_{\psub'}(s'\mid s{, \initac})}{P_0(s'\mid s{, \initac})}}}.
\end{align*}
To proceed, we calculate the value of the ratio $\frac{ P_\psub(s'\mid s{, \initac}) P_{\psub'}(s'\mid s{,
    \initac})}{P_0(s'\mid s{, \initac})}$ for each possible choice for
$s\in\cS^1$ and $s'\in\crl{X,Y,Z}$ in \pref{tb:val} below.
\begin{table}[h]
\begin{center}
\begin{tabular}{|c | c | c | c|}
\hline
& $s' = X$ & $s'=Y$ & $s'=Z$\\
\hline\hline
$s \in \psub \cap \psub'$ & $\alpha/\theta$ & $(1-\alpha)^2/(\theta(1-\alpha)+(1-\theta)(1-\beta))$ &  0 \\
\hline
$s \in (\psub \cup \psub') \setminus (\psub \cap \psub')$ & 0 &  $(1-\alpha)(1-\beta)/(\theta(1-\alpha)+(1-\theta)(1-\beta))$ &  0 \\
\hline
$s \notin (\psub \cup \psub')$ & 0 & $(1-\beta)^2/(\theta(1-\alpha)+(1-\theta)(1-\beta))$ & $\beta/(1-\theta)$\\
\hline
\end{tabular}
\caption{Value of $\frac{ P_\psub(s'\mid s, 2) P_{\psub'}(s'\mid s,
    2)}{P_0(s'\mid s, 2)}$ for all possible pairs $(s,s')$.}\label{tb:val}
\end{center}
\end{table}

Define $t\ldef\abs{\psub\cap \psub'}$. From
\cref{lm:cap-bound}, we must have $t\in[(2\theta-1)_+S_{{1}},\theta
S_{{1}}]$. We also have $\abs{\psub\cup\psub'}=\abs{\psub}+\abs{\psub'}-\abs{\psub\cap\psub'}=2\theta S_{{1}}-t$. Hence, the event in the first row of \cref{tb:val} occurs with probability $\abs{\psub\cap\psub'}/S_{{1}}=t/S_{{1}}$,
the event in the second row occurs with probability $\abs{(\psub \cup \psub') \setminus (\psub \cap \psub')}/S_{{1}}=(2\theta S_{{1}}-2t)/S_{{1}}$
and the event in the third row
occurs with probability $\abs{S_{{1}}\setminus\prn{\psub\cup\psub'}}/S_{{1}}=((1-2\theta)S_{{1}}+t)/S_{{1}}$. 
Using these values, we obtain
\newcommand{\groupi}{\text{(i)}}
\newcommand{\groupii}{\text{(ii)}}
\newcommand{\groupiii}{\text{(iii)}}
\begin{align*}
& \bbE_{s\sim {\rm Unif}(\cS^1)}\brk*{ {\sum_{s'\in\crl{X,Y,Z}}\frac{ P_\psub(s'\mid s{, \initac}) P_{\psub'}(s'\mid s{, \initac})}{P_0(s'\mid s{, \initac})}}}\\
& =  \frac{t}{S_{{1}}}\cdot\prn*{\frac{\alpha}{\theta}+ \frac{(1-\alpha)^2}{\theta(1-\alpha)+(1-\theta)(1-\beta)} } + \prn*{2\theta - \frac{2t}{S_{{1}}}}\cdot\frac{(1-\alpha)(1-\beta)}{\theta(1-\alpha)+(1-\theta)(1-\beta)}  \\
& ~~~~~~+ \prn*{1-2\theta+\frac{t}{S_{{1}}}}\cdot\prn*{ \frac{(1-\beta)^2}{\theta(1-\alpha)+(1-\theta)(1-\beta)} + \frac{\beta}{1-\theta}}\\
& = \frac{t}{S_{{1}}}\prn*{\frac{(\beta-\alpha)^2}{\theta(\beta-\alpha)+1-\beta}+\frac{\alpha}{\theta}+\frac{\beta}{1-\theta}}+\frac{2\theta(\beta-\alpha)(1-\beta)+(1-\beta)^2}{\theta(\beta-\alpha)+1-\beta}+\frac{(1-2\theta)\beta}{1-\theta}\\
& = \frac{t}{S_{{1}}}\prn*{\frac{(\beta-\alpha)^2}{\theta(\beta-\alpha)+1-\beta}+\frac{\alpha}{\theta}+\frac{\beta}{1-\theta}}+\frac{2\theta(\beta-\alpha)(1-\beta)+(1-\beta)^2}{\theta(\beta-\alpha)+1-\beta}+\beta-\frac{\theta\beta}{1-\theta}\\
& = \prn*{\frac{t}{S_{{1}}}-\theta^2}\prn*{\frac{(\beta-\alpha)^2}{\theta(\beta-\alpha)+1-\beta}+\frac{\alpha}{\theta}+\frac{\beta}{1-\theta}}\\
&~~~~~~+\underbrace{\theta^{2}\cdot\frac{(\beta-\alpha)^2}{\theta(\beta-\alpha)+1-\beta}}_{\groupi}+\underbrace{\theta^2\cdot\frac{\alpha}{\theta}}_{\groupii}+\underbrace{\theta^2\cdot\frac{\beta}{1-\theta}}_{\groupiii}+\underbrace{\frac{2\theta(\beta-\alpha)(1-\beta)+(1-\beta)^2}{\theta(\beta-\alpha)+1-\beta}}_{\groupi}+ \underbrace{\beta}_{\groupii}-\underbrace{\frac{\theta\beta}{1-\theta}}_{\groupiii}
.
\end{align*}
Grouping the terms in the second line together, we find that $\groupi =
\theta(\beta-\alpha)+1-\beta$, $\groupii=\theta\alpha+\beta$, and
$\groupiii=-\theta\beta$, and by summing,
\[
  \groupi + \groupii + \groupiii = 1.
\]
Hence, the above expression is equal to
\begin{align*}
& \prn*{\frac{t}{S_{{1}}}-\theta^2}\prn*{\frac{(\beta-\alpha)^2}{\theta(\beta-\alpha)+1-\beta}+\frac{\alpha}{\theta}+\frac{\beta}{1-\theta}}+1\\
&=\prn*{\frac{t}{S_{{1}}}-\theta^2}\prn*{\frac{(\beta-\alpha)^2}{\theta(\beta-\alpha)+1-\beta}+\frac{\theta(\beta-\alpha)+\alpha}{\theta(1-\theta)}}+1.
\end{align*}
Recalling the definition of $\phi_{\theta,\alpha,\beta}$, this completes the proof.
\end{proof}

The next lemma bounds the magnitude of
$\phi_{\theta,\alpha,\beta}$ in terms of the parameter $\theta$.
\begin{lemma}\label{lm:phi-bound}
For any $\alpha,\beta,\theta\in(0,1)$, we have
\[
\theta^2\abs{\alpha-\beta}\le \phi_{\theta,\alpha,\beta}\le\frac{\theta}{1-\theta}\max\crl{\alpha,\beta}\leq\frac{\theta}{1-\theta}.
\]
\end{lemma}
\begin{proof}Recall that $\phi_{\theta,\alpha,\beta}=\theta^2\prn*{\frac{(\beta-\alpha)^2}{\theta(\beta-\alpha)+1-\beta}+\frac{\theta(\beta-\alpha)+\alpha}{\theta(1-\theta)}}$.
We consider two cases.

\textbf{Case 1: $\alpha\le \beta$.}  Assume $\alpha<\beta$, as the
result is immediate if $\alpha=\beta$. We have
\begin{align*}
    {\frac{(\beta-\alpha)^2}{\theta(\beta-\alpha)+1-\beta}+\frac{\theta(\beta-\alpha)+\alpha}{\theta(1-\theta)}}\ge 0+\frac{\theta(\beta-\alpha)+\alpha}{\theta(1-\theta)}\ge\frac{\theta(\beta-\alpha)}{\theta(1-\theta)}=\frac{\beta-\alpha}{1-\theta} { > |\alpha - \beta|}
\end{align*}
and
\begin{align*}
      {\frac{(\beta-\alpha)^2}{\theta(\beta-\alpha)+1-\beta}+\frac{\theta(\beta-\alpha)+\alpha}{\theta(1-\theta)}}
      &=  {\frac{(\beta-\alpha)^2}{\theta(\beta-\alpha)+1-\beta}}+\frac{\beta-\alpha}{1-\theta}+\frac{\alpha}{\theta(1-\theta)}\\
      &\le \frac{\beta-\alpha}{\theta}+\frac{\beta-\alpha}{1-\theta}+\frac{\alpha}{\theta(1-\theta)}\\
      &=\frac{\beta}{\theta(1-\theta)},
\end{align*}
where the inequality above follows since $\theta(\beta-\alpha)+1-\beta>\theta(\beta-\alpha)>0$.

\textbf{Case 2: $\alpha> \beta$.}   We have
\begin{align*}
    {\frac{(\beta-\alpha)^2}{\theta(\beta-\alpha)+1-\beta}+\frac{\theta(\beta-\alpha)+\alpha}{\theta(1-\theta)}}\ge 0+\frac{\theta(\beta-\alpha)+\alpha}{\theta(1-\theta)}\ge\frac{\theta(\beta-\alpha)+\alpha-\beta}{\theta(1-\theta)}=\frac{\alpha-\beta}{\theta} { > |\alpha - \beta|}
\end{align*}
and
\begin{align*}
      {\frac{(\beta-\alpha)^2}{\theta(\beta-\alpha)+1-\beta}+\frac{\theta(\beta-\alpha)+\alpha}{\theta(1-\theta)}}
      &= {\frac{(\alpha-\beta)^2}{1-\beta-\theta(\alpha-\beta)}}+\frac{\beta-\alpha}{1-\theta}+\frac{\alpha}{\theta(1-\theta)}\\
      &\le {\frac{(\alpha-\beta)^2}{(\alpha-\beta)-\theta(\alpha-\beta)}}+\frac{\beta-\alpha}{1-\theta}+\frac{\alpha}{\theta(1-\theta)}\\
      &=\frac{\alpha}{\theta(1-\theta)},
\end{align*}
where the inequality uses that
$1-\beta>\alpha-\beta>\theta(\alpha-\beta)$.

The lemma immediately follows.%
\end{proof}

The final lemma in this section controls the density ratio for the initial state $\mathfrak{s}$ when action 2 is chosen. Again, we use $P_\psub$ to denote $P_{\alpha,\beta,\psub}$ to keep notation compact.
\begin{lemma}\label{lm:ratio2}
For any $\psub,\psub'\in\cI_{\theta}$, we have
\[\bbE_{{s'}\sim P_0(\cdot\mid \nullstate{, 2})}\brk*{ \frac{ P_\psub(s'\mid \nullstate{, 2}) P_{\psub'}({s'}\mid \nullstate{, 2})}{P_0^2({s'}\mid \nullstate{, 2})}}=\frac{\abs{\psub\cap \psub'}}{\theta^2  S_1}.\]
\end{lemma}

\begin{proof}
Let $\psub,\psub'\in\cI_{\theta}$ be given and observe that
\begin{align*}
    \bbE_{s'\sim P_0(\cdot\mid \nullstate{, 2})}\brk*{ \frac{ P_\psub(s'\mid \nullstate{, 2}) \times P_{\psub'}({s'}\mid \nullstate{, 2})}{P_0^2({s'\mid \nullstate{, 2}})}}
    &=\bbE_{s'\sim {\rm Unif}(\cS^1)}\brk*{ \frac{\Ind_{\crl{{s'}\in \psub\cap \psub'}}}{\theta^2}}
    =\frac{\abs{\psub\cap \psub'}}{\theta^2  S_1}.\tag*\qedhere
\end{align*}
\end{proof}

\subsection{Completing the Proof}\label{sec:calc}

To keep notation compact, define
\[
g_{\theta,\alpha,\beta}(t; n):=\prn*{\prn*{\frac{t}{\theta^2 S_{{1}}}-1}\frac{8\phi_{\theta,\alpha,\beta}+1}{16}+1}^n.
\]
For all $M\in\cM$, $P_M(\cdot\mid s,a)$ and
$P_0(\cdot\mid s,a)$ differ only when $(s,a)=(\nullstate, 2)$ or
$(s,a)\in\cS^1\times\cA$, so---recalling the value of $\mu$---we have
\begin{align*}
&\frac{1}{\abs{\cM_1}^2}\sum_{M,M'\in\cM_{1}}\prn*{\bbE_{\substack{(s,a)\sim\mu,~~~~~\\s'\sim P_0(\cdot\mid s,a)}}\brk*{{\frac{P_M(s'\mid s,a)P_{M'}(s'\mid s,a)}{P_0^2(s'\mid s,a)}}}}^n\\
\arxiv{
&=\frac{1}{\abs{\cM_1}^2}\sum_{M,M'\in\cM_{1}}\prn*{\frac{1}{2}\bbE_{\substack{s\sim\textrm{Unif}(\cS^1),\\s'\sim P_0(\cdot\mid s{,\initac})}}\brk*{\frac{P_{M}(s'\mid s{,\initac}) P_{M'}(s'\mid s{,\initac})}{P_0^2(s'\mid s{,\initac})}}+\frac{1}{16}\bbE_{{s'}\sim P_0(\cdot\mid \nullstate{, 2})}\brk*{ \frac{ P_{M}(s'\mid \nullstate{, 2}) P_{M'}({s'}\mid \nullstate{, 2})}{P_0^2({s'}\mid \nullstate{,2})}}+\frac{7}{16}}^n\\}
\colt{
  &=\frac{1}{\abs{\cM_1}^2}\sum_{M,M'\in\cM_{1}}\left(\frac{1}{4}\bbE_{\substack{s\sim\textrm{Unif}(\cS^1),\\s'\sim P_0(\cdot\mid s{,2})}}\brk*{\frac{P_{M}(s'\mid s{,2}) P_{M'}(s'\mid s{,2})}{P_0^2(s'\mid s{,2})}}\right.\\
  &~~~~~~~~~~~~~~~~~~~~~~~~~~~~~~~~~\left.+\frac{1}{8}\bbE_{{s'}\sim P_0(\cdot\mid \nullstate{, \initac})}\brk*{ \frac{ P_{M}(s'\mid \nullstate{, \initac}) P_{M'}({s'}\mid \nullstate{, \initac})}{P_0^2({s'}\mid \nullstate{,2})}}+\frac{5}{8}\right)^n\\
  }
&=\frac{1}{{ S_1 \choose \theta_1  S_1}^2 }
          \sum_{t}\sum_{\psub,\psub'\in\cI_{\theta_1}:\abs{\psub\cap \psub'}=t}\prn*{\frac{1}{2}\prn*{\prn*{\frac{t}{\theta_1^2  S_1}-1}\phi_{\theta_1,\alpha_1,\beta_1}+1}+\frac{1}{16}\frac{t}{\theta_1^2 S_1}+\frac{7}{16}}^n,
\end{align*}
where we have used the expressions for the density ratio from \cref{lm:ratio1,lm:ratio2}. We
further simplify to
\begin{align*}
&=\frac{1}{{ S_1 \choose \theta_1  S_1}^2 } \sum_{t}\sum_{\psub,\psub'\in\cI_{\theta_1}:\abs{\psub\cap \psub'}=t}\prn*{{\prn*{\frac{t}{\theta_1^2  S_1}-1}\frac{8\phi_{\theta_1,\alpha_1,\beta_1}+1}{16}+1}}^n\\
&=\sum_{t=(2\theta_1-1)_+  S_1}^{\theta_1  S_1}\frac{{\theta_1  S_1 \choose t}{ S_1-\theta_1  S_1 \choose \theta_1  S_1 - t}}{{ S_1 \choose \theta  S_1}}\prn*{\prn*{\frac{t}{\theta_1^2  S_1}-1}\frac{8\phi_{\theta_1,\alpha_1,\beta_1}+1}{16}+1}^n   \notag\\
  &=\sum_{t=(2\theta_1-1)_+  S_1}^{\theta_1  S_1}\frac{{\theta_1  S_1 \choose t}{ S_1-\theta_1  S_1 \choose \theta_1  S_1 - t}}{{ S_1 \choose \theta_1  S_1}}g_{\theta_1,\alpha_1,\beta_1}(t; n),
\end{align*}
where the second equality uses \pref{lm:cap-bound}. Applying the
same calculation for $\cM_2$, we also have that
\begin{align*}
&\frac{1}{\abs{\cM_2}^2}\sum_{M,M'\in\cM_{2}}\prn*{\bbE_{(s,a)\sim\mu,s'\sim P_0(\cdot\mid s,a)}\brk*{{\frac{P_M(s'\mid s,a)P_{M'}(s'\mid s,a)}{P_0^2(s'\mid s,a)}}}}^n\\
&=\sum_{t=(2\theta_2-1)_+  S_1}^{\theta_2  S_1}\frac{{\theta_2  S_1 \choose t}{ S_1-\theta_2  S_1 \choose \theta_2  S_1 - t}}{{ S_1 \choose \theta_2  S_1}}g_{\theta_2,\alpha_2,\beta_2}(t; n).
\end{align*}
Therefore, to upper bound the right-hand sides of \cref{eq:simplify1,eq:simplify2}, we only need to upper bound the quantity
\begin{equation}\label{eq:likelihood}
\sum_{t=(2\theta-1)_+  S_1}^{\theta  S_1}\frac{{\theta  S_1 \choose t}{ S_1-\theta  S_1 \choose \theta  S_1 - t}}{{ S_1 \choose \theta  S_1}}g_{\theta,\alpha,\beta}(t; n),
\end{equation}
for both $(\theta,\alpha,\beta)=(\theta_1,\alpha_1,\beta_1)$ and
$(\theta,\alpha,\beta)=(\theta_2,\alpha_2,\beta_2)$. To upper bound
this quantity, we use the following two
lemmas.
\begin{lemma}[Monotonicity of $g_{\theta,\alpha,\beta}$]\label{lm:monotone}
For any $\theta,\alpha,\beta\in(0,1)$ and any $n\in\bbN$, the function
$t\mapsto{}g_{\theta,\alpha,\beta}(t; n)$ is non-decreasing for $t\in [(2\theta-1)_+S_1,\theta S_1]$. 
\end{lemma}
\begin{proof}
By \cref{lm:ratio1}, we have $\prn*{\frac{t}{\theta^2
    S_1}-1}\phi_{\theta,\alpha,\beta}+1\ge0$ for all $t\in
[(2\theta-1)_+S_1,\theta S_1]$, and hence \[\prn*{\frac{t}{\theta^2
      S_1}-1}\frac{8\phi_{\theta,\alpha,\beta}+1}{16}+1=\frac{1}{2}\prn*{\prn*{\frac{t}{\theta^2
        S_1}-1}\phi_{\theta,\alpha,\beta}+1}+\frac{1}{16}\frac{t}{\theta^2
    S_1}+\frac{7}{16}\ge0\]  for all $t\in [(2\theta-1)_+S_1,\theta
S_1]$. This ensures that we are in the domain where $x\mapsto{}x^{n}$
is non-decreasing. Next, by \cref{lm:phi-bound}, we know that
$\phi_{\theta,\alpha,\beta}\ge0$, so the coefficient on $t$ is
non-negative. It follows that $g_{\theta,\alpha,\beta}(t; n)$ is non-decreasing in $t\in [(2\theta-1)_+S_1,\theta S_1]$. 
\end{proof}

\begin{lemma}[Hypergeometric tail bound]\label{lm:hypergeo}
For any {$\theta\in\crl{\theta_1,\theta_2}$} and $\epsilon\in(0,\theta^2 S_1)$, we have
\begin{align}\label{eq:hypergeo}
\sum_{t \geq (\theta+\epsilon) \cdot \theta S_1} \frac{ {\theta S_1 \choose t}{S_1 - \theta S_1 \choose \theta S_1 - t}}{{S_1 \choose \theta S_1}} \leq \exp(-2\epsilon^2 \theta S_1).
\end{align}
\end{lemma}

\begin{proof}
Let {$\mathrm{Hyper}(t; K,N,N'):={K \choose t}{N-K \choose N'-t}/{N \choose N'}$} denote the hypergeometric probability mass
function, which corresponds to the probability that exactly $t$ balls are blue when $N'$ balls are
sampled without replacement from a jar containing $N$ total balls, $K$
of which are blue {(see, e.g., Chapter 2.1.4 of \cite{rice2006mathematical} for background)}.
We observe that the term $\frac{{\theta S_1 \choose
    t}{S_1-\theta S_1 \choose \theta S_1 - t}}{{S_1 \choose \theta
    S_1}}$ arising in \cref{eq:likelihood,eq:hypergeo} is precisely
$\mathrm{Hyper}(t; \theta S_1, S_1, \theta S_1)$, which corresponds to
the process in which we sample $\theta S_1$ balls without replacement from a jar with $S_1$ balls, $\theta S_1$ of which are blue.

We now apply a classical tail bound for hypergeometric random
variables.
\begin{lemma}[\cite{hoeffding1963probability}]
  Let $X\sim\mathrm{Hyper}(K,N,N')$ and define $p = K/N$. Then for any
  $0 < \epsilon < pN'$, we have
  \begin{align*}
    \Pr\brk{X \geq (p+\epsilon )N' }\leq \exp\prn*{-2\epsilon^2 N'}.
  \end{align*}
\end{lemma}

Instantiating this bound with $\mathrm{Hyper}(\theta S_1, S_1, \theta
S_1)$ (since $\theta{}S_1$ is an integer), we have $p = \theta$ and 
\begin{align*}
  \sum_{t \geq (\theta+\epsilon) \cdot \theta S_1} \frac{ {\theta S_1
  \choose t}{S_1 - \theta S_1 \choose \theta S_1 - t}}{{S_1
  \choose \theta S_1}}= \Pr\brk{X \geq (\theta+\epsilon )\cdot\theta{}S_1 } \leq \exp(-2\epsilon^2 \theta S_1).\tag*\qedhere
\end{align*}
\end{proof}

Returning to the quantity in \cref{eq:likelihood}, for any
{$(\theta,\alpha,\beta)\in\crl{(\theta_1,\alpha_1,\beta),(\theta_2,\alpha_2,\beta_2)}$ and any}
$\epsilon\in(0,\theta^2 S_1)$ we can split the sum and upper bound as follows:
\begin{align}\label{eq:truncate}
& \sum_{t=(2\theta-1)_+ S_1}^{\theta S_1} \frac{ {\theta S_1 \choose t} {S_1 - \theta S_1 \choose \theta S_1-t}}{{S_1 \choose \theta S_1}} g_{\theta,\alpha,\beta}(t;n)  \notag\\
& \leq \sum_{t=0}^{\floor{(\theta+\epsilon)\theta S_1}}\frac{ {\theta S_1 \choose t} {S_1 - \theta S_1 \choose \theta S_1-t}}{{S_1 \choose \theta S_1}} g_{\theta,\alpha,\beta}(t;n)  + \exp(-2\epsilon^2 \theta S_1)\cdot g_{\theta,\alpha,\beta}(\theta S_1; n)   \notag\\
& \le \prn*{\sum_{t=0}^{\floor{(\theta+\epsilon)\theta S_1}}\frac{ {\theta S_1 \choose t} {S_1 - \theta S_1 \choose \theta S_1-t}}{{S_1 \choose \theta S_1}}} g_{\theta,\alpha,\beta}((\theta+\epsilon)\theta S_1;n) + \exp(-2\epsilon^2 \theta S_1)\cdot g_{\theta,\alpha,\beta}(\theta S_1; n)   \notag\\
&\leq{}g_{\theta,\alpha,\beta}((\theta+\epsilon)\theta S_1;n) + \exp(-2\epsilon^2 \theta S_1)\cdot g_{\theta,\alpha,\beta}(\theta S_1; n),
\end{align}
{where the first two inequalities follow from
  \cref{lm:monotone,lm:hypergeo} and the last uses that the sum in the
  penultimate line is at most $1$.}
We further calculate
\begin{align}\label{eq:first}
    g_{\theta,\alpha,\beta}((\theta+\epsilon)\theta S_1;n)&=\prn*{\prn*{\frac{(\theta+\epsilon)\theta S_1}{\theta^2 S_1}-1}\frac{8\phi_{\theta,\alpha,\beta}+1}{16}+1}^n\notag\\
    &=\prn*{\frac{\eps}{\theta}{\frac{8\phi_{\theta,\alpha,\beta}+1}{16}}+1}^n\ \notag\\
    &\le \prn*{{\frac{\epsilon}{2(1-\theta)\theta}}+1}^n,
\end{align}
where the inequality follows from \cref{lm:phi-bound}. Similarly, we have
\begin{align}\label{eq:second}
    \exp(-2\epsilon^2 \theta S_1)\cdot g_{\theta,\alpha,\beta}(\theta S_1; n)&=\exp(-2\epsilon^2 \theta S_1)\cdot\prn*{\prn*{\frac{\theta S_1}{\theta^2 S_1}-1}\frac{8\phi_{\theta,\alpha,\beta}+1}{16}+1}^n \notag\\
&\le \exp\prn*{-2\epsilon^2 \theta S_1 }\cdot  \prn*{\prn*{\frac{1}{\theta}-1}\frac{8\theta/(1-\theta)+1}{16}+1}^n \notag\\
&\le \exp\prn*{-2\epsilon^2 \theta S_1 }\cdot  \prn*{1+\frac{1}{2\theta}}^n \notag\\
&= \exp\prn*{n\ln(1+1/(2\theta))-2\eps^2\theta S_1} \notag\\
&\le\exp\prn*{n/(2\theta)-2\eps^2\theta S_1},
\end{align}
where the first inequality follows from \cref{lm:phi-bound} {and the
  last inequality uses that $\log(1+x) \leq x$.}

Combining
\cref{eq:simplify1,eq:simplify2,eq:likelihood,eq:truncate,eq:first,eq:second}
and instantiating the bounds for $(\theta_1,\alpha_1,\beta_1)$ and
$(\theta_2,\alpha_2,\beta_2)$, 
we have
\begin{equation*}
D_{\chi^2}(\bbP_{n}^1\dmid\bbP_{n}^0)\le\inf_{\epsilon\in(0,\theta_1^2 S_1)}\crl*{\prn*{{\frac{\epsilon}{2(1-\theta_1)\theta_1}}+1}^n+\exp\prn*{n/(2\theta_1)-2\eps^2\theta_1 S_1}}-1.
\end{equation*}
\begin{equation*}
D_{\chi^2}(\bbP_{n}^2\dmid\bbP_{n}^0)\le\inf_{\epsilon\in(0,\theta_2^2 S_1)}\crl*{\prn*{{\frac{\epsilon}{2(1-\theta_2)\theta_2}}+1}^n+\exp\prn*{n/(2\theta_2)-2\eps^2\theta_2 S_1}}-1.
\end{equation*}
Let $c\in(0,1/2)$ be an arbitrary constant. For each $i\in\crl{1,2}$, we set $\eps=2c\cdot \frac{(1-\theta_i)\theta_i}{n}$ (which belongs to $(0,\theta_i^2 S_1)$ {because $\epsilon < \theta_i$ since $n \geq 1$ and $\theta_iS_1\geq 1$ by assumption}). Then we have
\[
\prn*{{\frac{\epsilon}{2(1-\theta_i)\theta_i}}+1}^n\le\prn*{1+\frac{c}{n}}^n\le e^c\le 1+2c,~~~\forall i\in\crl{1,2},
\]
and
\[
D_{\chi^2}(\bbP_{n}^i\dmid\bbP_{n}^0)\le2c+\exp\prn*{\frac{n}{2\theta_i}-8c^2\theta_i {\frac{(1-\theta_i)^2\theta_i^2}{n^2}}S_1},~~~\forall i\in\crl{1,2}.
\]
In particular, whenever
$S_1\ge\max_{i\in\crl{1,2}}\frac{n^3}{8c^2\theta_i^4(1-\theta_i)^2}$, we have
\[
D_{\chi^2}(\bbP_{n}^i\dmid\bbP_{n}^0)\le2c+\exp\prn*{-n/(2\theta_i)},~~~\forall i\in\crl{1,2}.
\]

Plugging in the values $\theta_1=1/2$, $\theta_2=1/4$ and setting $c=1/10$, we have that whenever $n\ge5$ and $S_1>6400n^3$, 
\[
D_{\chi^2}(\bbP_{n}^i\dmid\bbP_{n}^0)\le\frac{1}{5}+\exp\prn*{-n}\le\frac{1}{4},~~~\forall i\in\crl{1,2}.
\]

{Combining this with~\cref{eq:tv_triangle}, we have that
  $\Dtv{\bbP_{n}^1}{\bbP_{n}^2} \leq \sqrt{1/4} = 1/2$, which
  proves the lemma.}

\qed

\colt{
\section{Extensions of Theorem \ref*{thm:main}}
\label{sec:extensions}

\arxiv{\pref{thm:main} presents the simplest variant of our lower
bound for clarity of exposition. In what follows we sketch some
straightforward extensions.}
\begin{itemize}

\arxiv{\item \emph{Policy evaluation.}}
\colt{\paragraph{Policy evaluation}}
  Our lower bound immediately extends
  from policy optimization to policy evaluation. Indeed, letting
  $\pistar_1$ and $\pistar_2$ denote the optimal policies for $\cM_1$
  and $\cM_2$ respectively, we have
  $\abs{J_M(\pistar_1)-J_M(\pistar_2)}\propto\frac{\gamma^2}{1-\gamma}$
  for all $M\in\cM$, and we know that $J_{M}(\pistar_1)$ is constant across all $M\in\cM$. {It follows that any algorithm which evaluates policy $\pi_2^\star$ to
  precision $\veps\cdot{}\frac{\gamma^2}{1-\gamma}$ with probability at
  least $1-\delta$ for sufficiently small 
  numerical constants $\veps,\delta>0$ can be used to select the optimal
  policy with probability $(1-\delta)$, and thus guarantee
  $J(\pi^\star) - \En\brk*{J(\hat{\pi})} \lesssim
  \delta\frac{\gamma^2}{1-\gamma}$. Hence, such an algorithm
  must use $n=\Omega(\abs{\cS}^{1/3})$ samples by our policy optimization lower bound}.

  To formally cast this setup in the policy evaluation setting, we
  take $\Pi=\crl*{\pistar_2}$ as the class of policies to be
  evaluated, and we require a value function class $\cF$ such that
  $Q^{\pi}_{M}\in\cM$ for all $\pi\in\Pi$, $M\in\cM$. By \cref{prop:value_calculation}, it suffices to
    select $\cF=\crl[\big]{f_1,f_2}$.

  \arxiv{\item \emph{Learning an $\veps$-suboptimal policy.}}
  \colt{\paragraph{Learning an $\veps$-suboptimal policy}}
  \pref{thm:main}
  shows that for any $\gamma\in(1/2,1)$, $n\approxgeq{}S^{1/3}$ samples are required to learn a
  $\gammaconst$-optimal policy. We can extend the construction to show
  that more generally, for any $\veps\in(0,1)$, $n\approxgeq{}\frac{S^{1/3}}{\veps}$
  samples are required to learn an $\veps\cdot\gammaconst$-optimal
  policy. We modify the MDP family $M_{\alpha,\beta,w,\psub}$ by adding
  a single dummy state $\term$ with a self-loop and zero reward. The
  initial state distribution is changed so that $d_0(\term)=1-\veps$
  and $d_1(\init)=\veps$. That is, with probability $1-\veps$,
  the agent begins in $\term$ and stays there forever, collecting no
  reward, and otherwise the agent begins at $\init$ and proceeds as in the
  original construction. Analogously, we replace the original data
  distribution $\mu$ with
  $\mu'\ldef{}(1-\veps)\delta_{\term}+\veps\mu$, where
  $\delta_{\term}$ is a point mass on $\term$. This preserves the
  concentrability bound $\Conc\leq{}16$.
  This modification rescales the optimal value functions, and the conclusion of \pref{lm:reduction} is replaced by
  \[
\sup_{M\in\cM}\crl*{J_{M}(\pi_M^\star)-\bbE^{M}_{n}\brk*{J_{M}(\wh{\pi}_{D_n})}}\ge\veps\cdot{}\frac{\gamma^2}{16(1-\gamma)}\prn*{1-\Dtv{\bbP_{n}^1}{\bbP_{n}^2}}.
\]
On the other hand, since samples from the state $\term$ provide no
information about the underlying instance, the effective number of
samples is reduced to $\veps{}n$. One can make this intuition
precise and prove that $\Dtv{\bbP_{n}^1}{\bbP_{n}^2}\le 3/4$ whenever
$\veps{}n\leq{}c\cdot{}S^{1/3}$ for a numerical constant
$c$. Combining this with the previous bound yields the result.

\arxiv{\item \emph{Linear function approximation.} } As
  discussed above,~\pref{thm:main} can be viewed as a special case of
  linear function approximation with $d=2$ and $\phi(s,a) = (f_1(s,a),
  f_2(s,a))$. Compared with recent lower bounds in the linear
  setting~\citep{wang2020statistical,zanette2021exponential}, this
  result is significantly stronger in that (a) it considers a stronger
  coverage condition, (b) holds with constant dimension and constant
  effective horizon, and (c) scales with the number of states, which
  can be arbitrarily large.
\end{itemize}

\arxiv{
Lastly, it should be clear at this point that our lower
bound construction extends to the finite-horizon setting
with $H=3$ by simply removing the self-loops from the terminal
states. The only difference is that the optimal Q-value functions require a new
calculation since rewards are no longer discounted.
}

}

\part{Proofs for Theorem \ref*{thm:admissible}}

\section{Theorem \ref*{thm:admissible}: Lower Bound Construction and Proof}
\label{app:admissible}

We restate \pref{thm:admissible} below for convenience.

\admissible*

\subsection{Lower Bound Construction}\label{sec:construction+}
We begin by specifying the structure of the MDPs in the family $\cM$
used to prove \pref{thm:admissible}. Let $\gamma\in(0,1)$ be fixed, and let $S\in\bbN$ be
given. Let $L\in\bbN$ be an integer parameter whose value will be
chosen at the end of the proof  (\cref{sec:admproof}). Define $L_{\rm
  div}:=\sum_{l=1}^L(2L+1-l)(L+2-l)\le 4L^3$, and assume without loss of generality that $S>5$ and that
$(S-5)/L_{\rm div}$ is an integer.\footnote{If $(S-5)/L_{\rm div}$ is
  not an integer, then we can simply construct the MDPs using
  $\underline{S}:=\floor{(S-5)/L_{\rm div}}L_{\rm div}+5$ states and
  then add $S-\underline{S}$ arbitrary states that are not reachable
  by any policy. Since we are considering the case where $\mu$ is
  admissible, those non-reachable states do not affect the sample
  complexity of any algorithm (as they do not affect $D_n$ at all). It
  is easy to show that the conclusion of \cref{thm:admissible} still
  holds. %
  } 
We consider a
parameterized class of MDPs illustrated in \cref{fig:adm-illustration-3}. Each MDP takes the form
$M_{L,\alpha,w,\mb{\psub}}=\crl{\cS,\cA,P_{L,\alpha,\mb{\psub}},R_{L,\alpha,w},\gamma,d_0}$,
and is parametrized by the integer $L\in\bbN$, a vector of subsets
$\mb{\psub}=(\psub^1,\dots,\psub^L)$ where 
$\psub^{l}\subseteq\cS$, and scalars {$\alpha\in(0,1/L)$} and $w\in[0,1]$.  %
All MDPs in the family $\crl{M_{L,\alpha,w,\mb{\psub}}}$ share the same state space $\cS$, action
space $\cA$,  discount factor $\gamma$, and initial state
distribution $\dnot$, and differ only in terms of the transition function $P_{L,\alpha,\mb{\psub}}$ and the reward function $R_{L,\alpha,w}$.

\paragraph{State space} We consider a layered\footnote{\label{fn:layer}Importantly, one should distinguish the concept of ``layer'' (which we use to simply refer to a group of states) and the concept of ``time step'' (which indexes the sequential evolution of the MDP). A state in layer $l\in[L]$ may be reached in any time step. For example, in \cref{fig:adm-illustration-3}, states in $\psub^3$ (which belongs to layer 3) can be reached in both time step 1 (through the blue arrow) and time step 2 (from $\widebar{\psub}^2$), but cannot be reached in time step 3.} state space
$\cS=\{\init\}\cup\cS^{1}\cup\cdots\cup\cS^L\cup\{W,X,Y,Z\}$, where
$\init$ is the initial state, $\cS^1,\dots,\cS^L$ are $L$ layers of
\emph{intermediate} (i.e., neither initial nor terminal) states, and
$\{W,X,Y,Z\}$ are self-looping terminal states. The number of
intermediate states in layer $l\in[L]$ is $S_l:=\frac{S-5}{L_{\rm
    div}}(2L+1-l)(L+2-l)$, which ensures that $\abs{\cS}=\sum_{l=1}^L
S_l+5=S$.\footnote{The precise value of $S_l$ given here is not
  essential to our proof. Its primarily serves to avoid a rounding
  issue that arises in~\cref{sec:specify+}, which can also be
  addressed through other methods.}

\paragraph{Action space} Our action space is given by
$\cA=\crl{1,2}$. For the initial state $\mathfrak{s}$, the two actions have distinct effects, while for all other states in $\cS\setminus\{\init\}$ both actions have identical effects. As a result, the value of a given
policy only depends on the action it selects in $\mathfrak{s}$. As in
the proof of \pref{thm:main}, we use the symbol $\mathfrak{a}$ as a placeholder to denote either action when taken in $s \in \cS\setminus\{\init\}$, since the choice is immaterial.

\paragraph{Transition operator} For each MDP
$\M_{L,\alpha,w,\mb{\psub}}$, recalling $\mb{\psub}=(\psub^1,\dots,\psub^L)$, we let
$\psub^l\subseteq\cS^l$ parameterize a subset of the {$l^{{\rm th}}$-layer intermediate states}. We
call each $s\in \psub^l$ an \emph{$l^{{\rm th}}$-layer planted state}
and call $s\in\widebar{I}^l:=\cS^l\setminus \psub^l$ an
\emph{$l^{{\rm th}}$-layer unplanted state}. The
dynamics $P_{L,\alpha,\mb{\psub}}$ for $M_{L,\alpha,w,\mb{\psub}}$
are determined by $L$, $\alpha\in(0,1/L)$, and $\mb{\psub}$ as
follows (cf. \pref{fig:adm-illustration-3}):
    \begin{itemize}
        \item \emph{Initial state $\init$}. For the dynamics from the
          initial state $\init$, we define
        \[
        P_{L,\alpha,\mb{\psub}}(\init,1)=\text{Unif}(\{W\}),
        \]
and          \[P_{L,\alpha,\mb{\psub}}(\init,2)=\frac{1}{2}\cdot\prn*{\sum_{l=1}^L\prn*{\frac{1}{2^l}\text{Unif}(\cS^l)}+\frac{1}{2^L}\text{Unif}(\{Z\})}+\frac{1}{2}\cdot\text{Unif}(\{X,Y\}).
          \] That is, from the
          initial state $\mathfrak{s}$, choosing action 1 always leads
          to state $W$ in the next time step (see the red arrow in
          \cref{fig:adm-illustration-3}), while choosing 2 leads to
          all states in $\cS^1\cup\cdots\cup\cS^L\cup\{X,Y,Z\}$ (i.e.,
          $\cS\setminus\{\init,W\}$) with certain probability (see the
          blue arrow in \cref{fig:adm-illustration-3}, {but note that transitions from $\mathfrak{s}$ to $\{X,Y\}$ are not displayed}). 
        \item \emph{Intermediate states}. Transitions from states in
          $\cS^1,\ldots,\cS^L$ are defined as follows.
          \begin{itemize}
            \item For each {$l^{{\rm th}}$-layer planted state} $s\in \psub^l\subseteq\cS^l$, define
            \[
            P_{L,\alpha,\mb{\psub}}(s,\initac)=\frac{\gamma^{L-l}\alpha}{1-(l-1)\alpha} \text{Unif}(\{X\})+\prn*{1-\frac{\gamma^{L-l}\alpha}{1-(l-1)\alpha}}\text{Unif}(\{Y\}).
            \]
        \item 
        For each {$l^{{\rm th}}$-layer unplanted states} $s\in \widebar{\psub}^l\subseteq\cS^l$, define
            \[
            P_{L,\alpha,\mb{\psub}}(s,\initac)=\frac{1-l\cdot\alpha}{1-(l-1)\alpha} \text{Unif}(\psub^{l+1})+{\frac{\alpha}{1-(l-1)\alpha}}\text{Unif}(\{Y\}),
            \]
            with the convention that $\psub^{L+1}\ldef\{Z\}$.
          \end{itemize}
          {Since we restrict to $\alpha \leq 1/L$, one can
            verify that these are valid probability distributions.}
        \item \emph{Terminal states}.
          All states in $\{W,X,Y,Z\}$ self-loop indefinitely. That is
          $P_{L,\alpha,\mb{\psub}}(s,{\mathfrak{a}}) =
          \text{Unif}(\{s\})$ for all $s \in \{W,X,Y,Z\}$.%
    \end{itemize}

\paragraph{Reward function} The initial and intermediate states have no reward, i.e., $R_{L,\alpha,w}(s,a)=0, \forall s\in\{\init\}\cup\cS^1\cdots\cup\cS^L,\forall
a\in\cA$. Each of
the self-looping terminal states in $\crl{W,X,Y,Z}$ has a fixed reward
determined by the parameters $L$, $\alpha$ and $w$. In particular, we
define $R_{L,\alpha,w}(W,{\mathfrak{a}})=w$, $R_{L,\alpha,w}(X,{\mathfrak{a}})=1$, $R_{L,\alpha,w}(Y,{\mathfrak{a}})=0$, and
$R_{L,\alpha,w}(Z,\mathfrak{a})=\alpha/(1-L\alpha)$.%
    
\paragraph{Initial state distribution}
All MDPs in $\{M_{L,\alpha,w,\mb{\psub}}\}$ start at $\mathfrak{s}$ deterministically
(that is, the initial state distribution $d_0$ places all its
probability mass on $\mathfrak{s}$). Since $d_0$ does not vary between instances, it should be thought of as \emph{known} to the learning algorithm.

\subsection{Specifying the MDP Family $\cM$}\label{sec:specify+}
{We leave $L\in\bbN$ (we interpret $\bbN$ to not include 0) as a free parameter until the end of \cref{app:admissible}, where we will give a concrete $L$ that leads to \cref{thm:admissible}.}
Given $L\in\bbN$, let $\alpha_1:=\frac{1}{2L}$ and
$\alpha_2:=\frac{1}{L+1}$. For $\alpha\in(0,1)$, define
\begin{equation}\label{eq:value}
    V_\alpha:=\sum_{l=1}^L\frac{1}{2^{l+1}}\frac{\gamma^{L-(l-1)}\alpha}{1-(l-1)\alpha}+\frac{1}{2^{L+1}}\frac{\alpha}{1-L\alpha}+\frac{1}{2},
\end{equation}
which has $0<V_{\alpha_1}<V_{\alpha_2}<1$, and let $w:=\frac{V_{\alpha_1}+V_{\alpha_2}}{2}$.
Define $\cI_{\mb{\theta}}:=\crl{\mb{\psub} : \abs{\psub^l}=\theta_l
  S_l}$ for any $\mb{\theta}=(\theta_1,\dots,\theta_L)\in(0,1)^L$ such
that $\theta_l S_l$ is an integer for all $l\in[L]$. We define two
sub-families of MDPs via
\begin{align*}
\cM_1:=\bigcup_{\mb{\psub}\in\cI_{\mb{\theta}\ind{1}}}\crl{M_{L,\alpha_1,w,\mb{\psub}}}, \quad \textrm{and} \quad 
\cM_2:=\bigcup_{\mb{\psub}\in\cI_{\mb{\theta}\ind{2}}}\crl{M_{L,\alpha_2,w,\mb{\psub}}}, 
\end{align*}
where $\cM_1$ is specified by
$\alpha_1$ and $\mb{\theta}\ind{1}=(\theta_{1}\ind{1},\dots,\theta_{L}\ind{1})$ with
\[\theta\ind{1}_{l}:=\frac{{\alpha_2}}{1-(l-1){\alpha_2}},~~\forall l\in[L],\] 
and $\cM_2$ is  specified by
$\alpha_2$ and $\mb{\theta}\ind{2}=(\theta_{1}\ind{2},\dots,\theta_{L}\ind{2})$ with
\[\theta_{l}\ind{2}:=\frac{{\alpha_1}}{1-(l-1){\alpha_1}},~~\forall
  l\in[L].\]
Finally, we define
the hard family $\cM$ via
\colt{
$\cM = \cM_1 \cup \cM_2$.
}
\arxiv{\[
\cM = \cM_1 \cup \cM_2.
\]}

{Note for this construction, that $\mb{\theta}\ind{1}$ is
  defined in terms of $\alpha_2$ and vice-versa, which is a crucial to
  the proof.} In addition, recall that we assume without loss of generality that $\frac{S-5}{L_{\rm div}}$ is an integer, which implies that $\theta_{l}\ind{i}S_l=\frac{S-5}{L_{\rm div}}(2L-(l-1))(L+1-(l-1))\theta_{l}\ind{i}$ is always an integer for any $i\in\{1,2\}$ and $l\in[L]$.

\subsection{Finishing the Construction: Value Functions and Data Distribution}\label{sec:finish+}

\paragraph{Value function class}
{Define functions $f_1,f_2: \cS \times \cA \to \mathbb{R}$ as
  follows, recalling that $w\ldef\frac{V_{\alpha_1}+V_{\alpha_2}}{2}$
  (differences are highlighted in blue):}
\begin{align}\label{eq:f1+}
    f_1(s,a) := \frac{1}{1-\gamma}\cdot\begin{cases}
   \gamma w,&s=\mathfrak{s},\;a=1\\
   \textcolor{blue}{\gamma V_{\alpha_1}},&s=\mathfrak{s},\;a=2\\
   \textcolor{blue}{\frac{\gamma^{L-(l-1)}\alpha_1}{1-(l-1)\alpha_1}},&{s\in\cS^l,\;l\in[L]}\\
    w,&s=W\\
    1,&s=X\\
    0,&s=Y\\
    \textcolor{blue}{\frac{\alpha_1}{1-L\alpha_1}},&s=Z\end{cases},\\
     f_2(s,a) := \frac{1}{1-\gamma}\cdot\begin{cases}
   \gamma w,&s=\mathfrak{s},\;a=1\\
   \textcolor{blue}{\gamma V_{\alpha_2}},&s=\mathfrak{s},\;a=2\\
   \textcolor{blue}{\frac{\gamma^{L-(l-1)}\alpha_2}{1-(l-1)\alpha_2}},&{s\in\cS^l,\;l\in[L]}\\
    w,&s=W\\
    1,&s=X\\
    0,&s=Y\\
    \textcolor{blue}{\frac{\alpha_2}{1-L\alpha_2}},&s=Z\end{cases}.
\end{align}
The following result is an elementary calculation. See \cref{sec:veryfication+} for a detailed calculation.
\begin{proposition}
  \label{prop:value_calculation+}For all $\pi:\cS\rightarrow\Delta(\cA)$,
  we have $Q^{\pi}_M=f_1$ for all
$M\in\cM_1$ and $Q^{\pi}_M=f_2$ for all $M\in\cM_2$.
\end{proposition}
It follows that by choosing $\cF =
  \{f_1,f_2\}$, all-policy realizability holds for all  $M \in
  \cM$.

\paragraph{Data distribution} Recall that in the offline RL setting,
the learner is provided with an \iid dataset
$D_n=\crl{(s_i,a_i,r_i,s'_i)}_{i=1}^n$ where $(s_i,a_i)\sim \mu$,
$s_i'\sim P(\cdot\mid s_i,a_i)$, and $r_i=R(s_i,a_i)$. To ensure admissibility of $\mu$, we consider the \emph{exploratory
  policy} $\pi_0$ given by
\[
\pi_0(s)=\text{Unif}(\cA),~\forall s\in\cS.
\]
We define the data collection distribution $\mu$ via:
  \[
{\mu(s,a) := \frac{1}{2}d_0^{\pi_0}(s,a)+\frac{1}{2}d^{\pi_{0}}_1(s,a),}
\]
which, by construction, is a mixture of admissible distributions as
desired. As a reminder, we use the notation $d_h^\pi \in
\Delta(\cS\times\cA)$ to denote the occupancy measure of $\pi$ at
time step $h$, that is $d_h^\pi(s,a) := \bbP^\pi\prn*{s_h=s,a_h=a}$, where the dependence on the MDP $M$ is suppressed.

In general, this choice of $\mu$ will depend on the
underlying MDP $M\in\cM$ through $d^{\pi_0}_1$. However, for our
specific construction, we calculate that
\begin{align*}
  {\mu(\cdot,\mathfrak{a})} &= \frac{1}{2}d_0+\frac{1}{2}\prn*{\frac{1}{2}P_{L,\alpha,\mb{\psub}}(\init,1)+\frac{1}{2}P_{L,\alpha,\mb{\psub}}(\init,2)}\\
    &= \frac{1}{2}d_0+\frac{1}{4}\text{Unif}(\{W\})+\frac{1}{4}\prn*{\frac{1}{2}\cdot\prn*{\sum_{l=1}^L\prn*{\frac{1}{2^l}\text{Unif}(\cS^l)}+\frac{1}{2^L}\text{Unif}(\{Z\})}+\frac{1}{2}\cdot\text{Unif}(\{X,Y\})}\\
    &=\frac{1}{8}\prn*{\sum_{l=1}^L\prn*{\frac{1}{2^l}\text{Unif}(\cS^l)}+\frac{1}{2^L}\text{Unif}(\{Z\})}+\frac{1}{2}\text{Unif}(\{\init\})+\frac{1}{4}\text{Unif}(\{W\})+\frac{1}{8}\text{Unif}(\{X,Y\}),
\end{align*}
which is in fact independent of the choice of $M\in\cM$.

In addition, by a straightforward calculation, we see that this choice of $\mu$ leads to the following
bound on the concentrability coefficient. 
See \cref{sec:veryfication+} for a detailed calculation.
\begin{proposition}
  \label{prop:conc_calculation+}We have $\Conc\le 32L$ for all models in $\cM$.
\end{proposition}

\subsection{Proof of Theorem \ref*{thm:admissible}}\label{sec:admproof}
Recall that for each $M\in\cM$, we let $\bbP^{\sM}_{n}$ denote the law of the offline
dataset $D_n$ when the underlying MDP is $M$, and we let $\bbE^{\sM}_{n}$
be the associated expectation operator.
\cref{lm:reduction+}, stated below, reduces the task of proving a policy
learning lower bound to the task of upper bounding the total variation distance between the mixture distributions $\bbP_{n}^1:=\frac{1}{\abs{\cM_1}}\sum_{M\in\cM_1}\bbP^{\sM}_{n}$ and $\bbP_{n}^2:=\frac{1}{\abs{\cM_2}}\sum_{M\in\cM_2}\bbP^{\sM}_{n}$.
\begin{lemma}\label{lm:reduction+}Consider any fixed $\gamma\in(0,1)$ and $L\in\bbN_{>1}$.  %
For any offline RL algorithm which takes $D_n=\crl{(s_i,a_i,r_i,s'_i)}_{i=1}^n$ as input and returns a stochastic policy $\wh{\pi}_{D_n}:\cS\rightarrow\Delta(\cA)$, we have
\[\sup_{M\in\cM}\crl*{J_{M}(\pi_M^\star)-\bbE^{\sM}_{n}\brk*{J_{M}(\wh{\pi}_{D_n})}}\ge\frac{\gamma^L}{16L}\frac{\gamma}{(1-\gamma)}\prn*{1-\Dtv{\bbP_{n}^1}{\bbP_{n}^2}}.\]
\end{lemma}

See \cref{sec:reduction+} for the proof of \pref{lm:reduction+}. We
conclude the proof of \pref{thm:admissible} by bounding the total
variation distance $\Dtv{\bbP_{n}^1}{\bbP_{n}^2}$. Because directly calculating the
total variation distance is difficult, we proceed in two steps. We first design two auxiliary reference measures $\bbQ_{n}^1$ and $\bbQ_{n}^2$, and then bound $\Dtv{\bbP_{n}^1}{\bbQ_{n}^1}$, 
$\Dtv{\bbP_{n}^2}{\bbQ_{n}^2}$ and $\Dtv{\bbQ_{n}^1}{\bbQ_{n}^2}$
separately. For the latter step, as in the proof of \cref{thm:main}, we move from total variation distance
to \chisquared, which we bound using similar arguments.
Our final bound on the total variation distance, which is proven in
\cref{sec:tv+}, is as follows.
\begin{lemma}\label{lm:tv+}Consider any fixed $\gamma\in(0,1)$ and $L\in\bbN$. 
  For all $n\le \sqrt[3]{(S-5)}/(20L^2)$, we have
\[\Dtv{\bbP_{n}^1}{\bbP_{n}^2}\le 1/2+n/{(8 \cdot 2^{L})}.\]
\end{lemma}

\cref{thm:admissible} immediately follows by choosing
\[L:=\floor*{\min\crl*{\frac{C}{32},\frac{1}{1-\gamma},
    \log_2(S)}}\]
and combining \cref{lm:reduction+} and \cref{lm:tv+}. With this choice of $L$, \pref{lm:tv+} implies that
$\Dtv{\bbP_{n}^1}{\bbP_{n}^2}\leq{}5/8$ whenever $n \leq c\cdot
\min\crl*{S^{1/3}/(\log S)^2,2^{C/32},2^{1/(1-\gamma)}}$
for a sufficiently small numerical constant $c$, and we have
$\Conc\leq{}C$ as desired by \pref{prop:conc_calculation+}. Finally, whenever $\gamma\geq{}1/2$, using
our choice for $L$ within \pref{lm:reduction+} gives
\[
\sup_{M\in\cM}\crl*{J_{M}(\pi_M^\star)-\bbE^{\sM}_{n}\brk*{J_{M}(\wh{\pi}_{D_n})}}\ge\bigom(1)\cdot{}\frac{\gamma^L}{L}\frac{\gamma}{(1-\gamma)}
= \bigom(1)
\]
where we use the fact that
\[
\frac{\gamma^L}{L(1-\gamma)}\ge\frac{\gamma^{1/(1-\gamma)}}{(1/(1-\gamma))(1-\gamma)}=\gamma^{1/(1-\gamma)}\ge(1/2)^2
\]
when $\gamma\in[1/2,1)$.

\section{Proof of \cref*{lm:reduction+}}\label{sec:reduction+}
We begin the proof by lower bounding the regret for any MDP in the family $\cM$. For
any $i\in\crl{1,2}$, any MDP $M\in\cM_i$, and any policy
$\pi:\cS\rightarrow\Delta(\cA)$, we have
\begin{align}\label{eq:lower+}
    J_{M}(\pistar_{\sM})-J_{M}({\pi})&={Q^\star_{\sM}(\init,\pistar_{\sM}(\init))-Q^\pi_{\sM}(\init,{\pi}(\init))}\notag\\ 
    &= {Q^\star_{\sM}(\init,\pistar_{\sM}(\init))-Q^\star_{\sM}(\init,{\pi}(\init))}\notag\\
    &= {Q^\star_{\sM}(\init,i)-Q^\star_{\sM}(\init,{\pi}(\init))}\notag\\
    &=\frac{\gamma}{1-\gamma}\frac{\abs{V_{\alpha_1}-V_{\alpha_2}}}{{2}}\bbP(\pi(\init)\ne i)\notag\\
    &\ge\frac{\gamma^L}{24L}\frac{\gamma}{(1-\gamma)}\bbP(\pi(\init)\ne i),
\end{align}
where the inequality follows because
\begin{align*}
\abs{V_{\alpha_1} - V_{\alpha_2}}&=\sum_{l=1}^L\frac{1}{2^{l-1}}\prn*{ \frac{\gamma^{L-(l-1)}\alpha_1}{{1-(l-1)\alpha_1}} - \frac{\gamma^{L-(l-1)}\alpha_2}{{1-(l-1)\alpha_2}}} + \frac{1}{2^{L+1}}\prn*{\frac{\alpha_1}{1-L\alpha_1} - \frac{\alpha_2}{1-L\alpha_2}}\\
&\ge \frac{1}{2}\gamma^L\abs{\alpha_1-\alpha_2}=\frac{\gamma^L}{2L}\prn*{\frac{1}{L+1}-\frac{1}{2L}}\ge\frac{\gamma^L}{12L}
\end{align*}
when $L\ge2$.

Now, consider any fixed offline reinforcement learning algorithm which takes
the offline dataset $D_n$ as an input and returns a stochastic policy
$\wh{\pi}_{\Dn}: \cS\rightarrow\Delta(\cA)$.
{For each $i \in \{1,2\}$, we apply} \cref{eq:lower+} to all MDPs in
$\cM_i$ and average to obtain
\begin{align*}
    \frac{1}{\abs{\cM_i}}\sum_{M\in\cM_i}\bbE^{M}_{n}\brk*{J_{M}(\pistar_{\sM})-J_{M}(\wh{\pi}_{D_n})}\ge \frac{\gamma^L}{24L}\frac{\gamma}{(1-\gamma)} \frac{1}{\abs{\cM_i}}\sum_{M\in\cM_{i}}\bbP^{M}_{n}(\wh{\pi}_{D_n}(\init)\ne i).
\end{align*}
Applying the inequality above for $i=1$ and $i=2$ and combining the results, we have
\begin{align*}%
    &\max_{M\in\cM}\bbE^{M}_{n}\brk*{J_{M}(\pistar_{\sM})-J_{M}(\wh{\pi}_{D_n})}\notag\\
    &\ge \frac{1}{2\abs{\cM_{1}}}\sum_{M\in\cM_{1}}\bbE^{M}_{n}\brk*{J_{M}(\pistar_{\sM})-J_{M}(\wh{\pi}_{D_n})}+ \frac{1}{2\abs{\cM_{2}}}\sum_{M\in\cM_{2}}\bbE^{M}_{n}\brk*{J_{M}(\pistar_{\sM})-J_{M}(\wh{\pi}_{D_n})}\notag\\
    &\ge\frac{\gamma^L}{48L}\frac{\gamma}{(1-\gamma)}\crl*{\frac{1}{\abs{\cM_{1}}}\sum_{M\in\cM_{1}}\bbP^{M}_{n}(\wh{\pi}_{D_n}(\init)\ne 1)+ \frac{1}{\abs{\cM_{2}}}\sum_{M\in\cM_{2}}\bbP^{M}_{n}(\wh{\pi}_{D_n}(\init)\ne 2)}\notag\\
    &\ge \frac{\gamma^L}{48L}\frac{\gamma}{(1-\gamma)}\prn*{1-\Dtv{\frac{1}{\abs{\cM_{1}}}\sum_{M\in\cM_{1}}\bbP^{M}_{n}}{\frac{1}{\abs{\cM_{2}}}\sum_{M\in\cM_{2}}\bbP^{M}_{n}}},
\end{align*}
where the last inequality follows because  $\bbP(E) + \mathbb{Q}(E^c)
\geq 1 - \Dtv{\bbP} {\mathbb{Q}}$ for any event $E$. \qed

\section{Proof of \cref*{lm:tv+}}\label{sec:tv+}

This proof is organized as follows. In \cref{sec:reference+}, we
introduce two reference measures and move from the total variation
distance to the $\chi^2$-divergence. This allows us to reduce the task
of upper bounding  $\Dtv{\bbP_{n}^1}{\bbP_{n}^2}$ to the task of upper
bounding two manageable density ratios
(\cref{eq:simplify1+,eq:simplify2+} in the sequel). We develop
several intermediate technical lemmas related to the density ratios in
\cref{sec:tech+}, and in \cref{sec:calc+} we put everything together to
bound the density ratios, thus completing the proof of \cref{lm:tv+}.

\subsection{Introducing Reference Measures and Moving to
  $\chi^2$-Divergence}
\label{sec:reference+}

Directly calculating the total variation distance $\Dtv{\bbP_{n}^1}{\bbP_{n}^2}$ 
is challenging, so we design two auxillary \emph{reference measures}
$\bbQ_{n}^1$ and $\bbQ_{n}^2$  which serves as intermediate quantities to help with the upper bound. 
The reference measures $\bbQ_{n}^1,\bbQ_n^2$ lies in the same measurable space as $\bbP_{n}^1$ and $\bbP_{n}^2$, and are defined as follows:
\begin{align*}
\bbQ_{n}^1\prn{\crl{\prn{s_i,a_i,r_i,s_i'}}_{i=1}^n}:=\prod_{i=1}^n \mu(s_i,a_i)\Ind_{\crl{r_i=R_1(s_i,a_i)}}P_0(s_i'\mid s_i,a_i),~~~\forall\; \crl{\prn{s_i,a_i,r_i,s_i'}}_{i=1}^n,\\
\bbQ_{n}^2\prn{\crl{\prn{s_i,a_i,r_i,s_i'}}_{i=1}^n}:=\prod_{i=1}^n \mu(s_i,a_i)\Ind_{\crl{r_i=R_2(s_i,a_i)}}P_0(s_i'\mid s_i,a_i),~~~\forall\; \crl{\prn{s_i,a_i,r_i,s_i'}}_{i=1}^n,
\end{align*}
where
\[
R_1(s,\initac):=\begin{cases}
0,&s\in\{\init\}\cup\cS^0\cup\cdots\cup\cS^L,\\
w,&s=W,\\
1,&s=X,\\
0,&s=Y,\\
\frac{\alpha_1}{1-L\alpha_1},&s=Z,
\end{cases},~~
R_2(s,\initac):=\begin{cases}
0,&s\in\{\init\}\cup\cS^0\cup\cdots\cup\cS^L,\\
w,&s=W,\\
1,&s=X,\\
0,&s=Y,\\
\frac{\alpha_2}{1-L\alpha_2},&s=Z,
\end{cases}
\]
and
\begin{align*}
& P_0(\nullstate,1)=\textrm{Unif}(\{W\}), \\[5pt]
& P_0(\init,2)=\frac{1}{2}\cdot\prn*{\sum_{l=1}^L\prn*{\frac{1}{2^l}\text{Unif}(\cS^l)}+\frac{1}{2^L}\text{Unif}(\{Z\})}+\frac{1}{2}\cdot\text{Unif}(\{X,Y\}),\\[5pt]
& \forall s\in\cS^l,\,\forall l\in[L] : P_0(s,\initac)=
\frac{(1-l\alpha_1)(1-l\alpha_2)}{(1-(l-1)\alpha_1)(1-(l-1)\alpha_2)} \text{Unif}(S^{l+1})\\
& ~~~~~~~~~~  +\frac{\gamma^{L-l}\alpha_1\alpha_2}{(1-(l-1)\alpha_1)(1-(l-1)\alpha_2)}\text{Unif}(\{X\})\\
& ~~~~~~~~~~  +\prn*{1-\frac{(1-l\alpha_1)(1-l\alpha_2)}{(1-(l-1)\alpha_1)(1-(l-1)\alpha_2)}-\frac{\gamma^{L-l}\alpha_1\alpha_2}{(1-(l-1)\alpha_1)(1-(l-1)\alpha_2)}}\text{Unif}(\{Y\}),\\
& \forall s\in\{W,X,Y,Z\}: P_0(s,\initac)=\text{Unif}(\{s\}).
\end{align*}
The reference measure $\bbQ_{n}^1$ is the law of $D_n$ when
the data collection distribution is $\mu$ and the underlying MDP is
$\widebar{M}_1:=(\cS,\cA,P_0,R_1,\gamma,d_0)$. Notably, $\widebar{M}_1$ shares the same reward function with all MDPs in $\cM_1$, and differs from the MDPs in $\cM_1$ only in terms of the transition operator $P_0$. 

{There are two ways to understand $P_0$. Operationally, $P_0$ is simply the pointwise \emph{average} transition operator of the MDPs in $\cM_1$, in the sense that
\[
\forall s \in \cS, a \in \cA: P_0(\cdot \mid s,a) =\frac{1}{\abs{\cM_1}}\sum_{M\in\cM_1}P_{M}(\cdot \mid s,a),
\]
where $P_M$ is the transition operator associated with each MDP $M$. For this reason, we call $\widebar{M}_1$ the \emph{average MDP} associated with $\cM_1$. 
More conceptually, $P_0$ is the transition operator obtained by performing state aggregation using the value function class $\cF = \{f_1,f_2\}$, where states with the same values for both $f_1$ and $f_2$ are viewed as identical and constrained to share dynamics (which is induced by averaging over the data collection distribution). 
}

Similarly, the reference measure $\bbQ_{n}^2$ can be understood as the law of $D_n$ when
the data collection distribution is $\mu$ and the underlying MDP is
$\widebar{M}_2:=(\cS,\cA,P_0,R_2,\gamma,d_0)$, where $\widebar{M}_2$ is the average MDP associated with $\cM_2$. 
An important property is that $\widebar{M}_1$ and $\widebar{M}_2$ share the same transition operator $P_0$ and differs only in terms of the reward on state $Z$. 
This is a consequence of our construction, as when we construct $\cM_1$ and $\cM_2$ we strive to ensure that 
\[
\forall s \in \cS, a \in \cA: \frac{1}{\abs{\cM_1}}\sum_{M\in\cM_1}P_{M}(\cdot \mid s,a) = P_0(\cdot \mid s,a) = \frac{1}{\abs{\cM_2}}\sum_{M\in\cM_2}P_{M}(\cdot \mid s,a),
\]
and there is no uncertainty in the reward function outside of state $Z$.

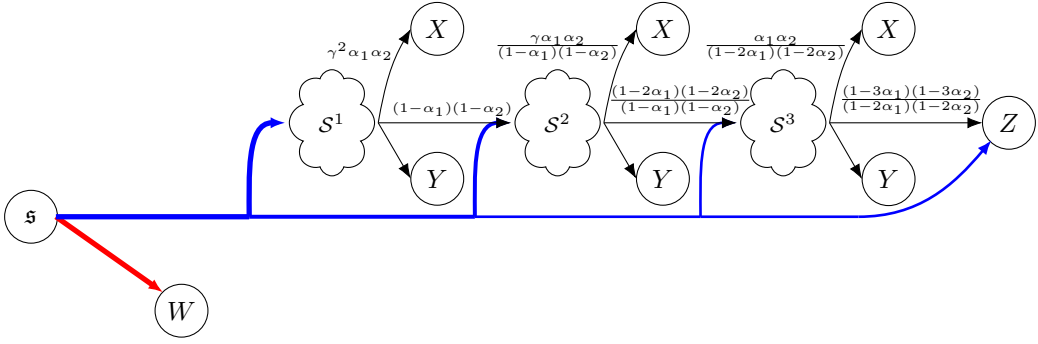
\begin{figure}[htbp]
    \centering
\newcommand{\ellipse}[4]{%
  \node[cloud,cloud puffs=8,draw,minimum width=5pt, minimum height=10pt,aspect=0.65, text opacity=1.0] () at (#1,#2) {#3};
  \coordinate[] (#4in) at (#1-17.5pt,#2) {};
  \coordinate[] (#4out) at (#1+17.5pt,#2) {};
  \coordinate[] (#4tr) at (#1+13pt,#2+13pt) {};
  }
\newcommand{\xnode}[4]{
  \draw (#1,#2) circle (10pt) node {#3};
  \coordinate[] (#4in) at (#1-10pt, #2);
  }
\newcommand{\ynode}[3]{
  \draw (#1,#2) circle (10pt) node {$Y$};
  \coordinate[] (#3in) at (#1-10pt, #2);
}
\newcommand{\znode}[3]{
  \draw (#1,#2) circle (10pt) node {$Z$};
  \coordinate[] (#3in) at (#1-10pt, #2);
  \coordinate[] (#3bl) at (#1-6.5pt, #2-6.5pt);
}
\begin{tikzpicture}
\draw[draw=black] (-2,0) circle (10pt) node {$W$};
\coordinate[] (wr) at (-2.24,0.24) {};
\draw[draw=black] (-4,1.25) circle (10pt) node {$\mathfrak{s}$};
\coordinate[] (sl) at (-3.67,1.25) {};
\coordinate[] (s0) at (-3.67,1.25) {};
\coordinate[] (s01) at (-1.1,1.25) {};
\coordinate[] (s02) at (1.9,1.25) {};
\coordinate[] (s03) at (4.9,1.25) {};
\coordinate[] (s04) at (7.0,1.25) {};
\draw [draw=red, -{Latex[length=2mm,color=red]},line width=2] (sl) to (wr);
\draw [draw=blue, -, line width=2] (s0) edge (s01);
\draw [draw=blue, -, line width=1.5] (s01) edge (s02);
\draw [draw=blue, -, line width=1.0] (s02) edge (s03);
\draw [draw=blue,line width=1.0] (s03) to [out=0,in=180] (s04);

\ellipse{0cm}{2.5cm}{{\footnotesize $\cS^1$}}{s1}
\xnode{1.4cm}{3.75cm}{$X$}{x1}
\ynode{1.4cm}{1.755cm}{y1}
\draw [draw=blue, -{Latex[length=2mm,color=blue]},line width=2] (s01) to [out=90,in=180] (s1in);
\draw [draw=black, -{Latex[length=2mm,color=black]}] (s1out) to [out=80,in=240]  (x1in);
\draw [draw=black, -{Latex[length=2mm,color=black]}] (s1out) to  (y1in);
\node at (1.5cm-1.13cm,3.45) {\tiny $\gamma^{2}\alpha_1\alpha_2$};

\ellipse{3cm}{2.5cm}{{\footnotesize $\cS^2$}}{s2}
\xnode{4.4cm}{3.75cm}{$X$}{x2}
\ynode{4.4cm}{1.755cm}{y2}
\draw [draw=blue, -{Latex[length=2mm,color=blue]},line width=1.5] (s02) to [out=90,in=180] (s2in);
\draw [draw=black, -{Latex[length=2mm,color=black]}] (s1out) to (s2in);
\node at (1.6,2.68) {\tiny $(1{-}\alpha_1)(1{-}\alpha_2)$};
\draw [draw=black, -{Latex[length=2mm,color=black]}] (s2out) to [out=80,in=240] (x2in);
\draw [draw=black, -{Latex[length=2mm,color=black]}] (s2out) to  (y2in);
\node at (4.5cm-1.5cm,3.5) {\tiny $\frac{\gamma\alpha_1\alpha_2}{(1{-}\alpha_1)(1{-}\alpha_2)}$};

\ellipse{6cm}{2.5cm}{{\footnotesize $\cS^3$}}{s3}
\xnode{7.4cm}{3.75cm}{$X$}{x3}
\ynode{7.4cm}{1.75cm}{y3}
\draw [draw=blue, -{Latex[length=2mm,color=blue]},line width=1.0] (s03) to [out=90,in=180] (s3in);
\draw [draw=black, -{Latex[length=2mm,color=black]}] (s2out) to (s3in);
\node at (4.63,2.8) {\tiny $\frac{(1{-}2\alpha_1)(1{-}2\alpha_2)}{(1{-}\alpha_1)(1{-}\alpha_2)}$};
\draw [draw=black, -{Latex[length=2mm,color=black]}] (s3out) to [out=80,in=240] (x3in);
\draw [draw=black, -{Latex[length=2mm,color=black]}] (s3out) to  (y3in);
\node at (7.5cm-1.6cm,3.5) {\tiny $\frac{\alpha_1\alpha_2}{(1{-}2\alpha_1)(1{-}2\alpha_2)}$};

\znode{9cm}{2.5cm}{z}
\draw [draw=black, -{Latex[length=2mm,color=black]}] (s3out) to  (zin);
\node at (7.7,2.81) {\tiny $\frac{(1-3\alpha_1)(1-3\alpha_2)}{(1-2\alpha_1)(1-2\alpha_2)}$};

\draw [draw=blue, -{Latex[length=2mm,color=blue]},line width=1.0] (s04) to [out=0,in=230] (zbl);

\end{tikzpicture}

    \caption{Illustration of the average MDP with $L=3$.}
    \label{fig:aggregated-3}
\end{figure}
\cref{fig:aggregated-3} illustrates the average MDPs $\widebar{M}_1$ and $\widebar{M}_2$ (the only difference between $\widebar{M}_1$ and $\widebar{M}_2$ is the reward on state $Z$, which is not displayed). Note that for each $l\in[L]$, all intermediate states in $\cS^l$ have the same dynamics, so the planted subset structure is erased by averaging/aggregating.

Starting with the triangle inequality for the total variation distance, we have
\begin{align}
\Dtv{\bbP_{n}^1}{\bbP_{n}^2}&\le \Dtv{\bbP_{n}^1}{\bbQ_{n}^1}+\Dtv{\bbP_{n}^2}{\bbQ_{n}^2}+\Dtv{\bbQ_{n}^1}{\bbQ_{n}^2}\notag\\
&\le\frac{1}{2}\sqrt{D_{\chi^2}\prn{\bbP_{n}^1\dmid\bbQ_{n}^1}}+\frac{1}{2}\sqrt{D_{\chi^2}\prn{\bbP_{n}^2\dmid\bbQ_{n}^2}}+\Dtv{\bbQ_{n}^1}{\bbQ_{n}^2},\label{eq:tvx}
\end{align}
where the second inequality follows from the fact that
$\Dtv{\bbP}{\bbQ}\le\frac{1}{2}\sqrt{D_{\chi^2}\prn{\bbP\dmid\bbQ}}$
for any $\bbP,\bbQ$ (see Proposition 7.2 or Section 7.6 of
\cite{polyanskiy}). 

{The next lemma shows that the total variation distance between $\bbQ_n^1$ and $\bbQ_n^2$ is small. Intuitively, this is because the average MDPs $\widebar{M}_1$ and $\widebar{M}_2$ only differ in the reward on state $Z$, but the data distribution $\mu$'s coverage of  on $Z$ is very small.}
\begin{lemma}\label{lm:tvref}
For all $n<\infty$, we have $ \Dtv{\bbQ_{n}^1}{\bbQ_{n}^2}\le n\mu(Z,\initac)=n/(8\times 2^L)$.
\end{lemma}
\begin{proof}Let $\cR:=\{1,0,\alpha_1/(1-L\alpha_1),\alpha_2/(1-L\alpha_2),R(W,\initac)\}$, then $R(s,a)\in\cR$ for all $(s,a)\in\cS\times\cA$. Since $\abs{\cS},\abs{\cA},\abs{\cR}<\infty$, the realization of the offline dataset $D_n=\crl{(s_i,a_i,r_i,s'_i)}_{i=1}^n$ only has finitely many possible outcomes, and we have
\begin{align*}
    &\Dtv{\bbQ_{n}^1}{\bbQ_{n}^2}\\
    &=\frac{1}{2}\sum_{(s_i,a_i,r_i,s'_i)\in\cS\times\cA\times\cR\times\cS,\,\forall i\in[n]}{\abs*{\bbQ_{n}^1(\crl{(s_i,a_i,r_i,s'_i)}_{i=1}^n)-{\bbQ_{n}^2(\crl{(s_i,a_i,r_i,s'_i)}_{i=1}^n)}}}\\
    &=\frac{1}{2}\sum_{(s_i,a_i,r_i,s'_i)\in\cS\times\cA\times\cR\times\cS,\,\forall i\in[n]}{\prod_{i=1}^n\mu(s_i,a_i){P_0(s_i'\mid s_i,a_i)}\abs*{\prod_{i=1}^n\Ind_{\crl{r_i=R_1(s_i,a_i)}}-\prod_{i=1}^n\Ind_{\crl{r_i=R_2(s_i,a_i)}}}}\\
    &=\frac{1}{2}\sum_{(s_i,a_i,r_i)\in\cS\times\cA\times\cR,\,\forall i\in[n]}\prod_{i=1}^n\mu(s_i,a_i)\abs*{\prod_{i=1}^n\Ind_{\crl{r_i=R_1(s_i,a_i)}}-\prod_{i=1}^n\Ind_{\crl{r_i=R_2(s_i,a_i)}}}\\
    &=\sum_{(s_i,a_i)\in\cS\times\cA,\,\forall i\in[n]}\Ind_{\crl{\exists i\in[n]\text{ s.t. }s_i=Z}}\prod_{i=1}^n\mu(s_i,a_i)\\
    &=\bbP_{s_1,\dots,s_n\sim\mu}\prn*{\crl{\exists i\in [n] \text{ s.t. }s_i=Z}}=\bbP_{s_1,\dots,s_n\sim\mu}\prn*{\crl{s_1=Z}\cup\cdots\cup\crl{s_n=Z}}\le n\mu(\crl{Z}),
\end{align*}
where the first equality follows from the well-known identity between the total variation distance and the $L_1$ norm and the last inequality follows from a union bound.
\end{proof}

Using \cref{lm:tvref}, we have
\begin{align}
\Dtv{\bbP_{n}^1}{\bbP_{n}^2}
&\le \frac{1}{2}\sqrt{D_{\chi^2}\prn{\bbP_{n}^1\dmid\bbQ_{n}^1}}+\frac{1}{2}\sqrt{D_{\chi^2}\prn{\bbP_{n}^2\dmid\bbQ_{n}^2}}+n\mu(Z,\initac).\label{eq:tv_triangle+}
\end{align}
{Note that $\mu(Z,\initac) = 1/8\cdot 1/2^{L}$ which produces the final term in the bound in~\cref{lm:tv+}.}

We now turn our focus to the $\chi^2$-divergence, which we expand as
\begin{align}\label{eq:simplify1+}
&D_{\chi^2}\prn{\bbP_{n}^1\dmid\bbQ_{n}^1}\notag\\
&=\bbE_{\crl{\prn{s_i,a_i,r_i,s_i'}}_{i=1}^n\sim \bbQ_{n}^1}\brk*{\prn*{\frac{\frac{1}{|\cM_1|}\sum_{M\in\cM_1}\bbP^M_{n}(\crl{\prn{s_i,a_i,r_i,s_i'}}_{i=1}^n)}{\bbQ_{n}^1(\crl{\prn{s_i,a_i,r_i,s_i'}}_{i=1}^n)}}^2}-1\notag\\
&=\bbE_{\crl{\prn{s_i,a_i,r_i,s_i'}}_{i=1}^n\sim\bbQ_{n}^1}\brk*{\prn*{\frac{{\frac{1}{|\cM_1|}\sum_{M\in\cM_1}\prod_{i=1}^n\mu(s_i,a_i)\Ind_{\crl{r_i=R_M(s_i,a_i)}}P_M(s_i'\mid s_i,a_i)}}{\prod_{i=1}^n\mu(s_i,a_i)\Ind_{\crl{r_i=R_1(s_i,a_i)}}{P_0(s_i'\mid s_i,a_i)}}}^2}-1\notag\\
  &=\bbE_{\crl{\prn{s_i,a_i,r_i,s_i'}}_{i=1}^n\sim \bbQ_{n}^1}\brk*{\prn*{\frac{{\frac{1}{|\cM_1|}\sum_{M\in\cM_1}\prod_{i=1}^nP_M(s_i'\mid s_i,a_i)}}{\prod_{i=1}^n{P_0(s_i'\mid s_i,a_i)}}}^2}-1\notag\\
&=\frac{1}{\abs{\cM_1}^2}\sum_{M,M'\in\cM_{1}}\bbE_{\crl{\prn{s_i,a_i,r_i,s_i'}}_{i=1}^n\sim \bbQ_{n}^1}\brk*{{\frac{\prod_{i=1}^n{P_M(s_i'\mid s_i,a_i)P_{M'}(s_i'\mid s_i,a_i)}}{\prod_{i=1}^n{P_0^2(s_i'\mid s_i,a_i)}}}}-1\notag\\
&=\frac{1}{\abs{\cM_1}^2}\sum_{M,M'\in\cM_{1}}\prn*{\bbE_{\substack{(s,a)\sim\mu,\\s'\sim P_0(\cdot\mid s,a)}}\brk*{{\frac{P_M(s'\mid s,a)P_{M'}(s'\mid s,a)}{P_0^2(s'\mid s,a)}}}}^n-1,
\end{align}
where the third equality follows from
$R_M(s,a)=R_1(s,a), \forall M\in\cM,\forall a\in\cA, \forall s\in\cS$.

Using an identical calculation, we also have
\begin{align}\label{eq:simplify2+}
D_{\chi^2}\prn{\bbP_{n}^2\dmid\bbQ_{n}^2}=\frac{1}{\abs{\cM_2}^2}\sum_{M,M'\in\cM_{2}}\prn*{\bbE_{\substack{(s,a)\sim\mu,\\s'\sim P_0(\cdot\mid s,a)}}\brk*{{\frac{P_M(s'\mid s,a)P_{M'}(s'\mid s,a)}{P_0^2(s'\mid s,a)}}}}^n-1.
\end{align}

Equipped with these expressions for the \chisquared,
the next step in the proof of \pref{lm:tv+} is to upper bound the
right-hand side for \cref{eq:simplify1+,eq:simplify2+}. This is done in
\cref{sec:calc+}, but before proceeding we require several intermediate technical lemmas.

\subsection{Technical Lemmas for Density Ratios}\label{sec:tech+}
{For this section only, we focus on MDPs in $\cM_1$ and suppress the subscript indexing the subfamily, i.e., we use $\mb{\theta}$ for $\mb{\theta}^{(1)}$ and $\alpha$ for $\alpha_1$. Exactly the same calculations apply for $\cM_2$, which we will use in the next section.
  To simplify the presentation and re-use lemmas from~\cref{sec:tech} it will be helpful to define the following notation:
  \begin{align*}
    \mb{\alpha} &= (\alpha_1,\ldots,\alpha_L), \qquad \alpha_l = \frac{\gamma^{L-l}\alpha}{1-(l-1)\alpha}\\
    \mb{\beta} &= (\beta_1,\ldots,\beta_L), \qquad \beta_l = 1-\alpha_l
  \end{align*}
  Additionally recall that
  \begin{align*}
    \mb{\theta} = (\theta_1,\ldots,\theta_L), \qquad \theta_l = \frac{\alpha}{1-(l-1)\alpha}
  \end{align*}
  These vectors parametrize the MDP transitions in the following sense: Let $\mb{\psub} \in \cI_{\mb{\theta}}$ denote the choice of planted states for each layer. Then for $l \in [L]$ we have:
  \begin{align*}
    s \in I^l: &~~ P_{L,\alpha,\mb{\psub}}(s,\initac) = \alpha_l \text{Unif}(\{X\}) + \beta_l \text{Unif}(\{Y\})\\
    s \in \bar{I}^l: &~~ P_{L,\alpha,\mb{\psub}}(s,\initac) = (1-\theta_l) \text{Unif}(I^{l+1}) + \theta_l \text{Unif}(\{Y\})
  \end{align*}
  where $I^{L+1} = \{Z\}$. 
}

To state the results compactly, we define
\begin{equation}
\phi^l_{\mb{\theta},\mb{\alpha},\mb{\beta}}:=\theta_l^2\prn*{\frac{(\beta_l-\alpha_l)^2}{\theta_l(\beta_l-\alpha_l)+1-\beta_l}+\frac{\theta_l(\beta_l-\alpha_l)+\alpha_l}{\theta_l(1-\theta_l)}}.\label{eq:phi+}
\end{equation}
{We also use $P_{\mb{\psub}}$ to denote $P_{L,\alpha,\mb{\psub}}$.}

{We will bound the density ratio terms for each layer separately. First we control the $L^{\mathrm{th}}$ layer.}
\begin{lemma}\label{lm:ratio1+}
For any $\mb{\psub},\mb{\qsub} \in\cI_{\mb{\theta}}$, we have
\begin{align*}
\bbE_{\substack{ s\sim{\rm Unif}(\cS^L),\\s'\sim P_0(\cdot\mid s{, \initac})}}\brk*{
    \frac{ P_{\mb{\psub}}(s'\mid s{, \initac}) P_{\mb{\qsub}}(s'\mid s{, \initac})}{P_0^2(s'\mid s{,
        \initac})}}=1 + \phi^L_{\mb{\theta},\mb{\alpha},\mb{\beta}}\cdot{}\prn*{\frac{\abs{\psub^L\cap \qsub^L}}{\theta_L^2{S_L}}-1}.
\end{align*}
\end{lemma}
We omit the proof, which is identical to that of \cref{lm:ratio1}. {Next we turn to intermediate layers.}
\begin{lemma}\label{lm:ratio2+}
For any $\mb{\psub},\mb{\qsub}\in\cI_{\mb{\theta}}$, for any $l\in[L-1]$, we have
\begin{align*}
\bbE_{\substack{s\sim{\rm Unif}(\cS^l),\\s'\sim P_0(\cdot\mid s{, \initac})}}\brk*{
    \frac{ P_{\mb{\psub}}(s'\mid s{, \initac}) P_{\mb{\qsub}}(s'\mid s{, \initac})}{P_0^2(s'\mid s{,
        \initac})}}\le1 + \phi^l_{\mb{\theta},\mb{\alpha},\mb{\beta}}\cdot{}\prn*{\frac{\abs{\psub^l\cap\qsub^l}}{\theta_l^2{S_l}}-1}+\prn*{\frac{\abs{\psub^{l+1}\cap\qsub^{l+1}}}{\theta_{l+1}^2 S_{l+1}}-1}_+.
\end{align*}
\end{lemma}

\begin{proof}
For any $\mb{\psub},\mb{\qsub}\in\cI_{\mb{\theta}}$, for any $l\in[L-1]$, we observe that
\begin{align*}
    \bbE_{\substack{s\sim\textrm{Unif}(\cS^l),\\s'\sim P_0(\cdot\mid s{, \initac})}}\brk*{ \frac{ P_{\mb{\psub}}(s'\mid s{, \initac})P_{\mb{\qsub}}(s'\mid s{, \initac})}{P_0^2(s'\mid s{, \initac})}}
    &=\bbE_{s\sim {\rm Unif}(\cS^l)}\brk*{ {\sum_{s'\in \crl{X,Y}\cup (\psub^{l+1}\cap\qsub^{l+1})}\frac{ P_{\mb{\psub}}(s'\mid s{, \initac}) P_{\mb{\qsub}}(s'\mid s{, \initac})}{P_0(s'\mid s{, \initac})}}}.
\end{align*}
To proceed, we calculate the value of the ratio $\frac{ P_\psub(s'\mid s{, \initac}) P_{\qsub}(s'\mid s{,
    \initac})}{P_0(s'\mid s{, \initac})}$ for each possible choice for
$s\in\cS^1$ and $s'\in\crl{X,Y}\cup \prn{\psub^{l+1}\cap\qsub^{l+1}}$ in \pref{tb:val+} below.
\begin{table}[h]\doublespacing
\begin{center}
\begin{tabular}{|c | c | c | c|}
\hline
& $s' = X$ & $s'=Y$ & $s'\in \psub^{l+1}\cap\qsub^{l+1}$\\
\hline\hline
$s \in \psub^l \cap \qsub^l$ & $\alpha_l/\theta_l$ &${\beta_l^2}/(\theta_l{\beta_l}+(1-\theta_l){\alpha_l})$ &   0 \\
\hline
$s \in (\psub^l \cup \qsub^l) \setminus (\psub^l \cap \qsub^l)$ &  0 &${\beta_l}{\alpha_l}/(\theta_l{\beta_l}+(1-\theta_l){\alpha_l})$ &   0 \\
\hline
$s \notin (\psub^l \cup \qsub^l)$ & 0 &${\alpha_l^2}/(\theta_l{\beta_l}+(1-\theta_l){\alpha_l})$ &  $\frac{\beta_l}{(1-\theta_l)}{\cdot\frac{\abs{\psub^{l+1}\cap\qsub^{l+1}}}{\theta_{l+1}^2S_{l+1}}}$\\
\hline
\end{tabular}
\caption{Value of $\frac{ P_{\mb{\psub}}(s'\mid s, \initac) P_{\mb{\psub}'}(s'\mid s,
    \initac)}{P_0(s'\mid s, \initac)}$ for all possible pairs $(s,s')$.}\label{tb:val+}
\end{center}
\end{table}

Define $t_l\ldef\abs{\psub^l\cap \qsub^l}$. From
\cref{lm:cap-bound}, we must have $t_l\in[(2\theta_l-1)_+S_l,\theta_l
S_l]$. We also have $\abs{\psub^l\cup\qsub^l}=\abs{\psub^l}+\abs{\qsub^l}-\abs{\psub^l\cap\qsub^l}=2\theta_l S_l-t_l$. Hence, the event in the first row of \cref{tb:val+} occurs with probability $\abs{\psub^l\cap\qsub^l}/S_l=t_l/S_l$,
the event in the second row occurs with probability $\abs{(\psub^l \cup \qsub^l) \setminus (\psub^l \cap \qsub^l)}/S_l=(2\theta_l S_l-2t_l)/S_l$
and the event in the third row
occurs with probability $\abs{S_l\setminus\prn{\psub^l\cup\psub^l}}/S=((1-2\theta_l)S_l+t_l)/S_l$. 
Using these values and performing a similar calculation to the one in the proof of \cref{lm:ratio1}, we obtain
\newcommand{\groupi}{\text{(i)}}
\newcommand{\groupii}{\text{(ii)}}
\newcommand{\groupiii}{\text{(iii)}}
\begin{align*}
& \bbE_{s\sim {\rm Unif}(\cS^l)}\brk*{ {\sum_{s'\in\crl{X,Y}\cup\prn{\psub^{l+1}\cap\qsub^{l+1}}}\frac{ P_{\mb{\psub}}(s'\mid s{, \initac}) P_{\mb{\qsub}}(s'\mid s{, \initac})}{P_0(s'\mid s{, \initac})}}}\\
&=1+\phi^l_{\mb{\theta},\mb{\alpha},\mb{\beta}}\prn*{\frac{t_l}{\theta_{l}^2S_l}-1}{+\prn*{1-2\theta_l+\frac{t_l}{S_l}}\frac{\beta_l}{1-\theta_l}\prn*{\frac{t_{l+1}}{\theta_{l+1}^2 S_{l+1}}-1}}\\
&\le1+\phi^l_{\mb{\theta},\mb{\alpha},\mb{\beta}}\prn*{\frac{t_l}{\theta_{l}^2S_l}-1}+\prn*{\frac{t_{l+1}}{\theta_{l+1}^2 S_{l+1}}-1}_+,
\end{align*}
where the last inequality follows from ${(2\theta_l -1)S_l \leq} t_l\le \theta_lS_l$ (which implies ${0 \leq} 1-2\theta_l+t_l/S_l\le1-\theta_l$) and ${0 \leq } \beta_l\le 1$.
\end{proof}

\subsection{Completing the Proof}\label{sec:calc+}

{For now, let us also focus on a single MDP subfamily $\cM_1$ and suppress the family indices associated with $\alpha$ and $(\mb{\theta},\mb{\alpha},\mb{\beta})$. As above the same calculations apply to $\cM_2$.}
To keep notation compact, for any $d\in\Delta(\cS\times\cA)$, define
\begin{align*}
\textsf{DR}_{M,M'}(d):=\bbE_{\substack{(s,a)\sim d,\\s'\sim P_0(\cdot\mid s,a)}}\brk*{\frac{P_{M}(s'\mid s{,a}) P_{M'}(s'\mid s{,a})}{P_0^2(s'\mid s{,a})}}.
\end{align*}

Consider any $M,M'\in\cM_1$. 
For any $\pi:\cS\rightarrow\Delta(\cA)$, by \cref{lm:ratio1+,lm:ratio2+}, we have
\begin{align}\label{eq:drs}
    &\sum_{l=1}^L\frac{1}{2^{l}}\textsf{DR}_{M,M'}(\text{Unif}(\cS^l)\times\pi)%
    \le  \sum_{l=1}^L\frac{1}{2^{l}} +  \sum_{l=1}^L\frac{1}{2^{l}}\phi_{\mb{\theta},\mb{\alpha},\mb{\beta}}^l\prn*{\frac{t_l}{\theta_l^2S_l}-1}+\sum_{l=2}^L\frac{1}{2^{l-1}}\prn*{\frac{t_l}{\theta_l^2S_l}-1}_+\notag\\
    &\le 1-\frac{1}{2^L}+  \sum_{l=1}^L\frac{1}{2^{l}}\phi_{\mb{\theta},\mb{\alpha},\mb{\beta}}^l\prn*{\frac{t_l}{\theta_l^2S_l}-1}+\sum_{l=2}^L\frac{1}{2^{l-1}}\prn*{\frac{t_l}{\theta_l^2S_l}-1}_+\notag\\
    & \le 1-\frac{1}{2^L}+  \sum_{l=1}^L\frac{\phi_{\mb{\theta},\mb{\alpha},\mb{\beta}}^l+2}{2^{l}}\prn*{\frac{t_l}{\theta_l^2S_l}-1}_+.
\end{align}
Note that $P_M(\cdot\mid s,a)$ and  $P_{M'}(\cdot\mid s,a)$ differ from
$P_0(\cdot\mid s,a)$ only when $(s,a)=(\init,2)$ or
$s\in\prn{\crl{\init}\cup\cS^1\cup\cdots\cup\cS^L}$, so, recalling the value of $\mu$, we have
\begin{align*}
    &\bbE_{\substack{(s,a)\sim\mu,s'\sim P_0(\cdot\mid s,a)}}\brk*{{\frac{P_M(s'\mid s,a)P_{M'}(s'\mid s,a)}{P_0^2(s'\mid s,a)}}}\\
    &=\frac{1}{8}\sum_{l=1}^L\frac{1}{2^{l}}\textsf{DR}_{M,M'}(\text{Unif}(\cS^l)\times\pi_0)+\frac{1}{8}\frac{1}{2^L}\textsf{DR}_{M,M'}(\text{Unif}(\{Z\})\times\pi_0)\\
    &~~~~~~+\frac{1}{2}\textsf{DR}_{M,M'}(\text{Unif}(\{\init\})\times\pi_0)+\frac{1}{4}\textsf{DR}_{M,M'}(\text{Unif}(\{W\})\times\pi_0)+\frac{1}{8}\textsf{DR}_{M,M'}(\text{Unif}(\{X,Y\})\times\pi_0)\\
    &\le  \frac{1}{8}\prn*{1-\frac{1}{2^L}+  \sum_{l=1}^L\frac{\phi_{\mb{\theta},\mb{\alpha},\mb{\beta}}^l+2}{2^{l}}\prn*{\frac{t_l}{\theta_l^2S_l}-1}_+}+\frac{1}{8}\frac{1}{2^{L}}+\frac{1}{2}+\frac{1}{4}+\frac{1}{8}\\
    &=1+ \sum_{l=1}^L\frac{\phi_{\mb{\theta},\mb{\alpha},\mb{\beta}}^l/8+1/4}{2^{l}}\prn*{\frac{t_l}{\theta_l^2S_l}-1}_+,
\end{align*}
where the first inequality follows from \cref{eq:drs}. As a result, we have
\begin{align}\label{eq:tohpg}
&\frac{1}{\abs{\cM_1}^2}\sum_{M,M'\in\cM_{1}}\prn*{\bbE_{\substack{(s,a)\sim\mu,~~~~~\\s'\sim P_0(\cdot\mid s,a)}}\brk*{{\frac{P_M(s'\mid s,a)P_{M'}(s'\mid s,a)}{P_0^2(s'\mid s,a)}}}}^n\notag\\
&\le\prn*{\prod_{l=1}^L\frac{1}{{ S_l \choose \theta_l  S_l}^2 }}\sum_{t_1,\dots,t_L}\sum_{\mb{\psub},\mb{\psub}'\in\cI_{\mb{\theta}}:\abs{\psub_l\cap \psub_l'}=t_l}\prn*{1+ \sum_{l=1}^L\frac{\phi_{\mb{\theta},\mb{\alpha},\mb{\beta}}^l/8+4}{2^{l}}\prn*{\frac{t_l}{\theta_l^2S_l}-1}_+}^n\notag\\
&=\bbE_{t_l\sim \mathrm{Hyper}(\theta_l S_l, S_l, \theta S_l),\,\forall l\in[L]}\brk*{\prn*{1+ \sum_{l=1}^L\frac{\phi_{\mb{\theta},\mb{\alpha},\mb{\beta}}^l/8+1/4}{2^{l}}\prn*{\frac{t_l}{\theta_l^2S_l}-1}_+}^n},
\end{align}
where $\mathrm{Hyper}(\cdot,\cdot,\cdot)$ denotes the hypergeometric distribution (cf. \cref{lm:hypergeo} for background).

By \cref{lm:hypergeo}, for any $l\in[L]$, the event
\[
E_l:=\crl*{t_l\ge(\theta_l+\eps_l)\theta_l S_l}.
\] happens with probability at most $\exp\prn*{-2\eps_l^2\theta_l S_l}$. Hence, the event
\[
E_{\text{bad}}:=\crl*{\exists l\in[L], t_l\ge(\theta_l+\eps_l)\theta_l S_l} =\bigcup_{l=1}^L E_l
\]
happens with probability at most $\sum_{l=1}^L\exp\prn*{-2\eps_l^2\theta_l S_l}$. Conditional on ${E}_{\text{clean}}:=E_{\text{bad}}^c$, i.e., the complement of $E_{\text{bad}}$, we have
\begin{align*}
    \prn*{1+ \sum_{l=1}^L\frac{\phi_{\mb{\theta},\mb{\alpha},\mb{\beta}}^l/8+1/4}{2^{l}}\prn*{\frac{t_l}{\theta_l^2S_l}-1}_+}^n
    &\le \prn*{1+ \sum_{l=1}^L\frac{\phi_{\mb{\theta},\mb{\alpha},\mb{\beta}}^l/8+1/4}{2^{l}}\prn*{\frac{\eps_l}{\theta_l}}}^n\\
    &\le \prn*{1+ \sum_{l=1}^L\frac{1/(8(1-\theta_l))+1/4}{2^{l}}\prn*{\frac{\eps_l}{\theta_l}}}^n\\
    &\le \prn*{1+ \sum_{l=1}^L\frac{1}{2^{l+1}\theta_l(1-\theta_l)}{{\eps_l}}}^n.
\end{align*}
{Here we are using the bound $\phi_{\mb{\theta},\mb{\alpha},\mb{\beta}}^l \leq \frac{1}{1-\theta_l}$, which follows from~\cref{lm:phi-bound}.}
On the other hand, under $E_{\text{bad}}$, we have
\begin{align*}
    \prn*{1+ \sum_{l=1}^L\frac{\phi_{\mb{\theta},\mb{\alpha},\mb{\beta}}^l/8+1/4}{2^{l}}\prn*{\frac{t_l}{\theta_l^2S_l}-1}_+}^n
    &\le \prn*{1+ \sum_{l=1}^L\frac{\phi_{\mb{\theta},\mb{\alpha},\mb{\beta}}^l/8+1/4}{2^{l}}\prn*{\frac{1}{\theta_l}-1}}^n\\
    &\le \prn*{1+ \sum_{l=1}^L\frac{1/(8(1-\theta_l))+1/4}{2^{l}}\prn*{\frac{1-\theta_l}{\theta_l}}}^n\\
    &\le \prn*{1+ \sum_{l=1}^L\frac{1}{2^{l+1}\theta_l}}^n,
\end{align*}
where the first inequality follows from $t_l\le \theta_l S_l$.
Hence we have
\begin{align}\label{eq:witheps}
    &\bbE_{t_l\sim \mathrm{Hyper}(\theta_l S_l, S_l, \theta S_l),\,\forall l\in[L]}\brk*{\prn*{1+ \sum_{l=1}^L\frac{\phi_{\mb{\theta},\mb{\alpha},\mb{\beta}}^l/8+1/4}{2^{l}}\prn*{\frac{t_l}{\theta_l^2S_l}-1}_+}^n}\notag\\
    &\le\prn*{1+ \sum_{l=1}^L\frac{1}{2^{l+1}\theta_l(1-\theta_l)}{{\eps}}}^n+ \prn*{1+ \sum_{l=1}^L\frac{1}{2^{l+1}\theta_l}{}}^n\cdot\bbP_{t_l\sim \mathrm{Hyper}(\theta_l S_l, S_l, \theta S_l),\,\forall l\in[L]}(E_{\text{bad}})\notag\\
    &\le \prn*{1+ \sum_{l=1}^L\frac{1}{2^{l+1}\theta_l(1-\theta_l)}{{\eps_l}}}^n+ \prn*{1+ \sum_{l=1}^L\frac{1}{2^{l+1}\theta_l}{}}^n\sum_{l=1}^L\exp\prn*{-2\eps_l^2\theta_l S_l}\notag\\
    &=\prn*{1+ \sum_{l=1}^L\frac{1}{2^{l+1}\theta_l(1-\theta_l)}{{\eps_l}}}^n+ \sum_{l=1}^L\exp\prn*{n\log\prn*{1+ \sum_{j=1}^L\frac{1}{2^{j+1}\theta_j}{}} -2\eps_l^2\theta_l S_l}\notag\\
    &=\prn*{1+ \sum_{l=1}^L\frac{1}{2^{l+1}\theta_l(1-\theta_l)}{{\eps_l}}}^n+ \sum_{l=1}^L\exp\prn*{n{ \sum_{j=1}^L\frac{1}{2^{j+1}\theta_j}{}} -2\eps_l^2\theta_l S_l}
\end{align}

Combining \cref{eq:simplify1+,eq:tohpg,eq:witheps} (note that we are focusing on $\cM_1$), we have
\begin{align*}
D_{\chi^2}\prn{\bbP_{n}^1\dmid\bbQ_{n}^1}\le \inf_{\substack{\eps_l\in(0,\theta_{l}^2 S_l),\\\forall l\in[L]}}\crl*{\prn*{1+ \sum_{l=1}^L\frac{\eps_{l}}{2^{l+1}\theta_{l}(1-\theta_{l})}{{}}}^n+ \sum_{l=1}^L\exp\prn*{n{ \sum_{j=1}^L\frac{1}{2^{j+1}\theta_{j}}{}} -2\eps_l^2\theta_{l} S_l}}-1,
\end{align*}
Let $c\in(0,1/2)$ be an arbitrary constant. We set $\eps_l=2{c}\cdot\frac{(1-\theta_{l})\theta_{l}}{n}$ (which belongs to $(0,\theta_{l}^2 S_l)$ {because $\epsilon_l < \theta_{l}$ since $n \geq 1$ and $\theta_{l}S_l\geq 1$ by assumption}) for all $l\in[L]$. Then we have
\[
\prn*{1+\sum_{l=1}^L{\frac{\eps_{l}}{2^{l+1}(1-\theta_{l})\theta_{l}}}}^n\le\prn*{1+\frac{c}{n}}^n\le e^c\le 1+2c,
\]
and
\[
D_{\chi^2}(\bbP_{n}^1\dmid\bbQ_{n}^1)\le2c+\sum_{l=1}^L\exp\prn*{n{ \sum_{j=1}^L\frac{1}{2^{j+1}\theta_{j}}}-{8c^2}
{\frac{(1-\theta_{l})^2\theta_{l}^3}{n^2}}S_l}.
\]
In particular, whenever
$S_l\ge \frac{n^3}{4c^2\theta_{l}^3(1-\theta_{l})^2}\frac{1}{\min_{j\in[L]}\theta_{j}}$, we have
\[
D_{\chi^2}(\bbP_{n}^1\dmid\bbQ_{n}^1)\le2c+\exp\prn*{n{ \sum_{j=1}^L\frac{1}{2^{j+1}\theta_{j}}}-2n\frac{1}{\min_{j\in[L]}\theta_j}}\le 2c+\exp(-n).
\]
Since $\theta_l=\frac{\alpha}{1-(l-1)\alpha}$ and the parameter $\alpha\in\brk{\frac{1}{2L},\frac{1}{L+1}}$ for the MDP family $\cM_1$, we have $\theta_{l}\in[\frac{1}{2L},\frac{1}{2}]$ for all $l\in[L]$. Setting $c=1/10$. Whenever $n\ge5$ and $S-5> 3200 n^3L^6$, we have
$S_l=\frac{S-5}{L_{\rm div}}(2L+1-l)(L+2-l)>1600n^3L^4$ for all $l\in[L]$ (recall that $L_{\rm div}\le4L^3$), and hence
\[
D_{\chi^2}(\bbP_{n}^1\dmid\bbQ_{n}^1)\le\frac{1}{5}+\exp\prn*{-n}\le\frac{1}{4}.
\]
Using the same calculation, whenever $n\ge5$ and $S-5> 800 n^3L^6$, it holds that
\[
D_{\chi^2}(\bbP_{n}^2\dmid\bbQ_{n}^2)\le\frac{1}{5}+\exp\prn*{-n}\le\frac{1}{4}.
\]
{Combining th above two inequalities with~\cref{eq:tv_triangle+}, we have 
  $\Dtv{\bbP_{n}^1}{\bbP_{n}^2} \leq 1/2+n/{(8\cdot2^{L})}$, which
  proves the lemma.}

\qed

\section{Proofs of \cref*{prop:value_calculation+} and \cref*{prop:conc_calculation+}}\label{sec:veryfication+}
  \begin{proof}[\pfref{prop:value_calculation+}]
  Since for all states in $\cS\setminus\{\init\}$ the two actions in $\cA$ have identical effects, we have $Q^\pi(s,a)=Q^\star(s,a)$ for all $(s,a)\in\cS\times\cA$  and for all $\pi:\cS\rightarrow\Delta(\cA)$. Hence we only need to show $Q_M^\star=f_1$ for all $M\in\cM_1$ and $Q_M^\star=f_2$ for all $M\in\cM_2$.
  
    Consider an arbitrary $M=M^L_{\mb{\psub},\alpha,w}\in\cM$.  First, for any self-looping terminal state
    $s\in\{W,X,Y,Z\}$, we have
    \[
      V_M^\star(s)=Q_M^\star(s,\initac)= \sum_{h=0}^\infty \gamma^h
      R_{L,\alpha,w}(s,\initac)=\frac{1}{1-\gamma}\cdot \begin{cases}
        w,&s=W,\\
        1,&s=X,\\
        0,&s=Y,\\
        \frac{\alpha}{1-L\alpha},&s=Z.\end{cases}
    \]
    Next, for $l=L,\dots,1$, for any $l^{\rm th}$-layer intermediate state $s\in\cS^l$, by the
    Bellman optimality equation, we have
    \begin{align*}
      V_M^\star(s)=Q_M^\star(s,\initac)&= R_{L,\alpha,w}(s,\initac)+\gamma\bbE_{s'\sim P^L_{\mb{\psub},\alpha,w}(s,\initac)}\brk{V_M^\star(s')}\\
                    &=\begin{cases}
                      0+\gamma\brk*{ \gamma^{L-l}\alpha V_M^\star(X)+0}, &s\in \psub^l\\
                      0+\gamma\brk*{ \frac{(1-l\alpha)}{1-(l-1)\alpha} \bbE_{s'\sim\text{Unif}(\psub^{l+1})}V_M^\star(s')+0}, &s\in\widebar{\psub}^l
                    \end{cases}\\
                    &=\begin{cases}
                      \frac{\gamma}{1-\gamma}\frac{\gamma^{L-l}\alpha}{1-(l-1)\alpha},&s\in \psub^l\\
             \frac{\gamma}{1-\gamma}\frac{\gamma^{L-l}\alpha}{1-(l-1)\alpha},&s\in\widebar{\psub}^l
                    \end{cases}\\
                    &=\frac{\gamma}{1-\gamma}\frac{\gamma^{L-l}\alpha}{1-(l-1)\alpha}.
    \end{align*}
    For the initial state $\nullstate$, we have
    \[
      Q_M^\star(\nullstate,1)=
      R_{L,\alpha,w}(\init,1)+\gamma\brk{V_M^\star(W)}=\frac{\gamma w}{1-\gamma},
    \]
    \[
    Q_M^\star(\nullstate,2)=R_{L,\alpha,w}(\init,2)+\gamma\bbE_{s'\sim P^L_{\mb{\psub},\alpha,w}(s,2)}\brk{V_M^\star(s')}=\frac{\gamma V_{\alpha}}{1-\gamma}.
    \]
    Therefore, $Q_M^\star=f_1$ if $M\in\cM_1$, and $Q_M^\star=f_2$ if $M\in\cM_2$.
  \end{proof}
  
  \begin{proof}[\pfref{prop:conc_calculation+}]
   We now verify the concentrability condition \cref{eq:concentrability}. 
   
   Consider any $M\in\cM_1$. For any $(s,a)\in\cS\times\cA$, we have
   \[
   \sup_{\nu \text{ is admissible}}\nu(s,a)\le\begin{cases}
   1, &\text{if } s\in\{\init,W,X,Y\}, a\in\cA,\\
   {\frac{1}{2}\cdot\frac{1}{2^L}}, &\text{if }s=Z,  a\in\cA, \\&{\text{(maximized when $h=1$)}}\\
   \frac{1}{2}\cdot\frac{1}{2^l}\frac{1}{S_l}, &\text{if } s\in\widebar{\psub}^l, a\in\cA, l\in[L],\\&{\text{(maximized when $h=1$)}}\\
   \frac{1}{2}\cdot\frac{1}{2}\frac{1}{S_1}, &\text{if }s\in\psub^1,a\in\cA,\\&{\text{(maximized when $h=1$)}}\\
   \max\crl*{\frac{1}{2}\cdot\frac{1}{2^l}
\frac{1}{S_l}, \frac{1}{2}\cdot\frac{1}{2^{l-1}}\frac{1-(l-1)\alpha_2}{1-(l-2)\alpha_2}\frac{1-(l)\alpha_1}{1-(l-1)\alpha_1}\frac{1}{\theta_l^{(1)}S_l}}, &\text{if } s\in{\psub}^l, a\in\cA, 2\le l\le L,\\&{\text{(maximized over $h=1,2$)}}
   \end{cases}
   \]
   
Recall the definition of $\mu$ in \cref{sec:finish+}. We have
\[\mu(s,a)\ge\begin{cases}
   \frac{1}{16}\cdot\frac{1}{2}, &\text{if } s\in\{\init,W,X,Y\}, a\in\cA,\\
   {\frac{1}{8}\cdot\frac{1}{2^L}}\cdot\frac{1}{2}, &\text{if }s=Z,  a\in\cA, \\
   \frac{1}{8}\cdot\frac{1}{2^l}\frac{1}{S_l}\cdot\frac{1}{2}, &\text{if } s\in\widebar{\psub}^l, a\in\cA, l\in[L],\\
   \frac{1}{8}\cdot\frac{1}{2}\frac{1}{S_1}\cdot\frac{1}{2}, &\text{if }s\in\psub^1,a\in\cA,\\
   \frac{1}{8}\cdot\frac{1}{2^l}\frac{1}{S_l}\cdot\frac{1}{2}, &\text{if } s\in{\psub}^l\cdot\frac{1}{2}, a\in\cA, 2\le l\le L
   \end{cases}\] 

Combining the above two inequalities, we have
\begin{align*}
    \sup_{\nu\text{ is admissible}}\nrm*{\frac{\nu}{\mu}}_{\infty}&\le\min_{2\le l\le L}~\prn*{8\cdot2^l S_l\cdot2}\frac{1}{2}\cdot\frac{1}{2^{l-1}}\frac{1-(l-1)\alpha_2}{1-(l-2)\alpha_2}\frac{1-(l)\alpha_1}{1-(l-1)\alpha_1}\frac{1}{\theta^{(1)}_lS_l}\\
    &=\min_{2\le l\le L}~\frac{16}{\theta_l^{(1)}}=\frac{16}{\theta_2^{(1)}}=\frac{16(1-\alpha_2)}{\alpha_2}\le 32L,
\end{align*}
where the last inequality follows from $\alpha_2\ge 1/(2L)$.

Similarly, consider any $M\in\cM_2$, we have
\[
\sup_{\nu\text{ is admissible}}\nrm*{\frac{\nu}{\mu}}_{\infty}\le 32L.
\]

We conclude that the construction satisfies concentrability with $\Conc\le{32L}$.
  \end{proof}

\end{document}